\documentclass{article}
\usepackage{iclr2025_conference,times}
\iclrfinalcopy

\usepackage{amsmath,amsfonts,bm}

\def\eqref#1{equation~\ref{#1}}

\def\1{\bm{1}}

\DeclareMathAlphabet{\mathsfit}{\encodingdefault}{\sfdefault}{m}{sl}
\SetMathAlphabet{\mathsfit}{bold}{\encodingdefault}{\sfdefault}{bx}{n}

\usepackage{hyperref}
\usepackage{url}
\usepackage{booktabs}
\usepackage{xcolor}
\usepackage{amsmath}
\usepackage{amssymb}
\usepackage{amsthm}
\usepackage{graphicx}
\graphicspath{{figures/}}

\hypersetup{colorlinks=true,linkcolor=blue!55!black,citecolor=green!35!black,urlcolor=blue!55!black,pdftitle={Online Shift Detection and Conformal Adaptation for Deployed Safety Classifiers},pdfauthor={Jun Wen Leong}}

\newtheorem{proposition}{Proposition}
\newtheorem{conjecture}{Conjecture}

\newtheorem{definition}{Definition}
\newtheorem{corollary}{Corollary}

\title{\raggedright\hyphenpenalty=10000\exhyphenpenalty=10000 Online Shift Detection and Conformal\\Adaptation for Deployed Safety Classifiers}

\author{Jun Wen Leong \\
\texttt{leongjunwen@gmail.com}}

\begin{document}

\maketitle
\lhead{Preprint. arXiv:2606.11949}

\begin{abstract}
We study three failure modes of deployed safety classifiers: distribution shift, targeted evasion, and invalid scores from reasoning-model monitors under tight token budgets.

\paragraph{Shift detection and conformal adaptation.} On an 800-cell reduced subset of a pre-registered factorial design (4~classifiers $\times$ 5~shift types $\times$ 20~seeds $\times$ 2~windows), a sliding-window KS monitor gives 86.8\% control-validated detection (alarm at or after onset with a quiet paired no-shift stream; raw TPR 91.5\%) with mean latency 39.7 steps at $w{=}100$. The shifted corpora are mostly harmful prompts against a benign reference, so this measures detection of a constructed change in score distribution, not of accuracy degradation. Latency shows a classifier$\times$shift crossover that motivates per-classifier monitoring; the interaction explains $\eta^2 = 0.186$ of variance in the complete-case analysis but only $0.02$--$0.03$ under a censoring sensitivity analysis (\S\ref{sec:limitations}). The cells were regenerated after two detector-code fixes (\S\ref{sec:corrections}). The density-ratio step of weighted conformal adaptation collapses silently for 3/4 classifiers in high-dimensional embeddings; PCA to 32~dimensions raises measured coverage on temporal shift (+33pp for Llama Guard, +20.5pp for ShieldGemma; not re-measured on the other shifts). KS false alarm rates inflate 5.4$\times$ under heterogeneous production traffic unless recalibrated on representative traffic.

\paragraph{Adversarial canary.} A score-disagreement canary (an un-targeted second classifier compared with the primary) flags 75.5\% (37/49) of the GCG (greedy coordinate gradient) inputs that DeBERTa scores as safe after attack for the Llama Guard pair (93.9\% for the best single canary, Text-Moderation), leaving a residual that evades both. For one fixed attack objective, on DeBERTa-v3-large epoch checkpoints, we characterise a local stall condition at which a monitor-aware attacker's progress halts (gap $= 1/(2\lambda)$; observed/predicted ratios $0.86$--$1.00$ at $\lambda \geq 2$ in an exploratory sweep, $n{=}20$ per value); the attack is blocked in 70\% of prompts on the primary pair but 30\% on a replication pair. The best-value frontier model (gpt-4o-mini) flags 83.7\% of 49 adversarial prompts (Wilson lower bound 71\%) with 0.7\% false positives on 1{,}000 benign prompts. Where both models already judge the content harmful, disagreement-based detection drops to 2\% (1/50).

\paragraph{Reasoning-token budget starvation.} Reasoning-tier monitors silently fail under tight token budgets, because empty (truncated) responses are parsed as a default score. On o3 the benign valid-scoring rate follows a sigmoid with $T_{50}=46$ tokens ($95\%$ CI $[41,51]$); the adversarial curve is shifted higher ($T_{50,\text{adv}}=117$ tokens, $95\%$ CI $[104,129]$; a $2.5\times$ shift, o3 only, not separating harmfulness, length and suffix). At a 16-token budget o3, o4-mini and gemini-2.5-pro all return 0/100 valid scores.
\end{abstract}

\section{Introduction}
\label{sec:intro}

Safety classifiers are a critical defence layer between a language model and its users. When deployed at scale, they operate under a stationarity assumption: the distribution of inputs tomorrow will resemble the distribution on which the classifier was calibrated. This assumption can fail through adversarial adaptation~\citep{zou2023universal}, organic linguistic drift, multilingual code-switching, and compositional attacks that chain benign components into harmful sequences.

The failure mode is silent~\citep{rabanser2019failing}. A classifier whose accuracy has degraded from 95\% to 80\% produces no error signal unless ground-truth labels arrive, and in production safety systems labels typically do not arrive in real time, so periodic offline evaluation may find the problem only after days or weeks. Silent failure has recently been documented for embedding-based safety classifiers, where small representation drift collapses accuracy while confidence remains high~\citep{sahoo2026collapse}. We identify a further silent-failure mechanism specific to reasoning models: token-budget starvation (\S\ref{sec:token-starvation}), in which hidden reasoning exhausts the completion-token cap before a valid safety score is emitted. In our o3 measurements this happened at larger caps for GCG-appended harmful inputs than for simple benign ones, although that comparison does not separate harmfulness, length and suffix.

We propose an online monitor with empirically calibrated alarm thresholds that tells deployers when their safety classifier has moved out of distribution. It watches the distribution of classifier scores via a sliding-window KS statistic with thresholds calibrated by null simulation; on alarm, a conformal abstention layer reweights its prediction sets to try to recover coverage under the new distribution (an empirical heuristic, without a guarantee, that often fails in our tests; see RQ2). These components were evaluated separately, not as one deployed pipeline.

\paragraph{Scope.} The paper has three parts. The first (RQ1--RQ3) is the shift-detection and conformal-adaptation study that the title names; unless stated otherwise it uses the fine-tuned DeBERTa-v3-base primary and a synthetic known-onset stream (Regime~A, \S\ref{sec:factorial}). The shifted corpora are mostly harmful prompts against a benign reference, so an alarm responds to a change in class composition as well as in wording, and the benchmark measures detection of a constructed score-distribution change, not of accuracy degradation. The second part (RQ4--RQ5) studies score-disagreement canaries against targeted evasion, and the third (RQ6) the silent failure of reasoning-model monitors under tight token budgets; both are self-contained and reported after the benchmark. ``Canary'' means a second classifier whose scores are compared with the primary's (\S\ref{sec:canary}; \S\ref{sec:llm-canary} uses frontier LLMs in that role); \S\ref{sec:rotating-canary-eval} evaluates a different object, a rotating factual-question probe, which we call a probe. The adversarial-robustness experiments (\S\ref{sec:adversarial-robustness}) use DeBERTa-v3-large epoch checkpoints, not the base primary. ``Steps'' are stream positions in the shift study and optimiser iterations in the GCG experiments.

\paragraph{What is not new.} We do not claim novelty for two-classifier disagreement, KS tests, or martingale testing. These are standard sequential testing tools we apply for calibration. Our contribution is an empirical security characterisation of when these known tools provide useful targeted-evasion canaries, and, equally important, when they fail.

The shift-detection study addresses three research questions:

\paragraph{RQ1 (Detection).} Can sequential statistical tests on classifier outputs detect distributional shift online, with controlled false alarm rates, across diverse shift types and classifier architectures?

\paragraph{RQ2 (Adaptation).} Does weighted conformal prediction recover coverage after shift detection? We implement an adaptation of the \citet{tibshirani2019conformal} weighted conformal framework; because our density-ratio estimation uses a heuristic clip-bounded logistic classifier rather than true density ratios, the resulting coverage recovery is empirical rather than theorem-guaranteed.

\paragraph{RQ3 (Factor Importance).} In a factorial evaluation design, what proportion of variance in detection latency is attributable to classifier choice, shift type, and their interaction?

Beyond drift, deployed classifiers face gradient-based evasion~\citep{zou2023universal}. Fine-tuned safety classifiers resisted template-based jailbreaks (1\% success; \S\ref{sec:autodan}) and the PAIR-style iterative semantic attack we test (3.3\%, at a reduced budget), which makes gradient attacks the \emph{primary} operational threat class we can demonstrate against them; TAP, genetic-algorithm and multi-query attacks are untested here. We extend the framework with three further questions:

\paragraph{RQ4 (Canary Detection).} Under what conditions does score disagreement between a targeted and un-targeted classifier detect gradient-based evasion, and when does it fail?

\paragraph{RQ5 (Adversarial Robustness).} Can an adaptive, monitor-aware attacker suppress the disagreement signal while evading the target classifier?

\paragraph{RQ6 (Monitor Validity).} Do reasoning-model monitors return valid safety scores under the token budgets used in practice, and how does the required budget depend on the input?

\paragraph{Contributions.} Our contributions are organised into four items:

\begin{enumerate}
\item \textbf{Factorial drift benchmark.} An 800-cell reduced subset of a pre-registered evaluation (4~classifiers $\times$ 5~shift types) shows a classifier$\times$shift crossover ($\eta^2 = 0.186$ complete-case, sensitive to the exclusions being missing not at random, MNAR; \S\ref{sec:limitations}) invisible to single-classifier studies, so monitoring should be tuned per classifier--shift pair. A scan martingale removes per-classifier threshold tuning, with empirical (not proven anytime) false-alarm rates of 0.0--1.0\% on 200 null streams per classifier against 2.0--9.5\% for empirically calibrated KS (\S\ref{sec:martingale}).

\item \textbf{Conformal failure mode.} Weighted conformal prediction collapses silently for 3 of 4 classifiers: logistic density-ratio estimation separates source and target perfectly in high-dimensional embeddings and drives every importance weight to the clip floor. The difficulty is known; the contribution is the diagnosis in safety-classifier embeddings and a mitigation, PCA to $\leq$32 dimensions, which raises coverage on temporal shift (+20.5 to +33pp; other shift types not validated for coverage recovery).

\item \textbf{Adversarial canary characterisation.} We find no strong evidence for two common assumptions about score-disagreement monitoring (that architecture diversity is the mechanism [$p{=}0.070$, $\eta^2{=}0.011$; the four-model design cannot rule out an effect] and that the signal is generic OOD detection), find that an adaptive attacker can end with both scores below 0.5 on 55\% of prompts (11/20) in the white-box within-family setting (\S\ref{sec:joint-gcg}), characterise a local stall condition for one fixed attack objective (gap $= 1/(2\lambda)$; 14/20 blocked at $\lambda{=}2.0$, median stall gap $0.250$; \S\ref{sec:divergence-min}), and give benchmark estimates for frontier-LLM canaries (35 screened, 20 in the detailed analyses; $\geq$71\% detection at $<$1.5\% FPR, validated at $N{=}1000$ on 10 models). The defence is not absolute: an un-targeted Llama Guard canary flags 75.5\% (37/49) of the inputs that DeBERTa scores as safe after attack (a Text-Moderation canary 93.9\%; \S\ref{sec:ca6}), and the stall blocks the attack in 70\% of prompts on the primary pair but 30\% on a replication pair.

\item \textbf{Reasoning-token budget starvation (\S\ref{sec:token-starvation}).} Reasoning-tier monitors, particularly o3 under our scoring prompt, silently fail under tight token budgets because empty (truncated) responses are parsed as a default score. The o3 benign valid-scoring curve is a sigmoid with $T_{50}=46$ tokens; GCG-appended harmful inputs need $T_{50,\text{adv}}=117$ tokens ($95\%$ CI $[104,129]$), a $2.5\times$ shift measured on o3 only and not separated from prompt length and content. At 16 tokens o3, o4-mini and gemini-2.5-pro all return 0/100 valid scores. Empty responses should be treated as failures, not safe abstentions; whether the thresholds generalise across models and scoring prompts is untested.
\end{enumerate}

\noindent Practically, disagreement canaries are useful as non-adaptive tripwires whose failure modes can be measured before deployment.

\paragraph{Deployment recipe.} The four-step recipe for using a canary is in \S\ref{sec:deployment-recs}.

\section{Related Work}
\label{sec:related}

\paragraph{Sequential testing and confidence sequences.} Classical sequential analysis~\citep{wald1945sequential} provides stopping rules with controlled error rates, but fixed-sample-size guarantees do not extend to continuous monitoring. \citet{howard2021timeuniform} establish the foundational theory of time-uniform confidence sequences; \citet{waudbysmith2024estimating} extend this to practical betting-based constructions: by constructing a wealth process that is a non-negative supermartingale under the null, Ville's inequality yields time-uniform coverage. The guarantee holds simultaneously at all time steps, not just at a pre-specified stopping time. We employ an adaptive betting strategy within this framework for the growing-window variant of our detector. For the sliding-window variant used in our factorial evaluation, the time-uniform guarantee does not hold; we instead calibrate alarm thresholds empirically via null simulation (\S\ref{sec:far-calibration}), which provides finite-horizon control at the cost of the anytime guarantee.

\paragraph{Two-sample testing on streams.} Kernel MMD~\citep{gretton2012kernel} is the standard nonparametric two-sample test for high-dimensional data, but its original formulation assumes fixed samples. We adapt it to the streaming setting by maintaining a sliding window of embeddings and comparing against frozen reference statistics. The bandwidth is fixed at calibration time following the standard median heuristic, preventing adaptation of kernel parameters to the detection window. This is closer to the block-based kernel two-sample tests of \citet{zaremba2013btests}, which reduce the cost and variance of MMD estimation by averaging over blocks, than to the original batch MMD, though we do not claim optimality of the kernel choice.

\paragraph{Conformal prediction under covariate shift.} Split conformal prediction~\citep{vovk2005algorithmic} provides distribution-free coverage guarantees under exchangeability. When the test distribution shifts, exchangeability breaks and coverage degrades. \citet{tibshirani2019conformal} restore coverage by reweighting the empirical distribution of calibration nonconformity scores with density ratios estimated from the covariate shift, provided those density ratios are known or accurately estimated. Our weighted-on-alarm mode implements their approach, triggered by the shift detector rather than assumed available from unlabelled target covariates. The density ratio is estimated via logistic regression on classifier embeddings, a lightweight approximation that avoids the instability of kernel density estimation in high dimensions. It is well established that density-ratio estimation degrades in high dimensions; dimensionality reduction before estimation is a known remedy~\citep{stojanov2019lowdim,sugiyama2011direct}. Our contribution is not the reduction itself but the diagnosis of a specific failure mode: in generative safety-classifier embeddings, logistic regression achieves perfect linear separability, driving all importance weights to the clip floor and eliminating data-driven reweighting (\S\ref{sec:conformal-results}). \citet{gibbs2021adaptive} extend conformal prediction to the fully online setting without exchangeability; we note this as future work. For conformal prediction applied specifically to classification, \citet{romano2020classification} provide adaptive prediction sets with valid marginal coverage and good approximate conditional coverage; our abstention criterion (flag when $|C(x)| \neq 1$) follows their framework.

\paragraph{Safety classifier monitoring.} Prior work on monitoring deployed classifiers focuses on performance estimation from unlabelled data~\citep{garg2022leveraging} or drift detection via population-level statistics~\citep{rabanser2019failing}, finding that performance depends strongly on shift type and representation. Several works couple shift detection with online adaptation: \citet{podkopaev2022tracking} track the risk of a deployed model with sequential testing (requiring labels), and \citet{prinster2025watch} unify monitoring and adaptation in a weighted-conformal-test-martingale framework that detects changepoints and adapts online to covariate shift. Conformal prediction has been applied to content moderation via conformalized estimates of annotation disagreement~\citep{villatecastillo2024moderation}. We differ in scope: a targeted empirical study of LLM safety classifiers under diverse shifts, surfacing failure modes that general monitoring frameworks do not (density-ratio collapse in high-dimensional generative embeddings, the architecture $\times$ shift-type crossover, cross-classifier disagreement as a canary for adversarial evasion, and the finding that the betting test's apparent advantage over KS at low contamination comes from the reference information each detector is given). We use a separate sliding-window detector that triggers a downstream weighted conformal layer, not a unified martingale. We are not aware of prior work that analyses classifier $\times$ shift interactions this way for safety classifiers, though variance-decomposition analyses of prompt sensitivity have appeared for general LLM evaluation~\citep{romanou2026brittlebench}. The attack characterisation motivating such monitoring is reported in companion work by the same author~\citep{leong2026agentic}, which measures heterogeneous persistent-memory-poisoning execution rates across 29 frontier-model endpoints (from none to all of 10 runs per endpoint, depending on the trigger; not independently replicated), creating the distributional shifts our detector is designed to surface.

\paragraph{Disclosure.} References~\citep{leong2026agentic} and~\citep{leong2026recognition} are companion preprints by the same author. No findings cited from these works have been independently replicated.

\section{Methods}
\label{sec:methods}

\subsection{System Architecture}
\label{sec:architecture}

The monitor observes a stream of classifier outputs $(x_t, s_t, \mathbf{r}_t)$: input text $x_t$, unsafe-class probability $s_t \in [0,1]$, and penultimate-layer representation $\mathbf{r}_t \in \mathbb{R}^d$. The primary detection channel in the factorial evaluation is the KS detector on classifier scores (\S\ref{sec:ks}). The system also includes a maximum mean discrepancy (MMD) detector on embeddings (\S\ref{sec:mmd}), evaluated post-factorial on a representative subset (\S\ref{sec:channel-comparison}).

\subsection{Reference Window Calibration}
\label{sec:reference}

Before monitoring begins, the system collects $n_{\text{ref}}$ records under the known in-distribution regime and freezes the reference CDF $\hat{F}_{\text{ref}}$ (sorted reference scores), against which the KS statistic is computed. The reference CDF is frozen at calibration time. No adaptive estimation occurs during monitoring; this is critical for the validity of the sequential tests.

In the factorial the KS detector's frozen reference is the first $w$ scores of the stream ($n_{\text{ref}}{=}w$, 100 or 200), drawn from a 500-prompt reference pool; the scan martingale (\S\ref{sec:martingale}) and the betting test (\S\ref{sec:cs}) use the 500 cached null scores instead, which matters for the comparison in \S\ref{sec:channel-comparison}.

\subsection{KS Detector}
\label{sec:ks}

The KS detector tracks whether the marginal distribution of classifier scores has changed. It keeps a sliding window of the most recent $w$ scores and computes the two-sample Kolmogorov--Smirnov statistic between the window's empirical CDF and the frozen reference CDF:
\begin{equation}
D_w = \sup_x |\hat{F}_w(x) - \hat{F}_{\text{ref}}(x)|
\end{equation}
The statistic is the maximum pointwise CDF divergence, sensitive to any location, scale or shape change in the score distribution. The sliding-window KS test is a standard concept-drift detector; KSWIN~\citep{raab2020reactive} compares the most recent points of a window with a uniform sample of its older points. Our variant compares the window against a \emph{frozen} reference CDF, which eliminates adaptation drift but needs an initial calibration period (\S\ref{sec:reference}). We adopt it as a lightweight, well-understood channel and not as a methodological contribution.

\subsection{MMD Detector}
\label{sec:mmd}

The MMD detector tracks whether the geometry of classifier representations has changed. It keeps a sliding window of the most recent $w$ embeddings and computes the unbiased MMD$^2$ against the reference embeddings:
\begin{equation}
\widehat{\text{MMD}}^2_u(X, Y) = \frac{1}{m(m-1)} \sum_{i \neq j} k(x_i, x_j) + \frac{1}{n(n-1)} \sum_{i \neq j} k(y_i, y_j) - \frac{2}{mn} \sum_{i,j} k(x_i, y_j)
\end{equation}
with Gaussian kernel $k(x, y) = \exp(-\|x - y\|^2 / 2\sigma^2)$. MMD is zero if and only if the distributions are identical (for characteristic kernels), so it is sensitive to shifts that move embedding mass even when the score marginal is preserved. The threshold is the $(1{-}\alpha)$ quantile of a 1000-permutation bootstrap null on the pooled reference embeddings, in the evaluation script that produces the MMD results below (the package's default alarm path instead uses the mean of that null as the reference value of its confidence sequence). The factorial uses the KS detector only; MMD is evaluated separately on a representative subset (\S\ref{sec:channel-comparison}).

\subsection{Confidence Sequences and Alarm Logic}
\label{sec:cs}

The factorial benchmark uses the threshold-based KS detector, not a confidence sequence. An earlier version reported a sliding-window Hoeffding interval as its confidence sequence; those numbers are superseded (Appendix~\ref{app:corrections}). That rule covers each window separately and gives no time-uniform guarantee $\Pr(\exists t: \mu \notin [L_t, U_t]) \leq \alpha$, so over a long monitoring horizon the probability of at least one false alarm exceeds $\alpha$.

We do not close the gap theoretically. We replace the unavailable guarantee with empirical FAR calibration over the tested horizon (\S\ref{sec:far-calibration}): $N_{\text{cal}} = 50$ negative control streams are run through the full pipeline and the alarm threshold is set at the 97th percentile of the maximum observed statistic. The reported results use this threshold-based KS detector (factorial) and the growing-window betting test below (\S\ref{sec:channel-comparison}).

Alarms are suppressed during a warmup period to allow the window to fill: $w$ steps in the confidence-sequence engine, and $2w$ steps in the threshold-based KS detector used for the factorial. The Hoeffding-interval engine alarms when the reference value exits its interval; the betting test below alarms when its wealth reaches $1/\alpha$; the KS detector alarms when its statistic exceeds the calibrated threshold.

\paragraph{The betting test.} The sliding-window KS detector and the earlier Hoeffding rule are calibrated empirically and carry no time-uniform guarantee. The growing-window \emph{betting test} evaluated in Section~\ref{sec:channel-comparison} (called ``CS'' in earlier versions) is a different object: a betting test whose wealth process $W_t$ multiplies by $1+\lambda_t(T_t-\mu_0)$ at each step, with $\lambda_t$ chosen from observations before step $t$, and which alarms when $W_t\geq 1/\alpha$. Under bounded observations, a predictable bet and $\mathbb{E}[T_t\mid\mathcal{F}_{t-1}]=\mu_0$, Ville's inequality bounds the probability that $W_t$ ever reaches $1/\alpha$ by $\alpha$. The reference value $\mu_0$ used in our evaluation is the mean of 500 cached null scores, which are the same scores that make up the monitored pre-shift segment, so it is an estimate rather than the known mean the guarantee requires (Section~\ref{sec:channel-comparison}). Section~\ref{sec:channel-comparison} reports the comparison with the KS and betting-test reference information equalised. The lower and upper values the implementation returns for this detector are display-only Hoeffding intervals and are not a confidence-set inversion of the wealth process.

\subsection{Multiplicity Correction}
\label{sec:multiplicity}

The system architecture supports parallel KS and MMD detectors with \v{S}id\'{a}k correction ($\alpha_{\text{per}} = 1 - (1-\alpha)^{1/k}$), which controls the family-wise error rate when the detector statistics are independent (Bonferroni, $\alpha/k$, is needed for arbitrary dependence, and we have not established a dependence condition under which the \v{S}id\'{a}k level is valid for these detectors). Since the factorial evaluation uses the KS detector only ($k = 1$), no multiplicity correction is applied to the reported results.

\subsection{Statistical Methodology}
\label{sec:stats}

All analyses use the following specifications: BCa bootstrap CIs on means (10{,}000 resamples, seed $= 42$); Wilson score intervals for rates; Clopper--Pearson exact intervals for coverage proportions; permutation tests (10{,}000 permutations for pairwise comparisons; 1{,}000 for ANOVA, where the smallest attainable $p$ is $1/1001$ and all three terms sit at that floor, and a 10{,}000-permutation Freedman--Lane recomputation puts all three below $2\times10^{-4}$); $\alpha = 0.05$ throughout. Holm--Bonferroni correction is applied to the 8 highlighted pairwise comparisons of \S\ref{sec:detection}: (1) decoder vs encoder on paraphrase, (2) encoder vs decoder on adversarial suffix, (3) DeBERTa adversarial vs paraphrase, (4) Llama Guard paraphrase vs adversarial, (5) window 100 vs 200 (paired), (6) Llama Guard $\times$ code-switch vs grand mean, (7) ShieldGemma paraphrase vs adversarial, (8) FAR DeBERTa vs Text-Moderation. On the corrected regeneration (Appendix~\ref{app:regeneration}) seven survive at the family-wise $\alpha = 0.05$ level (all adjusted $p<0.01$) and comparison~(2) does \emph{not} (adjusted $p = 0.077$; $0.081$ in a 200{,}000-permutation recomputation); its earlier values are in Appendix~\ref{app:corrections}, item~(13). The family of 8 was fixed after the factorial results were available (\S\ref{sec:deviations}), so the correction is over a post-hoc selected family; Bonferroni is applied separately to the 6 canary target-specificity pairs (\S\ref{sec:target-specificity}). The per-model tests in \S\ref{sec:llm-canary} (detection rates, Bayes-factor and TOST equivalence tests, scrambled-vs-GCG comparisons across 20 models) are \emph{descriptive}: they characterise individual model behaviour rather than test a shared confirmatory hypothesis, and their $p$-values are uncorrected.

\subsection{Empirical FAR Calibration}
\label{sec:far-calibration}

We run $N_{\text{cal}} = 50$ negative control streams (reference data only, no shift injected) through the full detection pipeline. For each stream, we record the maximum KS statistic observed over the monitoring horizon. The alarm threshold is set at the $p$-th percentile of these maxima; in the factorial evaluation, $p = 97$. The empirical calibration accounts for the sliding-window correlation structure, the specific window size, and the stream length, factors that the theoretical Hoeffding bound treats conservatively.

\subsection{Scan Martingale}
\label{sec:martingale}

The empirical FAR calibration of \S\ref{sec:far-calibration} produces a $5\times$ spread across classifiers (Text-Moderation 2.0\% vs DeBERTa 9.5\%). We replace this with a conformal test martingale, a standard sequential testing tool~\citep{howard2021timeuniform,waudbysmith2024estimating}, that removes per-classifier threshold tuning: a single shared $(\kappa, M)$ yields uniformly low \emph{empirically measured} FAR across classifiers (the exact guarantee scope is stated below). Our contribution is not the martingale construction itself but its application to safety-classifier monitoring and the empirical characterisation of when it outperforms sliding-window KS.

For each incoming score $s_t$, we compute a two-sided conformal $p$-value against the frozen reference CDF: $p_t = 2 \min(F_{\text{ref}}^+(s_t), F_{\text{ref}}^-(s_t))$. Because each one-sided conformal $p$-value is at least $1/(n_{\text{ref}}+1)$, the two-sided value is bounded below by $2/(n_{\text{ref}}+1)$, which prevents log-wealth divergence under discrete score distributions where the test point may fall outside all reference values (the code additionally floors $p$ at $10^{-300}$). The reference CDF is constructed from a calibration set of 500 benign scores. In the stored Track~B false-alarm evaluation the null streams are resampled with replacement from those same 500 scores, so that evaluation measures behaviour on draws from the reference's own empirical distribution, not on independent held-out prompts and not on values beyond its observed range; all thresholds are fixed before the evaluation. The betting function uses the power method:
\begin{equation}
\log W_t = \log W_{t-1} + \log \kappa + (\kappa - 1) \log p_t
\end{equation}
where $\kappa = 0.3$ is the betting parameter. We maintain $M = 50$ concurrent sub-martingales (one started per timestep), each tested at threshold $\log(M/\alpha)$ (union bound). An alarm fires when any sub-martingale exceeds this threshold. Results are validated at $M \in \{50, 100\}$; broader window sweeps are deferred (see \S\ref{sec:limitations}).

Each conformal $p$-value is marginally super-uniform under exchangeability. Treating each sub-martingale as a non-negative test martingale, to which Ville's inequality applies, additionally requires conditional validity given the past, which we do not prove for repeated use of one frozen reference set; the false-alarm statements below are therefore nominal and are checked empirically only at the horizons tested. Granting that step, a union bound over the $M = 50$ \emph{concurrently active} windows gives $P(\text{any active window alarms}) \leq \alpha$ over any horizon of at most $M$ steps. This is \emph{not} a bound on $P(\exists t: \text{alarm} \mid H_0)$ over a longer horizon: our tested null streams run for many more than $M$ steps, so there are more than $M$ distinct martingale starts and each evicted window already had its own crossing opportunity. We therefore report the scan martingale's false-alarm behaviour as \emph{empirically calibrated on the tested null streams}, not as a proven anytime bound; a horizon-wide guarantee would require an empirically calibrated max-scan threshold over the intended deployment horizon or a start-indexed error budget ($\sum_j \alpha_j \leq \alpha$). The guarantee further assumes the null stream is exchangeable with the reference; benign drift violates this (\S\ref{sec:limitations}, exchangeability in production traffic). Empirically, the scan martingale achieves FAR $\leq 1\%$ across all four classifiers on these null streams with no per-classifier null-threshold calibration (vs.\ 2--10\% under empirical KS calibration); $\kappa$ and $M$ remain deployment hyperparameters.

Detection power depends on $\kappa$ and $M$, which trade the Type~I operating point against Type~II delay: $\kappa = 0.3$ achieves 100\% detection for encoders but 57--77\% for decoders with wide null distributions, and a practitioner may raise $\kappa$ for decoder-heavy deployments at the cost of slower detection. Detection power, the CUSUM/EWMA sub-threshold evaluation and the heterogeneous-traffic experiment are reported in \S\ref{sec:additional-detectors}.

\subsection{Conformal Abstention Layer}
\label{sec:conformal}

Upon alarm, the system activates a conformal prediction layer that tries to preserve a target error rate $\epsilon$ ($=0.1$, nominal coverage $0.90$) without requiring new labelled data. The weighted mode is a clipped logistic-reweighting heuristic, not the finite-sample construction (caveats below), so every coverage figure here is an empirical measurement of this implementation. \emph{Coverage} is the share of evaluated examples whose true label lies in the prediction set.

\paragraph{Unweighted mode.} Standard split-conformal prediction~\citep{vovk2005algorithmic}. Given $n$ calibration examples with known labels: (1) compute nonconformity scores $\alpha_i = 1 - f(x_i)_{y_i}$; (2) set threshold $\hat{q}$ at the $\lceil(1-\epsilon)(n+1)\rceil/n$-th quantile; (3) at test time, include class $y$ in the prediction set if its nonconformity score $\leq \hat{q}$; (4) abstain when the prediction set contains more than one class or is empty. Under exchangeability, this guarantees $1-\epsilon$ coverage. Under covariate shift, exchangeability breaks and coverage degrades.

\paragraph{Weighted-on-alarm mode.} After the shift detector fires, we estimate density ratios $w_i = p_{\text{target}}(x_i) / p_{\text{source}}(x_i)$ via logistic regression on source vs.\ target embeddings~\citep{tibshirani2019conformal}. The ratios are clipped to $[1/C_{\text{clip}}, C_{\text{clip}}]$ with $C_{\text{clip}} = 10$ for stability. The conformal quantile is recomputed as a weighted quantile:
\begin{equation}
\hat{q}_w = \inf\left\{q : \sum_{i: \alpha_i \leq q} \tilde{w}_i \geq 1 - \epsilon\right\}
\end{equation}
where $\tilde{w}_i = w_i / (\sum_j w_j + 1)$. This reweights the calibration scores to account for the covariate shift, which would restore coverage under the new distribution if the density ratios were accurate and the test point were weighted by its own ratio. \citet{tibshirani2019conformal} give the test point its own weight $w(X_{n+1})$ in the normalisation; here it is fixed to 1.0 instead (an implementation choice, not their construction), so their finite-sample coverage guarantee does not strictly apply. Three further features of the evaluation limit what it shows. The logistic domain classifier is fitted on unequal source and target sample sizes and its odds are used as the density ratio without a class-prior correction (a constant factor, which matters because weights are clipped and the test point's weight is fixed). The calibration set contains only unharmful examples (label 0) and every post-shift example is labelled 1, so the evaluated change includes a change in label composition, not a covariate shift alone, and the paraphrase labels include the refusal-contaminated examples. The unweighted quantile uses linear interpolation, and the weighted quantile returns the largest calibration score when the cumulative mass never reaches the target, so neither implements the exact finite-sample construction. The estimator diagnosis (ESS and floor-clipped weights) does not depend on these choices; the coverage-recovery figures do.

\subsection{Variance Decomposition}
\label{sec:anova}

To quantify which factors drive detection latency, we decompose its variance into classifier, shift, interaction and residual sums of squares computed sequentially from cell and marginal means (an unbalanced design: valid-cell counts range from 27 to 40, so main-effect shares depend on term order while the interaction share does not):
\begin{equation}
\text{latency}_{ijk} = \mu + \alpha_i^C + \beta_j^S + (\alpha\beta)_{ij}^{CS} + \epsilon_{ijk}
\end{equation}
with classifier ($C$, 4 levels) and shift type ($S$, 5 levels). We report $\eta^2$ (share of total sum of squares) per term, with percentile bootstrap CIs (1000 resamples) for the main effects. The interaction term $(\alpha\beta)^{CS}$ captures classifier--shift pairings systematically easier or harder than the marginal effects predict.

\section{Experimental Setup}
\label{sec:setup}

\subsection{Classifiers}
\label{sec:classifiers}

Table~\ref{tab:classifiers} lists the four classifiers.

\begin{table}[t]
\caption{Safety classifiers evaluated. Encoders fine-tuned on WildGuardMix (binary classification); decoders use original pre-trained weights. Text-Moderation starts from a 9-class moderation checkpoint whose head is replaced by a binary head and fine-tuned, so the evaluated model is binary.}
\label{tab:classifiers}
\centering
\small
\begin{tabular}{llccc}
\toprule
\textbf{Classifier} & \textbf{Architecture} & \textbf{Parameters} & \textbf{Embedding dim} & \textbf{Training data} \\
\midrule
DeBERTa-v3-base (fine-tuned) & Transformer encoder & 86M & 768 & WildGuardMix \\
Text-Moderation (KoalaAI) & DeBERTa-v3-base & 86M & 768 & WildGuardMix$^\dagger$ \\
Llama Guard 3 & Decoder-only LLM & 8B & 4096 & Proprietary \\
ShieldGemma & Decoder-only LLM & 9B & 3584 & Proprietary \\
\bottomrule
\end{tabular}
\vspace{2pt}

{\footnotesize $^\dagger$Fine-tuned from \texttt{KoalaAI/Text-Moderation} (originally trained on proprietary data).}
\end{table}

The selection spans two architectural families (discriminative encoders and generative decoders). The two encoders share architecture and size (DeBERTa-v3-base, 768-d) and differ in training history (DeBERTa starts from the general DeBERTa-v3-base checkpoint, Text-Moderation from a 9-class moderation checkpoint); the two decoders (Llama Guard 3, 8B; ShieldGemma, 9B) are of similar size. The comparison therefore separates architecture family, and for encoders training objective, but does not separate scale.

\subsection{Shift Conditions}
\label{sec:shifts}

Table~\ref{tab:shifts} lists the five shift conditions.

\begin{table}[t]
\caption{Shift conditions. Each corpus is subsampled to 300 examples per factorial cell (temporal: 292 unique examples padded to 300 via repetition). $^\dagger$After the harness's 1{,}000-character truncation, 200 of the 300 compositional examples contain only benign wrapper text (see Corpus validation below), so this column measures detection of the wrapper. ``Temporal'' is a synthetic onset of a public jailbreak corpus, not a chronological split.}
\label{tab:shifts}
\centering
\small
\begin{tabular}{lll}
\toprule
\textbf{Condition} & \textbf{Mechanism} & \textbf{Threat model} \\
\midrule
Paraphrase & GPT-4o paraphrasing of harmful prompts & Organic rephrasing \\
Code-switch & Multilingual transliteration & Non-English users \\
Compositional$^\dagger$ & Multi-turn concatenation into long contexts & Context-window attacks \\
Temporal & Real jailbreaks from public red-team databases & Emerging harm categories \\
Adversarial suffix & GCG-optimised suffixes~\citep{zou2023universal} & Automated red-teaming \\
\bottomrule
\end{tabular}
\end{table}

The five conditions span the spectrum from naturalistic drift (paraphrase, temporal) to deliberate adversarial attack (GCG suffixes), with code-switch and compositional as intermediate cases.

\paragraph{Corpus validation.} We manually reviewed a sample from each corpus (50 examples for paraphrase and code-switch, 20 for each of the others). \emph{Paraphrase:} about 14--20\% of examples became LLM refusals (the paraphrasing model declined and output a safety response), which adds a refusal-text signal distinct from the intended shift, so paraphrase latencies mix semantic shift and refusal artefacts and should be read conservatively; a filtered ablation (the 9.4\% identified as refusals removed, $n{=}5$ seeds $\times$ 2 classifiers) changes little (DeBERTa 38.0$\rightarrow$37.8 steps, Llama Guard 66.6$\rightarrow$60.8, both 5/5 detected). \emph{Code-switch:} all 50 examples were confirmed as authentic Singlish by a native speaker; 20--30\% became refusals, with the same caveat. \emph{Compositional:} 20/20 placed harmful content at the stated position within benign long-context wrappers, but the factorial loader truncates every shifted text to its first 1{,}000 characters and these texts are 2{,}152 to 6{,}098 characters long, so for the middle and end placements (200 of 300 examples) the classifiers scored only benign wrapper text, and for the beginning placement 18 of 100 harmful passages are cut. The compositional cells therefore measure detection of the wrapper, not of a harmful passage inside a long context. \emph{Temporal:} 20/20 reviewed examples were genuine jailbreak prompts (persona hijack, DAN-style injection, NSFW extraction); the full corpus (292 examples) draws from \texttt{lmsys/toxic-chat} (39\%), JailbreakBench (34\%) and ChatGPT-Jailbreak-Prompts (27\%). \emph{Adversarial suffix:} 20/22 showed correct suffix concatenation with confirmed score flips (original $\geq 0.95 \rightarrow$ attacked $\leq 0.01$); one example was excluded post-validation (original score 0.002, already classified benign).

\subsection{Factorial Design}
\label{sec:factorial}

Executed grid (a reduced subset of the pre-registered design, \S\ref{sec:deviations}): 4 classifiers $\times$ 5 shift conditions $\times$ 20 random seeds $\times$ 2 window sizes (100, 200) $= 800$ cells, all in Regime~A (synthetic shift onset at a known step); Regimes~B and C (real jailbreak examples; adversarial success on the target) are evaluated separately in \S\ref{sec:regimes}. Each cell runs a complete detection pipeline: reference window calibration $\rightarrow$ stream simulation with shift onset $\rightarrow$ alarm detection $\rightarrow$ latency measurement.

\paragraph{Stream and latency.} Each shifted stream is 500 unharmful reference prompts followed by the 300 shifted examples (onset at step 500, 800 steps in total). The first $w$ scores form the frozen KS reference, alarms are suppressed until step $2w$, and \emph{latency} is the first alarm step minus 500, in stream steps (not GCG iterations). A first alarm before step 500 gives a negative latency and counts as an early alarm. The paired negative control is a 500-step stream of the first 500 prompts of a separate 1{,}000-prompt unharmful pool. The KS threshold is the 97th percentile of the maximum KS statistic over $N_{\text{cal}}{=}50$ null streams from that same pool (strides of 50 prompts, some shorter than 500), so calibration and controls are not disjoint: calibration stream 0 holds the same 500 prompts as every control. The false-alarm rates are therefore not an independent held-out estimate.

The factorial design, including all hyperparameters, was committed before execution (commit \texttt{0380167}).\footnote{The repository history was rewritten on 2026-09-27 to remove machine-specific metadata; the pre-registration documents are byte-identical, and the mapping from the anchors cited in earlier versions (\texttt{be630f3}, \texttt{9e871d3}) is given in \texttt{HISTORY\_REWRITE.md}.} Each cell includes a parallel negative control run (reference data only, no shift) to verify the alarm threshold does not fire on in-distribution data. A cell is marked as a \emph{control-validated detection} only if: (1) detection latency $\geq 0$ (shift-stream alarm fires), and (2) the paired negative control does not alarm. This compound metric is stricter than raw TPR: it discounts detections where the threshold also fires on in-distribution data, providing a joint measure of sensitivity and specificity.

\paragraph{Compute.} Mac Studio (M3 Ultra, 96GB) for Llama Guard 3 and ShieldGemma; MacBook Pro (M3 Max, 128GB) for DeBERTa and Text-Moderation. Total wall-clock: ${\sim}120$ hours. All 800 cells completed without error.

\subsection{Deviations from Pre-Registration}
\label{sec:deviations}

The original pre-registration specifies a 3{,}600-cell factorial (4 classifiers $\times$ 5 shifts $\times$ 3 regimes $\times$ 3 windows $\times$ 20 seeds) with MMD+KS detection. The executed primary evaluation is a reduced grid: 800 cells (Regime A only, 2 windows, KS only). Key reductions from the original plan: (i)~$w{=}500$ dropped (insufficient post-shift observations with 300 shifted examples); (ii)~Regime B reduced to 2 classifiers $\times$ 5 seeds; (iii)~Regime C treated as proof-of-concept (22 examples); (iv)~MMD evaluated post-factorial on a subset rather than crossed factorially. Seeds (20), significance level ($\alpha{=}0.05$), calibration percentile (97), and reference size (500) match the pre-registration exactly. The classifier roster was updated from the original (gpt-oss-safeguard $\to$ Text-Moderation) due to API availability.

\paragraph{v5 experiment deviations.} The v5 pre-registration (commit \texttt{b3165a6}) specifies six experiments. Deviations from the registered protocols:
\begin{itemize}
\item \textbf{Experiment 1 (lambda sweep):} Pre-registered at $n{=}100$ per $\lambda$; executed at $n{=}20$ (primary pair only). Replication pair and budget-sensitivity arms deferred due to GCG compute cost (${\sim}$8 hours per $\lambda$ on Apple Silicon).
\item \textbf{Experiment 2 (sub-threshold):} Pre-registered at 500 null streams $\times$ 100 shifted streams; executed at 50 null $\times$ 20 shifted per mixing rate (DeBERTa only). Llama Guard arm deferred (Mac Studio scheduling).
\item \textbf{Experiment 3 (heterogeneous):} Pre-registered at $n{\geq}5{,}000$ prompts; executed at $n{=}3{,}768$ (after filtering). DeBERTa only; other classifiers deferred.
\item \textbf{Experiment 4 (PAIR):} Pre-registered at 100 prompts $\times$ 100 iterations; executed at 30 prompts $\times$ 20 iterations. Nested budget ablation not performed.
\item \textbf{Experiment 5 (SMSR, Smoothed Model for Safe Responses, a randomised-smoothing certification):} Decision rule triggered ``excise to appendix'' (DeBERTa SMSR not run). The RoBERTa-only preliminary results were then removed with the randomised-smoothing appendix (correction 5) rather than retained as an exploratory appendix.
\item \textbf{Experiment 6 (reproducibility):} \texttt{make reproduce} implemented and cross-hardware validation completed. Raw data are committed under \texttt{results/} in the repository.
\end{itemize}
Several deviations also reduce scope (the sub-threshold and heterogeneous experiments are DeBERTa only; the $\lambda$ sweep covers the primary pair only). The analysis procedures differ from those registered beyond sample size: ordinary least squares and a Spearman correlation over five block rates replace the registered Deming regression and Jonckheere--Terpstra test; the confidence-stratified analysis covers only the passive-transfer data; the PAIR semantic-preservation judge was not run; the sub-threshold experiment draws shifted scores from a clipped normal distribution and starts every detector at the known onset; and the eight factorial comparisons were fixed after the factorial results were available (Amendment~2). The registered documents are unchanged. \textbf{Because the executed sample sizes are substantially below their pre-registered targets, v5 results should be read as \emph{exploratory} (hypothesis-generating), not confirmatory} (Table~\ref{tab:prereg-vs-actual}).

\begin{table}[h]
\centering
\small
\caption{Pre-registered vs.\ executed sample sizes for v5 experiments. Three of the four arms ran at 20\% or less of their registered budget and the heterogeneous-traffic arm at 75\%; results are exploratory.}
\label{tab:prereg-vs-actual}
\begin{tabular}{llll}
\toprule
\textbf{Experiment} & \textbf{Pre-registered} & \textbf{Executed} & \textbf{Ratio} \\
\midrule
$\lambda$ sweep (Exp.~1) & $n{=}100$ per $\lambda$ & $n{=}20$ & 20\% \\
Sub-threshold (Exp.~2) & 500 null $\times$ 100 shifted & 50 $\times$ 20 & 10--20\% \\
Heterogeneous (Exp.~3) & $n{\geq}5{,}000$ prompts & $n{=}3{,}768$ & 75\% \\
PAIR (Exp.~4) & 100 prompts $\times$ 100 iter. & 30 $\times$ 20 & 6\% \\
\bottomrule
\end{tabular}
\end{table}

\subsection{Corrections Since the Previous Public Version}
\label{sec:corrections}

This version makes nineteen corrections to the previous public arXiv version. Appendix~\ref{app:corrections} lists all of them, with the two detector-code fixes, the model-identity correction and the audit of results without a committed generator. Those that change numbers or claims: the adversarial token-budget threshold is re-measured ($T_{50,\text{adv}}=117$ tokens, a $2.5\times$ shift; the earlier $154$ tokens, $3.3\times$, is withdrawn); the primary encoder is the fine-tuned DeBERTa-v3-base (768-d), not DeBERTa-v3-large; two detector-code defects (KS tie handling, a non-predictable bet in the growing-window betting test) are fixed and all 800 factorial cells regenerated (control-validated detection $86.6\%\to86.8\%$, mean latency at $w{=}100$ $39.5\to39.7$ steps, $\eta^2_{\text{interaction}}$ $0.185\to0.186$; Appendix~\ref{app:regeneration}); the betting-test-versus-KS comparison at low mixing, the newer-frontier-model detection and false-positive rates, and the confidence-gating contrast are withdrawn; and the deployment Pareto table, the prefix-mutation count, the cost--detection correlation, and the passive-transfer and joint-optimisation counts (restated on prompts first scored unsafe) are corrected. The headline benchmark and canary findings keep their direction.

\subsection{Reproducibility}
\label{sec:reproducibility}

Code, configurations, pre-registration document, and raw results are available at \url{https://github.com/junwenleong/safety-classifier-shift-monitor}. The pre-registration is anchored at commit \texttt{0380167}; all hyperparameters are specified in version-controlled YAML files under \texttt{configs/}. Encoder fine-tuned checkpoints are not committed due to size; \texttt{scripts/finetune\_deberta.py} and \texttt{scripts/finetune\_text\_moderation.py} reproduce them from WildGuardMix with fixed seeds (${\sim}$2 hours each on Apple Silicon). A verification script (\texttt{scripts/verify\_paper\_numbers.py}) checks 225 selected reported values against the committed result files, spanning the original factorial, the conformal evaluation (Table~\ref{tab:conformal}, verified against \texttt{results/conformal\_full.json}), and the post-hoc numbers (extended token sweep and no-shift control). It does not re-run experiments or validate their design, and \texttt{make reproduce} runs one experiment and the verifier, not every table and figure. Shifted corpora were generated offline with fixed seeds; the corpus files under \texttt{data/shifted/} are git-ignored, so the public repository holds their hashes and per-prompt scores, not the texts; the one exception is \texttt{data/shifted/adversarial\_suffix/deberta\_suffixes\_gate\_a.jsonl}, which holds 100 harmful prompts with their GCG suffixes. The file \texttt{paper/latex/smsr\_connection.tex} is a withdrawn appendix and is not part of the paper.

Corrections to the previous public version and results that cannot be regenerated from the repository alone are listed in Appendix~\ref{app:corrections}.

\section{Results}
\label{sec:results}

\subsection{RQ1: Detection Performance}
\label{sec:detection}

The system detects shift in 694 of 800 cells (86.8\% control-validated detection, 95\% Wilson CI [0.842, 0.889]), with empirical false alarm rates of 2--10\% across classifiers. ``Control-validated detection'' requires both a shift-stream alarm (latency $\geq 0$) and a quiet paired negative-control stream. No shifted stream failed to alarm: all 800 eventually did. The 106 cells that are not control-validated alarmed before onset (68) or had a paired control that also alarmed (45; 7 both), so they are specificity failures, not missed shifts (raw TPR 91.5\%; the 4.75 pp gap comes from the 45/800 cells (5.6\%) whose paired control also alarms, not all incremental; Appendix~\ref{app:tpr-far}). Values in this section are from the regeneration with the corrected KS detector (Appendix~\ref{app:regeneration}); no cell-mean latency moves by more than 1.5 steps relative to the original run, and the adversarial-suffix corpus is the one currently held, which differs slightly from the original (Appendix~\ref{app:corrections}).

\begin{table}[t]
\caption{Mean detection latency (steps) with 95\% bootstrap CIs. Each combination has $N{=}20$ seeds $\times$ 2 window sizes $= 40$ cells; means are over the control-validated detections (counts in Appendix~\ref{app:detection-rates}).}
\label{tab:latency}
\centering
\small
\begin{tabular}{lccccc}
\toprule
\textbf{Classifier} & \textbf{Paraphrase} & \textbf{Code-switch} & \textbf{Compositional} & \textbf{Temporal} & \textbf{Adversarial} \\
\midrule
DeBERTa & 28.4 [26, 31] & 32.2 [30, 35] & 24.5 [22, 27] & 23.5 [22, 25] & 36.5 [33, 40] \\
Text-Mod. & 34.6 [32, 37] & 33.2 [31, 36] & 29.6 [28, 32] & 24.8 [23, 26] & 25.3 [23, 27] \\
Llama Guard & 69.6 [63, 76] & 93.9 [82, 108] & 47.2 [43, 52] & 42.4 [39, 45] & 27.9 [26, 30] \\
ShieldGemma & 85.6 [76, 95] & 83.3 [74, 92] & 43.9 [40, 48] & 27.6 [26, 29] & 27.7 [26, 29] \\
\bottomrule
\end{tabular}
\end{table}

The table reveals a crossover interaction. Paraphrase is easy for encoders (28--35 steps) but hard for decoders (70--86 steps). Adversarial suffix is the hardest condition for DeBERTa (36.5) but the easiest for Llama Guard (27.9). This motivates RQ3. The slowest cell is Llama Guard $\times$ code-switch (93.9 steps [82, 108]); one possible explanation, which we did not test, is that these scores are largely insensitive to surface-level script changes, so many shifted examples are needed before the score distribution diverges detectably.

\begin{figure}[t]
\centering
\includegraphics[width=0.85\textwidth]{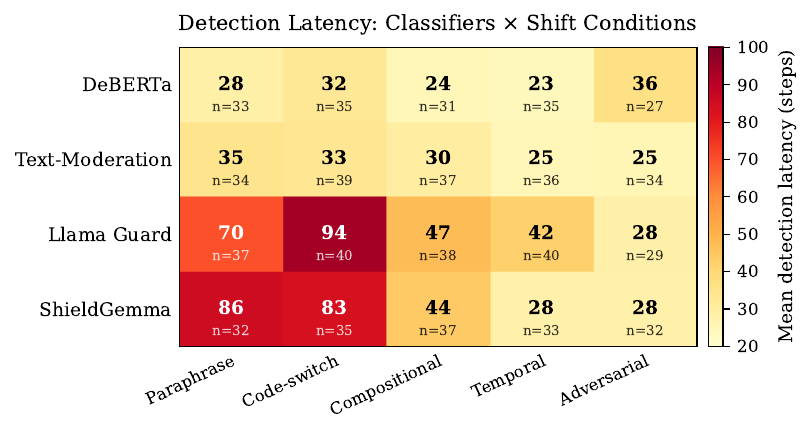}
\caption{Detection latency heatmap (classifier $\times$ shift condition). Darker cells indicate slower detection. The crossover interaction is visible: DeBERTa detects paraphrase fast but adversarial suffix slow, and the decoders show the opposite pattern; Text-Moderation does not follow DeBERTa.}
\label{fig:heatmap}
\end{figure}

\paragraph{Window size.} $w{=}100$ detects faster (means 39.7 [37.3, 42.5] vs 45.6 [42.9, 48.4]; paired difference over the $n{=}315$ cells valid at both window sizes: $-7.1$ [$-8.7$, $-5.4$] steps) at the cost of slightly higher false alarm rates.

\paragraph{False alarm rates (95\% Wilson CIs).} Text-Moderation 2.0\% [0.8\%, 5.0\%] $<$ Llama Guard 3.0\% [1.4\%, 6.4\%] $<$ ShieldGemma 8.0\% [5.0\%, 12.6\%] $<$ DeBERTa 9.5\% [6.2\%, 14.4\%]. The $5\times$ spread is not explained by the width of the null score distribution (Figure~\ref{fig:null-kde}): DeBERTa ($\sigma{=}0.087$) has the highest FAR and Llama Guard ($\sigma{=}0.144$) the second lowest. Each threshold is the 97th percentile of the maxima of only 50 null streams, a noisy order statistic (\S\ref{sec:limitations}), and we did not test whether threshold noise accounts for the spread.

\begin{figure}[t]
\centering
\includegraphics[width=0.85\textwidth]{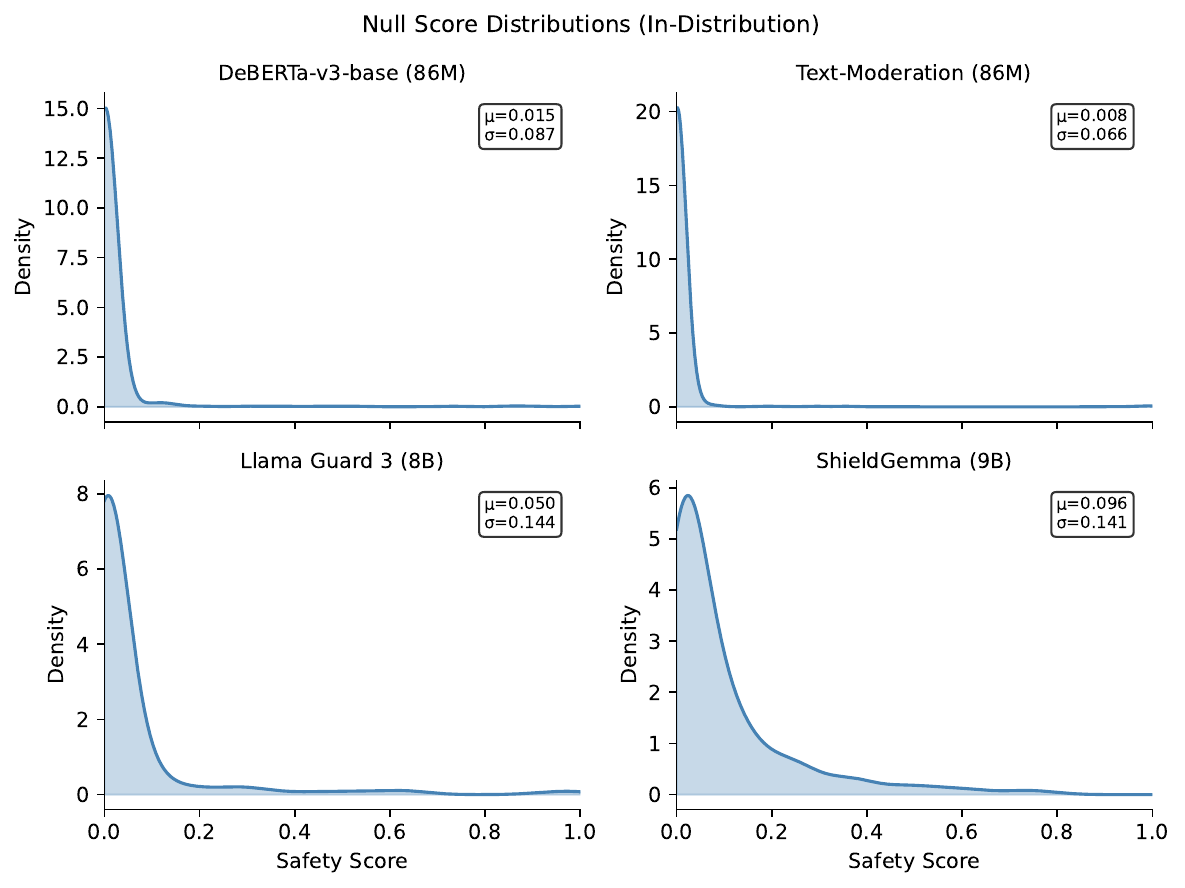}
\caption{Null score distributions (in-distribution, $n{=}500$ per classifier). Text-Moderation and DeBERTa produce tightly concentrated scores near zero; Llama Guard and ShieldGemma show wider distributions with heavier tails. The $5\times$ FAR spread across classifiers does not follow this ordering (see text).}
\label{fig:null-kde}
\end{figure}

\paragraph{Failure analysis.} The 106 invalid cells concentrate in DeBERTa $\times$ adversarial-suffix (13/40, 32.5\%) and Llama Guard $\times$ adversarial-suffix (11/40, 27.5\%). An alarm before onset or on a control stream cannot be caused by shifted inputs that have not yet arrived, so these failures are properties of the threshold and the reference stream; why they concentrate in the adversarial-suffix condition is not identified. The Regime~C result (\S\ref{sec:regime-c}) is a separate observation about post-onset scores and does not explain these pre-onset alarms.

\subsection{RQ2: Conformal Adaptation}
\label{sec:conformal-results}

\noindent\textit{Note:} Density-ratio estimation is known to degrade as the feature dimension grows, because the estimated ratios have large variance~\citep{stojanov2019lowdim}. The collapse diagnosed below is an extreme case: with embeddings of this dimension a logistic classifier separates source from target perfectly (our observation, not a result of that reference), so every importance weight sits at the clip floor. The contribution of this section is not the mechanism but (a)~documenting that it silently disables the standard post-shift recovery method on the decoder-based safety classifiers most widely deployed, and (b)~showing that PCA to $\leq$32 dimensions restores adaptation.

Evaluated on all four classifiers and three shift types (temporal, paraphrase, adversarial suffix).

\begin{table}[t]
\caption{Conformal adaptation across classifiers and shift types (nominal coverage $0.90$; calibration set all label~0, post-shift examples all label~1, so each gap mixes covariate shift with a label-composition change; the weighted layer is a heuristic, not a coverage guarantee). Pre-shift coverage is $0.91$ (DeBERTa), $0.88$ (Text-Mod.), $0.89$ (Llama Guard) and $0.915$ (ShieldGemma). Gap is pre-shift minus post-shift coverage of the unweighted set; Rec. is weighted-on-alarm minus unweighted post-shift coverage (DeBERTa, paraphrase: unweighted post-shift coverage $0.91-0.515=0.395$, weighted $0.395+0.390=0.785$). ESS $= (\sum_i w_i)^2 / \sum_i w_i^2$; in this context ESS near the maximum ($n_{\text{cal}}{=}300$) indicates \emph{degenerate} uniformity (all calibration weights floor-clipped to $1/C_{\text{clip}}$) rather than successful adaptation, so high ESS is the bad case here. $n{=}200$ post-shift examples per cell.}
\label{tab:conformal}
\centering
\small
\begin{tabular}{l ccc ccc ccc}
\toprule
& \multicolumn{3}{c}{\textbf{Temporal}} & \multicolumn{3}{c}{\textbf{Paraphrase}} & \multicolumn{3}{c}{\textbf{Adversarial suffix}} \\
\cmidrule(lr){2-4} \cmidrule(lr){5-7} \cmidrule(lr){8-10}
\textbf{Classifier} & Gap & Rec. & ESS & Gap & Rec. & ESS & Gap & Rec. & ESS \\
\midrule
DeBERTa & .085 & +.160 & 88 & .515 & +.390 & 46 & .325 & +.020 & 206 \\
Text-Mod. & .090 & +.020 & 300 & .515 & +.005 & 300 & $-.020$ & +.000 & 300 \\
Llama Guard & .350 & +.020 & 300 & .715 & +.015 & 300 & .260 & +.100 & 300 \\
ShieldGemma & .225 & +.075 & 300 & .725 & +.005 & 300 & .445 & +.060 & 300 \\
\bottomrule
\end{tabular}
\end{table}

The table reveals three findings. First, density-ratio collapse (ESS${\approx}$300, indicating uniform floor-clipped weights) is consistent across all three shift types for Text-Moderation, Llama Guard, and ShieldGemma. The collapse is a high-dimensional separability artefact, not a shift-specific phenomenon. Second, DeBERTa is the sole classifier where weighted conformal achieves meaningful recovery, but its effectiveness shows a gradient across shift types: paraphrase (ESS${=}$46, recovery${=}$+0.390) $>$ temporal (ESS${=}$88, recovery${=}$+0.160) $>$ adversarial suffix (ESS${=}$206, recovery${=}$+0.020). Note that recovery can exceed the coverage gap (e.g., DeBERTa/temporal: gap${=}$0.085, recovery${=}$+0.160), so weighted coverage can exceed pre-shift coverage (for DeBERTa on temporal shift, $0.985$ against $0.91$ before the shift). We did not measure prediction-set sizes or abstention rates, so the cost of this extra coverage is untested; one possible reading is that density-ratio reweighting overestimates the shift magnitude. Third, DeBERTa on adversarial suffix shows near-total collapse (92.7\% of weights at floor, max weight 0.82) despite not reaching the binary collapse threshold; the GCG suffixes, optimised against DeBERTa specifically, push its embeddings into a region nearly perfectly separable from reference.

\paragraph{Mechanism.} The logistic density-ratio estimator separates source and target perfectly in all collapsed cases, assigning $\hat{P}(\text{target} \mid x) \approx 0$ to every calibration point, so all weights clip to the floor $1/C_{\text{clip}} = 0.1$ and data-driven reweighting disappears. The residual $\leq$+0.10 recoveries at ESS$\approx$300 are at least partly the test-point term: the implicit weight $w(X_{n+1})=1.0$ raises the effective quantile level from $90.3\%$ to $93.0\%$ at $n_{\text{cal}}=300$, $\varepsilon=0.1$, a formula artefact and not adaptation. Within DeBERTa, ESS tracks recovery ($46$, $88$, $206$ give $+0.390$, $+0.160$, $+0.020$), but not across classifiers: the floor-collapsed generative classifiers (ESS${\approx}$300) recover between $0.000$ and $+0.100$ (Llama Guard on adversarial suffix, $+0.100$).

\paragraph{Dimensionality reduction as diagnostic.} To confirm the collapse is a dimensionality artefact, we applied PCA to 32 dimensions before density ratio estimation. This breaks the collapse for both generative classifiers on temporal shift: Llama Guard recovers 33.0 pp (coverage 0.555 $\rightarrow$ 0.885, ESS drops from 300 to 19.6), and ShieldGemma recovers 20.5 pp (0.690 $\rightarrow$ 0.895, ESS drops from 300 to 85.0). At 32 dimensions, 82--91\% of embedding variance is retained, but the logistic classifier can no longer achieve perfect separability. At 64 dimensions, ShieldGemma re-collapses (ESS${=}$300) while Llama Guard remains fixed (ESS${=}$135); at 128 dimensions, both re-collapse. The critical threshold is ${\leq}$32 dimensions for these embedding spaces. DeBERTa (768-d, no baseline collapse) shows no measured coverage degradation under this PCA intervention; the experiment does not identify which discarded dimensions carry safety-relevant signal. This is consistent with the collapse being driven by dimension relative to sample size, the regime in which density-ratio estimation is unreliable~\citep{stojanov2019lowdim}.

\paragraph{No-shift control.} To test whether the collapse is an $n \ll d$ interpolation artefact and not a shift signal, we ran a source-vs-source control: in-distribution reference embeddings were randomly split into pseudo-source ($n{=}300$) and pseudo-target ($n{=}200$) with \emph{no} shift present, and the same estimator was applied (8 random splits). Density-ratio hyperparameters, fixed throughout the paper: logistic regression, \texttt{solver=lbfgs}, $C_{\text{LR}}{=}1.0$, $\ell_2$ penalty, \texttt{max\_iter=1000}; weights $\exp(\text{decision})$ clipped to $[1/C_{\text{clip}}, C_{\text{clip}}]$ with $C_{\text{clip}}{=}10$; no feature standardisation; when PCA is applied it is fitted on the combined pseudo-source and pseudo-target set and the estimator refit in the reduced space. On Text-Moderation (768-d) the control reproduces the collapse \emph{direction}: the source-vs-target classifier attains train accuracy 0.94 but \emph{chance} 5-fold held-out AUC (0.53), ESS falls to 147/300 with 30\% of weights at the floor, and PCA removes the floor-clipped weights (30\% at 768-d, 0\% at 8-d) as under real shift, with ESS rising towards uniformity (147 to 277). Under real shift ESS moves the other way (300 to 19.6 for Llama Guard and 85.0 for ShieldGemma at 32 dimensions) because PCA there exposes a genuine density-ratio signal. The collapse is \emph{partial} here (ESS 147/300, 30\% at floor) against near-total floor-clipping under real shift (Table~\ref{tab:conformal}: ESS ${\approx}$ 300/300, ${\sim}$100\% at floor), because the only driver present is $n \ll d$ interpolation and not interpolation plus a distributional gap. For Text-Moderation the separability and its PCA reversal are therefore an interpolation artefact. The fine-tuned DeBERTa-v3-base classifier does not collapse under the control (ESS $\approx$ 289, 0\% at floor). Llama Guard and ShieldGemma were not run (their embeddings could not be extracted on the 24\,GB machine used), so for them the artefact interpretation rests on the observed dimension dependence only (\S\ref{sec:limitations}).

\subsection{RQ3: Variance Decomposition}
\label{sec:variance}

Table~\ref{tab:anova} gives the decomposition.

\begin{table}[t]
\caption{Two-way ANOVA on detection latency (694 control-validated detections). Main-effect 95\% CIs are percentile bootstrap (1000 resamples); interaction CI is from seed-clustered bootstrap (Appendix~\ref{app:anova-robustness}). Permutation $p$: 1000 permutations, floor $1/1001$.}
\label{tab:anova}
\centering
\small
\begin{tabular}{lccc}
\toprule
\textbf{Factor} & $\boldsymbol{\eta^2}$ & \textbf{95\% CI} & \textbf{Permutation $p$} \\
\midrule
Classifier & 0.248 & [0.210, 0.297] & $\leq 0.001$ \\
Shift type & 0.239 & [0.196, 0.296] & $\leq 0.001$ \\
Classifier $\times$ Shift & 0.186 & [0.166, 0.221] & $\leq 0.001$ \\
Residual & 0.328 & n/a & n/a \\
\bottomrule
\end{tabular}
\end{table}

All three systematic factors are significant (1000 permutations; all at the floor $p = 1/1001$). Classifier (24.8\%) and shift type (23.9\%) are the largest, with the complete-case interaction (18.6\%) close behind. The interaction magnitude is \emph{not} robust: it is a complete-case estimate on the 694 control-validated detections, and the censoring sensitivity analysis of \S\ref{sec:limitations} (imputing the 106 non-random exclusions, which concentrate in the adversarial-suffix cells) drives it to $0.02$--$0.03$. The robust claim is the qualitative crossover (Table~\ref{tab:latency}, Figure~\ref{fig:heatmap}); the variance share depends on how the invalid cells are treated.

The largest interactions: DeBERTa $\times$ adversarial-suffix ($+21.4$ steps above expected), ShieldGemma $\times$ paraphrase ($+20.0$), Llama Guard $\times$ code-switch ($+17.8$), Llama Guard $\times$ adversarial-suffix ($-16.1$). Thresholds set from classifier-level or shift-level averages may therefore under-alert on hard pairings and over-alert on easy ones; we did not test such a configuration.

\paragraph{Replication stability.} The initial $N{=}5$ estimate (200 cells) put the interaction at $0.258$, the largest factor; at $N{=}20$ (800 cells) it is $0.186$ while the main effects grow (Appendix~\ref{app:replication}). The qualitative finding (a significant, substantial interaction) replicates; the claim that it ``dominates'' does not.

\subsection{Robustness Across Ground-Truth Regimes}
\label{sec:regimes}

The preceding results use Regime A: synthetic shift onset at a known step. We evaluate two additional ground-truth regimes.

\paragraph{Regime B (synthetic onset, real jailbreak examples).} Real jailbreak prompts from public red-team databases are inserted after a fixed onset at step 500 of a stream built from the same reference data as Regime~A; the harness does not read timestamps or build a chronological source/target split, and the examples are not shown to postdate the classifier's training. The monitor is not told when the onset occurs. Evaluated on DeBERTa and Llama Guard $\times$ temporal corpus $\times$ 5 seeds $\times$ 2 window sizes $= 20$ cells. The system detects shift in 17 of 20 cells (85\% detection rate, Wilson 95\% CI [64\%, 95\%]), with mean detection latency of 32.6 steps. Llama Guard achieves 10/10; DeBERTa achieves 7/10. Detection works on real jailbreak text in this synthetic-onset setting; this is not a test on a naturally occurring chronological shift.

\paragraph{Regime C (adversarial success on the target: the target's own monitor misses it).}
\label{sec:regime-c}
We filter GCG-optimised suffixes to the 22 examples that successfully flip DeBERTa's classification (one further example excluded post-validation: original score 0.002, already classified benign) and stream them through all four classifiers, giving a classifier-level demonstration that the same adversarial inputs produce opposite distributional effects across architectures. The harness truncates every text to 1000 characters, which removes the suffix from 4 of the 22 (see the corpus limitation below).

DeBERTa (target classifier): adversarial suffixes shift scores towards safe (post $\approx 0.01$ for the 18 untruncated examples), making attacked examples indistinguishable from the safe reference distribution. The monitor fails to detect in 38/40 seed-window combinations.

Llama Guard (non-target classifier): the same adversarial suffixes shift scores towards unsafe (post $= 0.78$ in the original run, reference mean $0.05$), producing a large upward shift ($+0.73$) that the monitor detects in 14/40 seed-window combinations.\footnote{Score-direction diagnostics were computed via \texttt{scripts/check\_regime\_c\_direction.py}. For this version the per-example scores were recomputed on the 2~June corpus (the May corpus was not retained): Llama Guard mean $0.77$ against a reference mean of $0.05$ ($+0.72$), stored in \texttt{results/regime\_c\_example\_scores.json}. The verification script checks only detection counts. Full KS time series for all four classifiers are in \texttt{results/regime\_c\_ks\_series.json} and can be reproduced without model inference.} A score direction analysis confirms all 14 detections are TOWARDS\_UNSAFE: the GCG suffixes did not transfer their adversarial effect to Llama Guard. They made inputs appear more dangerous, not safer.

\begin{figure}[t]
\centering
\includegraphics[width=0.85\textwidth]{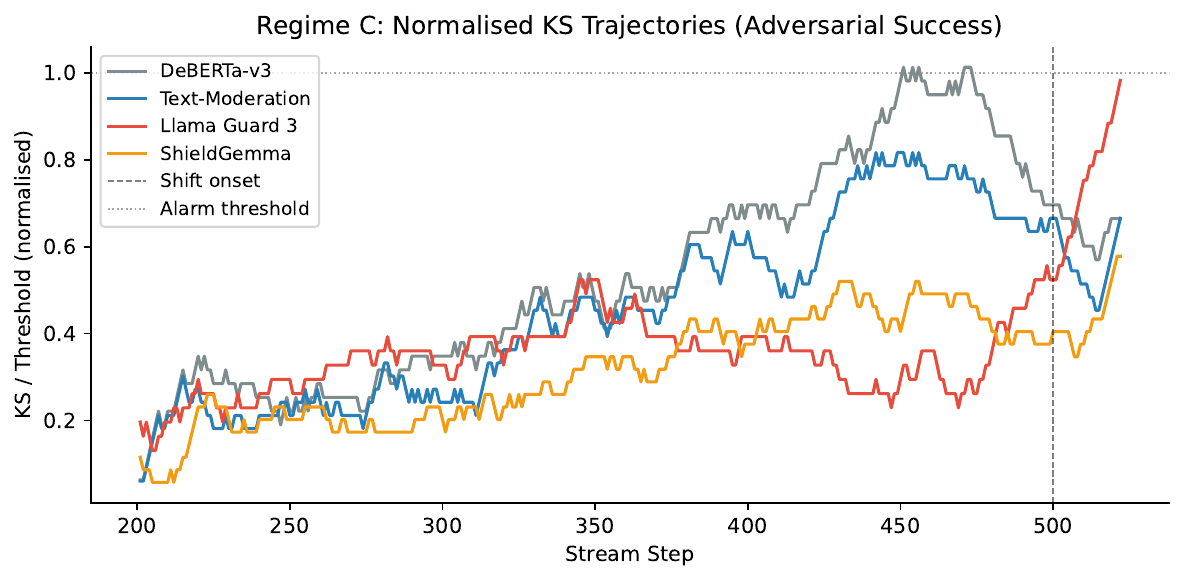}
\caption{Regime C: KS statistic trajectories normalised by per-classifier threshold ($y{=}1.0$ is the alarm line), one stored stream per classifier (seed 0, window size 200), not the aggregate over the 40 seed-window combinations. All classifiers hover near but below threshold on the 22-example corpus. In this stream DeBERTa exceeds threshold briefly (a false alarm before onset); Llama Guard's 14 detections of 40 occur in other seed-window combinations.}
\label{fig:regime-c-ks}
\end{figure}

This is consistent with cross-classifier anomaly detection and not with adversarial transfer: perturbations optimised against one classifier move the scores of architecturally different classifiers, which a score-distribution monitor can in principle pick up. With 22 examples and 14 non-independent detections this is a classifier-level observation, not an estimate of a detection rate.

\paragraph{Corpus limitation.} The Regime~C corpus comprises 22 successful adversarial examples. Llama Guard scores them unevenly (2~June corpus: median $0.97$, range $0.01$ to $1.00$, 18 of 22 above $0.5$, mean $0.77$), so the 0.78 above is a mean over examples. The 4 examples longer than the truncation (1304, 1317, 1579 and 1910 characters) lose their suffix before scoring, so DeBERTa scores them 0.75--1.00 as streamed and the streamed DeBERTa mean is $0.18$, not the $0.006$ recorded with intact suffixes; whether the May corpus had the same four cannot be checked. The 14 Llama Guard detections are the same 22 examples re-sampled across seeds, not independent shift events. At window size 200 with calibrated thresholds no classifier reliably detects shift from this corpus (Figure~\ref{fig:regime-c-ks}), so it sits at the margin of detectability; the finding is reproducible (full KS series stored) but classifier-level, and a larger corpus with varied attack patterns would be needed for a cell-level detection rate.

\subsection{Scan Martingale, Sub-Threshold Detection and Heterogeneous Traffic}
\label{sec:additional-detectors}

This subsection reports the evaluations of the detectors of \S\ref{sec:martingale} and the CUSUM/EWMA detectors that go beyond the KS factorial.

\paragraph{Detection power of the scan martingale.} Under instantaneous (step-onset) shift at full mixing, per-condition calibrated KS thresholds (97th-percentile tuned per shift type) match or exceed the scan martingale's detection rate. The scan martingale's contribution is \emph{operational}: a nominal false-alarm target $\alpha$ under an exchangeability argument (not proved here for repeated reuse of a frozen reference set) with no per-classifier null-threshold tuning ($\kappa$ and $M$ remain shared hyperparameters). Under a 50-step ramp the comparison depends on the score distribution and the classifier. On the cached DeBERTa $\times$ paraphrase scores the scan martingale detects somewhat more streams than KS at 15--20\% mixing (for example 93\% against 67--87\% at 20\%, depending on the KS threshold rule; Table~\ref{tab:ad1}). On a synthetic shifted pool the committed evaluation reverses this for DeBERTa (scan 40 and 80\% against KS 83 and 100\% at 15 and 20\%), and on the two decoder classifiers scan is well behind KS at every mixing rate below 50\% (for example Llama Guard 0, 13 and 57\% against KS 73, 97 and 100\% at 15, 20 and 30\%; ShieldGemma 7, 30 and 77\% against 67, 100 and 100\%). We therefore do not claim a detection-power advantage for the scan martingale; the claim we retain is the one above, a nominal false-alarm target with no per-classifier null-threshold tuning.

\paragraph{Monitor-aware adversary (AD1).} We simulate a ramped-onset adversary that gradually introduces shifted samples over a 50-step ramp, with shifted examples drawn from the cached DeBERTa $\times$ paraphrase pool (Table~\ref{tab:ad1}):

\begin{table}[h]
\centering
\small
\caption{Detection rate under a ramped-onset adversary (DeBERTa, 50-step ramp to the target mixing rate, shifted examples drawn from the cached DeBERTa $\times$ paraphrase pool, 30 streams per cell). KS rows use three threshold rules that differ only in how the null maximum is calibrated; the scan martingale uses $\kappa=0.3$, $M=50$ sub-martingales. Reproduced by \texttt{scripts/investigations/ad1\_ks\_threshold\_sensitivity.py}.}
\label{tab:ad1}
\begin{tabular}{lcccc}
\toprule
\textbf{Mixing rate} & \textbf{Scan} & \textbf{KS (500-step null)} & \textbf{KS (800-step null)} & \textbf{KS (pooled)} \\
\midrule
15\% & 60\% & 57\% & 37\% & 33\% \\
20\% & 93\% & 87\% & 73\% & 67\% \\
25\% & 100\% & 100\% & 97\% & 93\% \\
30\% & 100\% & 100\% & 100\% & 100\% \\
\bottomrule
\end{tabular}
\par\vspace{4pt}
{\footnotesize\textit{Note:} The three KS thresholds are the 97th percentile of the maximum KS statistic over 30 null streams of 500 steps (null false-alarm rate 4\%), over 30 null streams of 800 steps (2\%), and over 200 pooled null streams of 800 steps (2\%); the scan martingale's null false-alarm rate on 100 DeBERTa streams is 0\% (the 200-stream rates per classifier are in \S\ref{sec:martingale}). These null rates are not like for like: the KS rule checks for an alarm only at or after the onset step, while the scan martingale is run over the whole stream and a pre-onset alarm is not counted as a detection, so neither is a stream-level probability of ever alarming under the same rule. The scan martingale is ahead of KS at 15--20\% mixing under every rule but by 3--27 points, and the ordering depends on the KS threshold. The earlier version of this table (KS 3, 7, 10 and 47\% against scan 37, 67, 87 and 100\%) could not be reproduced (Section~\ref{sec:corrections}).}
\end{table}

\paragraph{CUSUM and EWMA in simulation.} Against a constant sub-threshold adversary (sustained low-rate injection, no ramp), sliding-window KS needs $\geq$20\% mixing for 80\% detection. The one-sided scan-martingale baseline implemented here detected constant-rate contamination at no tested rate (MDM~$>$30\%); it forms its conformal $p$-value from the lower tail only, so it is insensitive to upward shifts and says nothing about the two-sided scan martingale of Section~\ref{sec:martingale}. We evaluate CUSUM and EWMA calibrated at 5\% FAR on 50 null streams, tested on 20 shifted streams per mixing rate across $\{4\%, 6\%, 8\%, 10\%, 12\%, 15\%, 20\%, 30\%\}$ (DeBERTa $\times$ paraphrase, stream length 5{,}000).

\begin{table}[h]
\centering
\small
\caption{Minimum detectable mixing (MDM) at $\geq$80\% detection power, with thresholds calibrated to a 5\% FAR target. The shifted scores are drawn from a clipped normal distribution parameterised from DeBERTa paraphrase statistics (not resampled measured scores), every detector starts at the known onset, and the realised FAR on the validation null streams was 8\% for KS, 6\% for CUSUM ($k{=}0.5$) and 4\% for EWMA (smoothing weight $\lambda_{\text{E}}{=}0.05$), so the rows do not share one measured operating point. In this simulation CUSUM and EWMA detect down to 4\% contamination at the cost of increased latency. The ``One-sided scan baseline'' row is a separate hand-written implementation with a lower-tail $p$-value, which does not respond to upward score shifts; it is not the two-sided scan martingale of Section~\ref{sec:martingale}.}
\label{tab:subthreshold}
\begin{tabular}{lccc}
\toprule
\textbf{Detector} & \textbf{MDM} & \textbf{Latency at MDM} & \textbf{Latency at 12\%} \\
\midrule
KS ($w{=}100$) & 20\% & 365 steps & 1{,}421 steps \\
One-sided scan baseline ($\kappa{=}0.3$) & $>$30\% & n/a & n/a \\
CUSUM ($k{=}0.5$) & \textbf{4\%} & 931 steps & 88 steps \\
EWMA ($\lambda_{\text{E}}{=}0.05$) & \textbf{4\%} & 917 steps & 97 steps \\
\bottomrule
\end{tabular}
\end{table}

The trade-off is latency: CUSUM detects 4\% mixing at ${\sim}$930 steps (KS needs ${\sim}$365 steps at 20\%+), as weaker signals need longer observation horizons. At 12\% mixing (the previous blind-spot boundary) CUSUM reaches 100\% detection at 88 steps, faster than KS at the same mixing (1{,}421 steps, 55\% detection). The simulation motivates testing CUSUM alongside KS for persistent low-rate contamination that the per-window KS test misses, after validating its operating point on representative heterogeneous traffic (next paragraph, where the CUSUM evidence conflicts).

\paragraph{Sensitivity to heterogeneous benign traffic.} Production benign traffic spans several domains, each with a slightly different score distribution even when all prompts are safe. We built a 5-source heterogeneous null corpus (target mixture WildGuardMix 40\%, OpenAssistant 20\%, code 15\%, creative 15\%, multi-turn 10\%; after filtering prompts scored $>0.5$ the retained $n{=}3{,}768$ prompts are 1{,}971 WildGuardMix, 950 OpenAssistant, 74 code, 282 creative and 491 multi-turn, so code is about 2\% and WildGuardMix about 52\%) and scored it with the primary DeBERTa classifier. KS FAR inflates from 4.3\% to 23.4\% ($5.4\times$) because it detects \emph{any} distributional difference, and recalibrating thresholds on a representative production sample recovers controlled FAR (2.6\%). CUSUM's behaviour is \emph{not established}: on the run whose KS figures are cited above CUSUM saturates (FAR $\to$ 100\%, a calibration failure under the heavier-tailed mixture), whereas a separately labelled run of the same nominal experiment leaves it essentially unchanged ($6.1\% \to 6.6\%$); the two runs were not reconciled and we make no robustness claim for CUSUM. Creative writing is the main driver: 62\% of creative prompts trigger DeBERTa above the unsafe threshold (filtered as false-benign), and the remaining low-scoring ones still carry a distributional signature KS detects. KS must therefore be calibrated on traffic representative of deployment, not on homogeneous training data.

For long-horizon deployments where benign drift may violate exchangeability we recommend a rolling reference window (block-structured violations trigger alarms in our tests; there is no committed comparison of gradual benign drift across encoders and decoders), and sub-sampling at session boundaries in highly autocorrelated traffic, which may reduce within-session dependence but does not by itself establish exchangeability.

\subsection{Detection Channel Comparison: KS vs.\ Betting Test vs.\ MMD}
\label{sec:channel-comparison}

\paragraph{Betting test (growing window).} Here the betting test is that of \S\ref{sec:cs} (called ``CS'' in earlier versions), not the earlier Hoeffding rule. On a representative subset (4 classifiers $\times$ 3 shifts $\times$ 10 seeds $= 120$ cells) it achieves 120/120 detection with 0/40 false alarms on reference-only streams, from a re-run under the corrected predictable-betting update (\S\ref{sec:corrections}); the counts match the earlier run, the latencies do not. The 0/40 is an \emph{empirical} observation, not a verification of the bound~\citep{waudbysmith2024estimating}: its 95\% Clopper--Pearson upper bound is 8.8\%, which does not separate a 5\% false-alarm rate from zero; the 40 streams are the 500-step pre-shift segments of the detection streams; and the earlier Hoeffding-type rule also produced 0/40. The reference mean $\mu_0$ is estimated from the same 500 null scores that make up those segments. In a simulation resampling each classifier's 500 null scores as the population (500-step streams, 600 per cell) the false-alarm rate is 0.0--0.8\% when $\mu_0$ is known and 1.0--4.5\% when it is estimated from 500 draws (0.2--2.2\% from 2{,}000), so an estimated $\mu_0$ erodes the guarantee and a larger reference sample restores it. In the factorial KS freezes its reference from the first $w$ scores of each stream while the betting test gets the mean of 500 null scores, so the two had different reference information. With KS given the same 500 null scores (same calibration rule) on the same 120 cached streams, KS has the shorter mean latency in 8 of the 12 classifier $\times$ shift cells and the betting test in 4 (ratio 0.76--1.69); the pooled mean is 27.5 steps for KS against 35.7 (\texttt{equal\_reference\_ks\_vs\_cs.py}). We therefore claim no latency advantage for the betting test on abrupt shifts.

The ramp and mixing sweep depends on the reference given to KS. As run (DeBERTa $\times$ paraphrase, 50-step ramp), KS freezes its reference from the first 100 scores of each stream with a threshold calibrated on windows of the pooled reference scores (0.292), while the betting test gets the mean of 500 null scores. At 30\% mixing the betting test then detects 10 of 10 cached streams and KS 4 of 10 (30 of 30 against 13 of 30 on the extended test, which uses 10 cached streams plus 20 resamples of the same pools, so is not 30 independent runs; no significance test is reported), and at 50--100\% mixing both detect 10 of 10 with the betting test 9--12 steps faster. When KS is instead given a reference frozen from the same 500 null scores with the same calibration rule (threshold 0.187), it detects 10 of 10 at 30\% mixing (30 of 30 extended) and is faster at every mixing level (71.5 against 122.6 steps at 30\%, 60.9 against 83.9 at 50\%, 45.1 against 54.2 at 100\%; \texttt{ramp\_ks\_reference\_size.py}). The betting test's low-mixing advantage in the earlier version is therefore explained by the smaller KS reference and not by accumulating evidence across the full history; it is withdrawn.

This is a single classifier $\times$ shift pair, the betting-test reference mean is in-sample, and neither detector's false-alarm rate was calibrated on streams independent of the reference. The data do not support preferring the betting test for weak signals or KS for strong ones; they support only that, given equal reference information, KS is at least as sensitive as the growing-window betting test on these streams.

\paragraph{Gradual drift.} Latency grows with ramp duration (DeBERTa $\times$ paraphrase, $n{=}10$ per condition): KS with a 100-score reference needs 94 steps at a 50-step ramp and 210 at a 200-step ramp (9/10 detected); the betting test 84 and 155 (10/10 at both). The KS figures are those of the weaker reference configuration above.

\paragraph{MMD on embeddings.} The MMD detector on penultimate-layer embeddings (4 classifiers $\times$ 3 shifts $\times$ 10 seeds, 1000-permutation bootstrap null at $\alpha{=}0.05$) detects 120/120 shifted streams at the first possible step (latency $= w{=}100$ in every cell). The script evaluates only full windows starting at the onset step, whereas the KS window can hold a mixture of pre- and post-onset scores (hence KS latencies of 23--93 steps in Table~\ref{tab:latency}), so the two latencies are not on the same footing. Empirical FAR is DeBERTa 3.3\%, Text-Moderation 3.3\%, ShieldGemma 0.0\% and Llama Guard 10\% (30 null windows each; the wider Llama Guard embedding distribution gives less stable null estimates at this size).

MMD therefore shows no latency gradation across shift types or classifiers, unlike KS, whose variation reflects the crossover. The detectors answer different questions: KS latency shows \emph{how quickly} the score distribution diverges (we did not validate it as a severity measure); MMD gives a binary alarm for \emph{whether} the embedding geometry changed at all, with no formal guarantee. A deployment could pair KS for graded alerts with MMD as a high-sensitivity binary alarm; we evaluated them separately.

\section{Canary Detection: Score Disagreement Under Targeted Evasion}
\label{sec:canary}

We expand the 22-example canary observation from \S\ref{sec:regime-c} into a full characterisation of score-disagreement monitoring. The target is the fine-tuned DeBERTa-v3-base classifier (the primary encoder of the factorial evaluation); the canary classifiers are Text-Moderation, Llama Guard~3, and ShieldGemma. We evaluate on $n=100$ GCG-optimised prompts, of which 49 end below DeBERTa's 0.5 threshold after attack (47 are genuine unsafe$\rightarrow$safe crossings; see \S\ref{sec:canary-setup}).

\paragraph{Note on terminology.} ``Steps'' in this section refer to GCG discrete token-search iterations on a single prompt (the attacker's optimisation budget), not production-stream timesteps monitored by the martingale (\S\ref{sec:cs}).

\subsection{Setup and Scoring Protocol}
\label{sec:canary-setup}

Of $n = 100$ GCG attempts, 49 end below the 0.5 threshold after attack; 47 of these are genuine unsafe$\rightarrow$safe crossings (2 of the 49 had original DeBERTa scores already below 0.5, so the attack-success count is 47/100). We retain the 49 post-attack-below-threshold inputs as the evaluation set (the canary is scored on whatever the primary now passes); detection rates below use this denominator and are within 2 pp of the 47-crossing one. Each is scored on all four classifiers. \emph{Divergence} for a pair $(A, B)$ is $|f_B(x_{\text{adv}}) - f_A(x_{\text{adv}})|$ with $A$ the target and $B$ the canary. A local canary \emph{detects} the attack when divergence exceeds a threshold set at the 95th percentile under clean inputs (the \emph{disagreement alarm}). The frontier-LLM canaries of \S\ref{sec:llm-canary} use a different rule, the \emph{rescue alarm}: the canary scores the input at or above 0.5 while the primary scores it below 0.5; this is not a divergence threshold (a primary at 0.49 and a canary at 0.51 satisfy it), so the two rates are not directly comparable.

GCG suffixes and random-token controls both use exactly 20 tokens (\texttt{SUFFIX\_LEN=20}), so the divergence difference below is attributable to the optimisation and not to perturbation size.

\subsection{Attack-Specificity (CA6)}
\label{sec:ca6}

Score disagreement is not a generic out-of-distribution signal. We compare GCG-optimised suffixes against random-token suffixes of identical length (20 tokens) appended to the same prompts.

\paragraph{Result.} GCG attacks produce a detection rate of 75.5\% (37/49, Wilson 95\% CI [0.619, 0.854]) on the DeBERTa$\leftrightarrow$Llama Guard pair (the cross-family pair used as the headline because it also has the random-suffix control; the best single canary reaches 93.9\%, \S\ref{sec:k-scaling}), versus 12.2\% (6/49, Wilson 95\% CI [0.057, 0.242]) for random-token controls. The mean divergence under GCG is 0.784 versus 0.213 under random tokens (Mann--Whitney $U$, $p = 2.37 \times 10^{-12}$).

The signal is attack-specific: gradient-optimised perturbations produce different score-disagreement signatures from random noise of the same length. The canary alarms far less often under random suffixes (12.2\%) than under GCG (75.5\%), so it is not merely sensitive to gibberish.

\subsection{Target-Specificity}
\label{sec:target-specificity}

We evaluate all 6 classifier pairs to test whether the divergence signal is specific to pairs involving the targeted classifier.

\paragraph{Result.} Three of six pairs show statistically significant GCG-vs-random differences (two-sided Mann--Whitney, Bonferroni-corrected $\alpha = 0.05/6 = 0.0083$): DeBERTa$\leftrightarrow$Text-Moderation ($p = 3.67 \times 10^{-17}$), DeBERTa$\leftrightarrow$Llama Guard ($p = 2.04 \times 10^{-12}$ here, $2.37 \times 10^{-12}$ in \S\ref{sec:ca6}; the two stored runs differ slightly in the mean divergence of the random-suffix control, 0.2137 against 0.2132, and we did not trace the cause), and DeBERTa$\leftrightarrow$ShieldGemma ($p = 0.0023$; survives correction). The three non-significant pairs (Text-Moderation$\leftrightarrow$Llama Guard, $p = 0.320$; Text-Moderation$\leftrightarrow$ShieldGemma, $p = 0.715$; Llama Guard$\leftrightarrow$ShieldGemma, $p = 0.862$) all exclude the target classifier. The canary effect is target-specific: it manifests in pairs that include the attacked model.

\subsection{No Clear Architecture Effect (CA4)}
\label{sec:ca4}

The prevailing assumption in multi-model defence literature is that architectural diversity fortifies ensembles against evasion. We test this by comparing within-family and cross-family divergence.

\paragraph{Result.} A one-way ANOVA on divergence with architecture type (within-family vs.\ cross-family) as the factor yields $\eta^2 = 0.011$, $F = 3.30$, $p = 0.070$. Architecture type explains 1.1\% of divergence variance. The highest mean divergence in our evaluation (0.899) occurs \emph{within}-family (DeBERTa$\leftrightarrow$Text-Moderation), not cross-family.

We find no clear advantage of cross-family over within-family pairs for passive disagreement strength in this four-model evaluation: every tested un-targeted classifier, including same-family checkpoints, produced a strong divergence signature under single-target optimisation. The six pairs reuse the same four models and prompts, so they are not six independent replications; $p=0.070$ is not evidence that architecture is irrelevant, the cause of the signal is not identified, and architecture diversity may still matter for transfer robustness (\S\ref{sec:transfer}).

\subsection{$k$-Classifier Scaling}
\label{sec:k-scaling}

We evaluate detection rate as a function of the number of canary classifiers $k \in \{1, 2, 3\}$, using an OR rule (alarm if any canary detects).

\paragraph{Result.} The best single canary (Text-Moderation) achieves 93.9\% (46/49). The best $k{=}2$ combination (Text-Moderation + Llama Guard) achieves 98.0\% (48/49, Wilson 95\% CI [89.3\%, 99.6\%]). Adding a third canary ($k{=}3$: all three) provides no additional detections (48/49 = 98.0\%). Within our evaluated range $k \in \{1, 2, 3\}$, $k{=}2$ gives the highest detection we saw (98\%) with no improvement from $k{=}3$; the best pair adds two detections over the best single canary (48/49 against 46/49), and selection at a fixed ensemble FPR is untested (below).

\paragraph{Multiplicity caveat.} The OR rule is not false-positive-rate-neutral: with per-canary false-alarm rates $\phi_i$, the ensemble rate is $1 - \prod_i (1 - \phi_i)$ if false alarms are independent (not guaranteed) and lies between $\max_i \phi_i$ and $\sum_i \phi_i$ in general, so it rises with $k$. The detection gains above are at fixed per-canary thresholds; a deployment targeting a fixed ensemble FPR must lower each canary's significance level as $k$ grows (Bonferroni $\alpha/k$ under arbitrary dependence, or \v{S}id\'{a}k $1-(1-\alpha)^{1/k}$ for independent canaries), which raises its score threshold. We do not re-tune per-canary thresholds here, so the reported $k$-scaling is an upper bound on detection at an unadjusted ensemble FPR.

\subsection{Template Attacks (AutoDAN)}
\label{sec:autodan}

Fine-tuned classifiers are robust to template-based jailbreaks. In our evaluation, AutoDAN-style template attacks achieve a 1\% flip rate against DeBERTa (the classifier has seen similar patterns during WildGuardMix training). This makes gradient-optimised attacks (GCG) the strongest attack we tested against fine-tuned safety classifiers; it is not the only possible one, as stronger adaptive methods (TAP, genetic-algorithm optimisation, multi-query injection) are untested here.

\paragraph{Iterative semantic attacks (PAIR-style).} To test whether an adaptive semantic attacker, rather than template or gradient attacks, can evade the canary pipeline, we implement a PAIR-style~\citep{chao2023jailbreaking} iterative attack: an attacker LLM (GPT-5.1, $T{=}1.0$) generates paraphrases of 30 harmful prompts over 20 iterations per prompt, receiving DeBERTa's score as feedback and attempting to minimise it below 0.5. Results ($N{=}30$, 600 total attack iterations): the PAIR attacker evades DeBERTa on only 1/30 prompts (3.3\%; best score 0.47), and on that single evasion the LLM canary (gpt-4o-mini) scores 1.0, perfect detection. Zero semantic transfer attacks succeed across all 30 prompts. With only one PAIR-style evasion there is little precision about canary performance conditional on semantic evasion. A plausible reading is that DeBERTa is hard to evade semantically (1\% template, 3.3\% PAIR) and that a paraphrase which evades it keeps harmful intent the LLM canary recognises. This complements the GCG results: gradient attacks evade DeBERTa reliably (47\%, \S\ref{sec:canary-setup}) and an un-targeted canary catches 75.5\% of those evasions (\S\ref{sec:ca6}); semantic attacks barely evade it (3.3\%) and the canary caught the one that did.

\subsection{Tokeniser Fragmentation}
\label{sec:tokeniser}

Cross-family classifier pairs (encoder $\leftrightarrow$ decoder) use different tokenisers. We measure the tokenisation ratio: the number of tokens produced by the canary's tokeniser divided by the number produced by the target's tokeniser, for the same token sequence; the sequences are 100 random DeBERTa token sequences, not the GCG suffixes themselves.

\paragraph{Result.} The mean tokenisation ratio is $1.729 \pm 0.161$ ($n = 100$). A 20-token DeBERTa sequence maps to approximately 35 tokens under Llama Guard's tokeniser. This fragmentation offers a hypothesised mechanism (not tested) for the observed cross-family transfer barrier: a single token substitution in the target's vocabulary shifts downstream subword boundaries in the canary's tokenisation, desynchronising the discrete search space (\S\ref{sec:transfer}).

\section{Adversarial Robustness of Score-Disagreement Monitoring}
\label{sec:adversarial-robustness}

Having established that score disagreement detects single-target attacks (\S\ref{sec:canary}), we now ask: can an adaptive attacker suppress the disagreement signal? We characterise attacker outcomes across a 4-tier threat model; the Tier~4 stall is a property of one penalised objective and one search procedure, not a boundary on every attacker.

The four tiers separate two knowledge axes: whether the attacker knows the canary's identity, and whether it has gradient access to the canary's weights. Black-box probing can reveal that a canary exists without giving optimisation access to its weights or routing thresholds; Tiers~1--2 cover identity-ignorant and identity-aware but weight-blind adversaries, Tiers~3--4 those with full or partial weight access.

\subsection{Threat Model}
\label{sec:threat-model}

\begin{table}[t]
\caption{Threat model tiers for score-disagreement monitoring. Tiers~1 and 2 report detection rates for the DeBERTa-v3-base target against Text-Moderation and Llama Guard canaries; Tiers~3 and 4 report blocked rates (the target score stays at or above 0.5), and Tier~2's passive-transfer count and Tiers~3--4 use checkpoints of the DeBERTa-v3-large epoch sweep. The two metrics differ and the outcomes use different denominators and criteria, so the rows are not a single robustness curve (a higher blocked rate at Tier~4 than Tier~3 does not mean robustness improves with a stronger attacker).}
\label{tab:threat-tiers}
\centering
\small
\begin{tabular}{lp{0.24\textwidth}p{0.30\textwidth}p{0.22\textwidth}}
\toprule
\textbf{Tier} & \textbf{Attacker knowledge} & \textbf{Empirical outcome} & \textbf{Deployment guidance} \\
\midrule
1 & Monitor-ignorant (single-target GCG) & 98\% detection (48/49) & $k{=}2$ canaries gave the highest detection \\
2 & Knows canary exists (hopes for transfer) & 93.9\% detection (46/49, best single canary); passive transfer on the epoch checkpoints: 2 of 28 eligible observations attack-induced & Confidence not shown to predict transfer \\
3 & Has canary weights (joint optimisation) & 45\% blocked (9/20, $n{=}20$) & Cross-family canary \\
4 & Stealth-aware (divergence minimisation) & 70\% blocked [48\%, 85\%] ($n{=}20$)\textsuperscript{$\dagger$} & Divergence-penalty stall \\
\bottomrule
\end{tabular}
\vspace{2pt}

{\footnotesize \textsuperscript{$\dagger$}Primary pair only. Replication pair (epoch-3 against epoch-10, no original scores recorded): 30\% blocked (3/10).}
\end{table}

Tier~1 is the baseline of \S\ref{sec:canary}, which uses the DeBERTa-v3-base primary of the factorial. Passive transfer, joint optimisation and divergence minimisation instead use checkpoints of a DeBERTa-v3-large epoch sweep (1024-d; Appendix~\ref{app:monitorability}), because they need gradient access to two models of the same architecture. ``DeBERTa'' alone means the base primary.

\subsection{Transfer Analysis and Canary Confidence}
\label{sec:transfer}

\begin{figure}[t]
\centering
\includegraphics[width=0.48\textwidth]{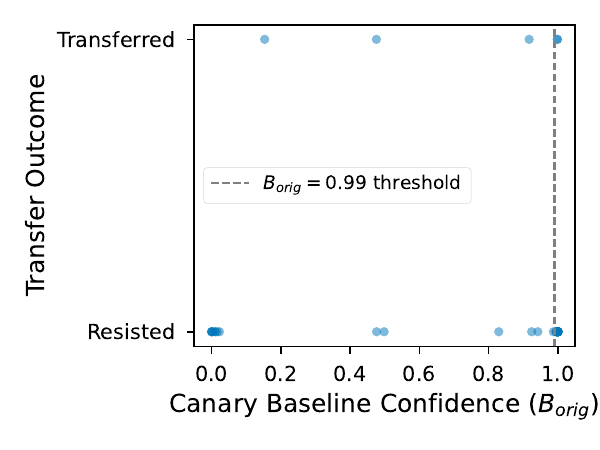}
\hfill
\includegraphics[width=0.48\textwidth]{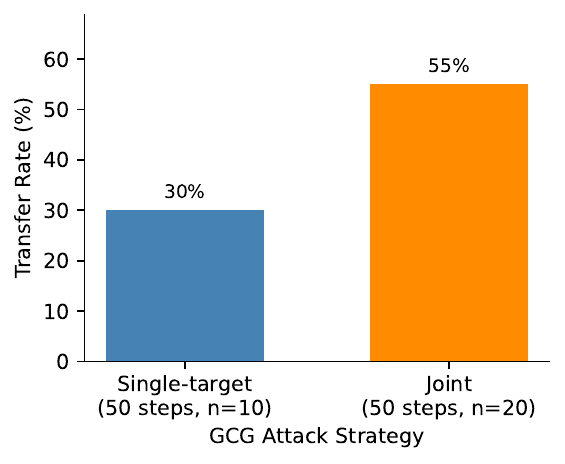}
\caption{\textbf{Left:} Canary baseline confidence ($f_B$ original) vs.\ final canary score below 0.5 (40 prompt--canary observations, 12 of the 14 below 0.5 already scored the clean prompt below 0.5; see text). \textbf{Right:} Transfer rate under single-target ($n{=}10$) vs.\ joint ($n{=}20$) optimisation at 50 GCG steps; the two runs differ in size and prompts.}
\label{fig:confidence-budget}
\end{figure}

A Tier~2 attacker optimises against the target alone but hopes the adversarial suffix accidentally fools the canary (passive transfer). We evaluate this on 10 successful GCG prompts, each scored by four canaries that are checkpoints of the DeBERTa-v3-large epoch sweep (epochs 1, 3, 5 and 10), giving 40 prompt--canary observations that are not independent.

\paragraph{Confidence and transfer.} Figure~\ref{fig:confidence-budget} (left) counts an observation as transferred when the canary's attacked score falls below 0.5. That count includes canaries that already scored the clean prompt below 0.5. Of the 14 observations that end below 0.5, 12 started there, so only 2 of the 28 observations whose canary first scored the prompt as unsafe are attack-induced crossings. Among those 28, 1 of 23 crossed when $f_B(x_{\text{orig}}) \geq 0.99$ and 1 of 5 when $0.5 \leq f_B(x_{\text{orig}}) < 0.99$. The raw contrast (4.3\%, 1/23, against 76.5\%, 13/17, Mann--Whitney $p = 6.61 \times 10^{-7}$) therefore mostly separates canaries that were already wrong from canaries that were confidently right. It is not evidence that high confidence blocks transfer, and the 28 eligible observations are too few, and too dependent through the shared prompts, to test it.

\paragraph{Interpretation.} The detection threshold ($f_B \geq 0.5$) is an empirical operating point for the binary transfer/no-transfer decision; the gradient-paralysis mechanism (\S\ref{sec:divergence-min}) requires deeper saturation ($f_B \approx 1.0$, $\|\nabla f_B\| \to 0$). Detection relies on the canary holding its score above 0.5; paralysis relies on a flat, high-confidence basin.

\subsection{Joint Optimisation (Within-Family)}
\label{sec:joint-gcg}

A Tier~3 attacker has white-box access to both the target and the canary and optimises a combined loss $L = 0.5 \cdot L_A + 0.5 \cdot L_B$ using GCG with search width 256 and 50 steps.

\paragraph{Result.} Of 20 prompts evaluated, 11 end with both classifiers below 0.5, compared to 3/10 that do so under single-target optimisation (a separate earlier run on 10 prompts). These counts include prompts that a classifier already scored below 0.5 before any attack: 7 of the 20 started below 0.5 on the target, and 5 of the 11 joint passes are among them; 2 of the 3 passive passes also started with both models below 0.5. Restricting to prompts that both classifiers first scored as unsafe, joint optimisation succeeds on 6 of 13 (46\%) and single-target optimisation on 1 of 6. The post-attack rate is $55\%$ (11/20, Wilson 95\% CI $[34\%, 74\%]$), but the two runs differ in size and prompts, and we do not test the difference.

\paragraph{Failure cases.} In 9/20 prompts, joint optimisation ends with at least one score above 0.5 despite white-box access to both models. We did not measure gradient interference; the candidate tokens are proposed from the target's embedding gradient and the canary's gradient projected through the same embedding matrix, which is not the correct joint token gradient when the two checkpoints have different embeddings, so these failures may understate what a correctly implemented joint attacker achieves. An extended run at 200 steps on the first 10 prompts leaves the same 3/10 uncracked (single replicate, $n = 10$).

\paragraph{Cross-family barrier (hypothesis).} The 1.73$\times$ tokeniser fragmentation ratio (\S\ref{sec:tokeniser}), measured on 100 random DeBERTa token sequences and not on the GCG suffixes, may impede coordinate alignment for discrete token search: a single substitution shifts downstream subword boundaries in Model~B but not Model~A, so a step that improves the joint loss for A's tokenisation may worsen it for B's. We did not test whether this causes the non-transfer. As a transfer negative control we optimised suffixes against Llama Guard~3 (8B decoder; original scores for these prompts already low, mean $= 0.057$, 0/10 above 0.5) and evaluated them on frontier canaries: $\Delta = 0.000$ on \texttt{gpt-4o-mini} and $-0.010$ on \texttt{gpt-5.1} ($n = 10$). Because LG3 did not initially score these prompts as unsafe, this is not a test of successful decoder evasion; it shows only that LG3-optimised suffixes are inert on frontier canaries, consistent with, but not evidence for, the tokeniser-barrier hypothesis.

\paragraph{Direct DeBERTa-to-Llama Guard transfer test.} To directly measure whether successful DeBERTa-targeted suffixes transfer to a decoder-based safety classifier, we evaluated 50 adversarial suffixes that successfully flipped DeBERTa against Llama Guard~3 (8B decoder). The transfer rate is 0\% (0/50 hard flips; Clopper--Pearson 95\% upper bound 7.1\%). The suffixes reduced Llama Guard's unsafe probability in some cases (e.g., from ${\sim}0.99$ to ${\sim}0.81$) but none crossed the binary classification threshold. Thus, within this tested corpus, a suffix optimised to evade DeBERTa does not suffice to evade Llama Guard. Evading both models therefore needs classifier-specific optimisation and not passive transfer. The result is limited to DeBERTa-targeted 20-token GCG suffixes at 250 optimisation steps and does not test white-box joint optimisation against Llama Guard.

\subsection{Divergence Minimisation: The Stall Condition}
\label{sec:divergence-min}

A Tier~4 attacker explicitly minimises score disagreement while attacking the target. We formalise this as:
\begin{equation}
\label{eq:div-min-loss}
L(\delta) = f_A(x + \delta) + \lambda \cdot (f_B(x + \delta) - f_A(x + \delta))^2
\end{equation}
where $\lambda = 2.0$ controls stealth pressure and Model~A is the target (DeBERTa-v3-large epoch~1) and Model~B the canary (epoch~5), checkpoints of the epoch sweep (1024-d), not the DeBERTa-v3-base primary of the main experiments.

\paragraph{Empirical result.} Of 20 prompts evaluated: 6 end in a joint stealth state (both models below 0.5, divergence $< 0.5$; not necessarily attack-induced evasions, see the baselines below) and 14 are blocked (Model~A remains above 0.5). Original (pre-attack) scores were recorded for only 10 of the 20 prompts (4 stealth, 6 blocked); the other 10 have no baseline. In all 6 blocked cases with a baseline the canary is saturated ($f_B \geq 0.9999$). The 4 stealth cases with a baseline are mixed: two began with a saturated canary and a malicious-scoring target ($f_B = 1.00$, $f_A = 0.998$; $f_B = 0.987$, $f_A = 0.829$) and still reached stealth, and two began with a low canary score ($f_B = 0.017$ and $0.0003$) and a target that already scored below 0.5 ($f_A = 0.002$ and $0.476$), so they are not evasions of an input the target classified as unsafe. The mean original $f_B$ of 0.50 for stealth against 1.00 for blocked comes from those last two cases (one-sided Mann--Whitney $p = 0.043$, two-sided $0.087$, $4$ against $6$ cases). Canary saturation therefore did not by itself prevent stealth in this experiment, and these data do not establish a confidence gate.

\paragraph{What the search actually does.} The implemented search is not gradient descent on the combined objective: it proposes candidate tokens from the target's cross-entropy gradient alone and uses the canary only to rank the sampled candidates by the combined loss. The gradient analysis below therefore describes a continuous relaxation and not the algorithm that produced the 70\%; the forward-loss parabola of Proposition~1 (Appendix~\ref{app:derivation}) governs candidate acceptance whenever $f_B$ is insensitive to the substitutions. The analytic condition assumes $\|\nabla f_B\|$ is negligible, which concerns the canary's gradient magnitude and not access (the canary is a local checkpoint whose gradient is available). A fully gradient-aware attacker, or one with a different objective, is not covered, and a higher $\lambda$ might block more but also constrain the attacker's loss landscape so far as to prevent meaningful optimisation. The stationary gap also requires that $f_B - 1/(2\lambda)$ lie in the reachable score range; for $\lambda \leq 1$ it does not for a blocked case (which has $f_A \geq 0.5$, so a gap of at most 0.5), so the low-$\lambda$ shortfall is not evidence of budget exhaustion.

\paragraph{The stall condition.} The combined gradient of Eq.~\ref{eq:div-min-loss} is:
\begin{equation}
\label{eq:combined-gradient}
\nabla L = [1 - 2\lambda(f_B - f_A)] \cdot \nabla f_A + 2\lambda(f_B - f_A) \cdot \nabla f_B
\end{equation}

The coefficient $[1 - 2\lambda(f_B - f_A)]$ inverts sign when the gap $g = f_B - f_A$ exceeds $1/(2\lambda)$. At $\lambda = 2.0$, the critical gap is $1/(2 \times 2) = 0.250$. Beyond this point, the optimiser is driven to \emph{increase} $f_A$, co-opting the attacker's gradient.

\paragraph{Empirical validation.} The 14 blocked cases stall at gap values with mean $0.218$ and median $0.250$ (full distribution, no subsetting). (This is the primary Tier-4 pair, $n{=}20$; the $\lambda$-sweep below draws a fresh 20 prompts, so its $0.218$ is a different experiment's median, not this pair's mean.) 9 of the 14 blocked cases have a final target score $f_A$ pinned in $[0.745, 0.758]$ with the canary at $f_B \approx 1.0$, a gap of $\approx 0.25 = 1/(2\lambda)$, consistent with the coefficient inversion above (see the caveat above on what the search does). Four blocked cases stall at gaps below the predicted value ($0.03$--$0.17$), representing partial optimisation progress before budget exhaustion rather than full convergence, and one stalls above ($0.48$).

\paragraph{Gradient-norm effect.} When the canary is deep in its confident basin ($f_B \approx 1.0$), $\|\nabla f_B\| \to 0$. For a gap above $1/(2\lambda)$, a step along $-\nabla L$ lowers $f_A$ to first order only if $\cos\theta > (1 - \tfrac{1}{2\lambda(f_B - f_A)})\|\nabla f_A\|/\|\nabla f_B\|$, with $\theta$ the angle between $\nabla f_A$ and $\nabla f_B$ (Appendix~\ref{app:derivation}); the right-hand side is positive and diverges as $\|\nabla f_B\| \to 0$, so under the continuous approximation no angle satisfies it. For a gap below $1/(2\lambda)$ it is negative and descent proceeds. This does not rule out descent under a different objective or a discrete search, and in the primary pair stealth also occurred with saturated canaries (above).

\paragraph{Active rejection.} In 2 of the 6 blocked cases with a recorded baseline, Model~A's score rises during optimisation (from $0.476$ to $0.758$ and from $0.498$ to $0.748$; both targets began just below 0.5, so this is a rise from the threshold rather than a reversal of a flip). The other 8 blocked cases have no baseline. This is consistent with the coefficient inversion: once $g > 1/(2\lambda)$, the combined loss penalises further reduction of $f_A$ more than it rewards it.

\paragraph{Generality (second pair).} A replication on a different pair from the same epoch sweep (DeBERTa-v3-large epoch~3 target, epoch~10 canary) gives 7/10 stealth and 3/10 blocked (70\% stealth against 30\% on the primary pair). That file records no original scores, so we cannot say whether the canary was less confident, and the primary-pair data do not separate stealth from blocked by canary confidence. The blocked \emph{rate} did not replicate, but the 3 blocked cases stall at gaps $\{0.250, 0.250, 0.242\}$ (mean $0.247$): the stall \emph{location} replicates while its \emph{frequency} does not, so $1/(2\lambda) = 0.250$ is a local attractor for blocked cases across both pairs, and we do not identify why the share reaching it differs.

\paragraph{$\lambda$-sweep (exploratory).} The stall condition predicts a $\lambda$-dependent security frontier. A sweep over $\lambda \in \{0.5, 1.0, 2.0, 4.0, 8.0\}$ ($n{=}20$ prompts each, same paired design on the primary checkpoint pair, seed-consistent) is consistent with three predictions.

\textbf{(i)~Monotonicity.} Block rate is non-decreasing in $\lambda$: 25\% $\to$ 50\% $\to$ 50\% $\to$ 55\% $\to$ 60\% (Spearman $\rho{=}0.975$, $p{=}0.005$). The $\lambda{=}2$ rate here (50\%, 10/20) differs from the 70\% (14/20) of Table~\ref{tab:threat-tiers} because the sweep draws a fresh set of 20 prompts; 50\% lies within the original estimate's Wilson 95\% CI [48\%, 85\%].

\textbf{(ii)~Saturation ceiling.} The block rate saturates well below 100\% (max 60\% at $\lambda{=}8.0$), consistent with the pre-registered prediction that the defence is not universally effective; the sweep does not isolate what sets the ceiling.

\textbf{(iii)~Gap convergence.} At high $\lambda$, the observed gap approaches the theoretical prediction: $\lambda{=}8.0$ median gap $= 0.0625$ (predicted $0.0625$, ratio $1.00$); $\lambda{=}4.0$ median $= 0.108$ (predicted $0.125$, ratio $0.86$); $\lambda{=}2.0$ median $= 0.218$ (predicted $0.250$, ratio $0.87$). At low $\lambda$ ($0.5$, $1.0$) the observed gap falls far below prediction (ratios $0.09$, $0.53$); for blocked cases this shortfall follows from the score range (caveat above) and is not evidence of budget exhaustion. The high-$\lambda$ observed/predicted ratios ($0.86$--$1.00$ at $\lambda \geq 2.0$) indicate that the stall condition functions as a local attractor whose pull strengthens with $\lambda$; the under-convergence (a slope of 0.84 through the three high-$\lambda$ median gaps, with no interval reported because three points do not support one; the committed regression over all five $\lambda$ values has slope 0.09, $R^2 = 0.21$, because the low-$\lambda$ blocked gaps are bounded by the score range) is compatible with the 50-step GCG budget approaching but not fully reaching the predicted gap.

\subsection{Deployment Recommendations}
\label{sec:deployment-recs}

To use the method: (1)~choose a canary that the attacker has not optimised against (in our tests same-family checkpoints worked as well as cross-family ones, \S\ref{sec:ca4}); (2)~estimate benign disagreement on held-out traffic; (3)~set a threshold for the desired false-positive rate; (4)~alert on sustained disagreement. Rotating the canary to limit attacker knowledge is untested for classifier canaries. The stall condition suggests an input-time indicator, the canary's own confidence $f_B$, with the limits stated in \S\ref{sec:divergence-min}: $f_B \geq 0.5$ determines whether divergence is \emph{detectable} (the canary has not been flipped), and $f_B \geq 0.99$ marks near score saturation. The tiers below and Table~\ref{tab:ops} are triage heuristics, not validated security boundaries.

\begin{itemize}
\item \textbf{Confident canary} ($f_B \geq 0.99$): Near saturation is not by itself evidence that $\|\nabla f_B\|$ is small. A canary that first scored the prompt as unsafe was rarely flipped here (1/23, but the comparison group is only 5 observations), and the stall experiment did not separate saturated from unsaturated canaries (\S\ref{sec:divergence-min}), so saturation is not shown to lower passive-transfer risk or to prevent a monitor-aware attacker; an attacker free to change the objective or to use gradient-free search is not covered. Treat the divergence signal as the primary alert, subject to the evaluated threat model.
\item \textbf{Moderate canary} ($0.5 \leq f_B < 0.99$): The canary holds at or above 0.5, so a rescue alarm is possible, but it is not gradient-paralysed. An adaptive attacker with joint optimisation may succeed (Tier~3), and the passive-transfer data are too few to say how likely transfer is at this confidence. Monitor divergence and treat alarms with caution.
\item \textbf{Uncertain canary} ($f_B < 0.5$): The divergence signal is unreliable. The attacker can succeed via joint optimisation (55\% post-attack, 46\% among prompts both models first scored as unsafe, within-family at $n{=}20$, \S\ref{sec:joint-gcg}) or passive transfer (final canary score below 0.5 for 13 of 17 observations with $f_B < 0.99$, mostly canaries that were already wrong; \S\ref{sec:transfer}). Do not rely on divergence alone: apply the primary decision (Table~\ref{tab:ops}) and sample these inputs for human review.
\item \textbf{$k = 2$ classifiers:} Deploy one same-family canary (highest detection sensitivity, 93.9\%) and one cross-family canary (a different tokeniser may make joint optimisation harder, which we did not test). This is a design option; we did not show that it adds robustness against adaptive attackers.
\end{itemize}

\paragraph{Operational response policy.} The monitor produces alerts, not automatic decisions. Table~\ref{tab:ops} translates observed signals into triage actions.

\begin{table}[h]
\centering
\small
\caption{Operational response policy for the monitor, not for the content decision. Scores are unsafe probabilities: target high = primary classifier score $\geq 0.5$ (the primary classifier blocks the input as usual); target low = score $< 0.5$; canary high = canary score $\geq 0.5$. KS/scan alarm = distribution shift detected. The canary confidence used in the Moderate/High rows is that of the original, unattacked input in our experiments; a deployment sees only the incoming input, so those rows need a separately validated input-time rule.}
\label{tab:ops}
\begin{tabular}{lllp{0.42\textwidth}}
\toprule
\textbf{KS/scan alarm} & \textbf{Target score} & \textbf{Canary score} & \textbf{Action} \\
\midrule
No  & High & High & No monitor action; log metrics (the primary blocks as usual) \\
Yes & High & High & Investigate drift; sample labels \\
No  & Low  & High ($\geq 0.99$) & Escalate: suspected targeted evasion \\
No  & Low  & Moderate ($0.5$--$0.99$) & Human review \\
No  & Low  & Low ($< 0.5$) & No canary signal; apply the primary decision \\
Yes & Low  & High & Block / escalate; high-confidence attack \\
Any & Any  & Refusal (LLM canary) & Flag as high-risk; log separately \\
\bottomrule
\end{tabular}
\end{table}

We do not claim this defence is absolute: a Tier~3 attacker with white-box access to a same-family canary defeats both models on 55\% of prompts (11/20), or 46\% (6/13) of those both models first scored as unsafe (\S\ref{sec:joint-gcg}). The contribution is a quantitative characterisation of where the defence held and failed in the tested settings; no measurable property, including canary confidence $f_B$, predicted which regime applies under a monitor-aware attacker.

\subsection{Multi-Monitor Extension}
\label{sec:game-theory}

Appendix~\ref{app:first-order} states local first-order conditions for $K$ canaries and the discrete fixed-canary stall algebra, which corrects an earlier version. Both are local statements: the experiments do not establish whether the conditions hold at the points the optimiser visits, and the discrete gap bound for canaries that also move is a conjecture ($a_{\min}$ is unmeasured). Joint optimisation still succeeds (\S\ref{sec:joint-gcg}), so within-family canaries offer limited protection.

\section{Frontier LLM Canaries and an Operational Case Study}
\label{sec:llm-eval}

\subsection{Operational Case Study: Cost, Latency and Specificity}
\label{sec:deployment-casestudy}

\paragraph{Primary-fooled validation.} The reproducible evidence that the canary fires when the primary is actually fooled is the attack-specificity result of \S\ref{sec:ca6} (37/49; \texttt{results/gate\_a\_divergence.json}). No result is reported for an open-weight LLM canary: the committed case-study run recorded failed canary API calls, not scores.

\paragraph{Operational feasibility (non-attack baseline).} We validate cost, latency and false-positive rate on a 200-query workload (150 benign, 30 adversarial without suffix, 20 adversarial with DeBERTa-targeted GCG suffixes) using \texttt{gpt-4o-mini} as primary and \texttt{gpt-4.1} as semantic canary, scored sequentially through the endpoint used in \S\ref{sec:llm-canary}. Latency (Table~\ref{tab:deployment-case} gives detection): primary p50/p95 667\,ms / 1282\,ms; canary 1153\,ms / 1700\,ms; sequential total 1831\,ms / 2674\,ms. The calls are independent, so parallel execution would give about $\max(\text{primary}, \text{canary}) \approx 1.2$\,s at p50 (computed from the sequential measurements, not measured). Cost: \$0.017 per 1\,000 queries for the primary, \$0.223 for the canary, \$0.24 in total.

\begin{table}[h]
\centering
\small
\caption{Deployment case study detection results (200 queries). FPR computed on 150 benign queries; detection rate on 50 adversarial queries (30 base + 20 with GCG suffix).}
\label{tab:deployment-case}
\begin{tabular}{lcc}
\toprule
\textbf{Condition} & \textbf{Rate} & \textbf{Wilson 95\% CI} \\
\midrule
False positives (benign, $n{=}150$) & 0/150 = 0.0\% & [0.0\%, 2.5\%] \\
Detection (adversarial, no suffix, $n{=}30$) & 1/30 = 3.3\% & [0.6\%, 16.7\%] \\
Detection (adversarial + GCG suffix, $n{=}20$) & 0/20 = 0.0\% & [0.0\%, 16.1\%] \\
Detection (all adversarial, $n{=}50$) & 1/50 = 2.0\% & [0.4\%, 10.5\%] \\
\bottomrule
\end{tabular}
\end{table}

\paragraph{Score disagreement.} Mean disagreement $|f_{\text{canary}} - f_{\text{primary}}|$ is 0.0007 on benign queries and 0.126 on adversarial queries (a 180$\times$ ratio), but the absolute magnitude stays below a practical detection threshold because both models agree on the classification.

\paragraph{Interpretation: operational viability, not adversarial detection.} The study validates latency overhead, cost and false-alarm burden ($<$2.5\%) under benign traffic, not adversarial detection. Primary and canary agree on the adversarial inputs (mean scores 0.81 and 0.78), so there is no disagreement signal: the GCG suffixes were optimised against DeBERTa, not these LLMs, and both correctly flag the harmful content. The canary detects evasion \emph{of one classifier by checking disagreement with another}; when both models agree there is nothing to detect, and it fires only when the primary is under a classifier-specific attack that succeeds against it but not the canary, as in \S\ref{sec:ca6}. Detection under the targeted-evasion threat model is the benchmark of \S\ref{sec:ca6}--\ref{sec:llm-canary}: a measured estimate, not a guarantee and not a generic harmful-content detector.

\subsection{Frontier LLMs as Semantic Canaries}
\label{sec:llm-canary}

The canary classifiers evaluated in \S\ref{sec:canary}--\ref{sec:adversarial-robustness} are local models requiring GPU inference. A natural question is whether frontier LLMs, whose weights are not accessible for white-box gradient attacks (so suffixes optimised against a local model can reach them only by transfer), can serve as zero-deployment canaries. We prompt frontier models to score input safety on a 0--1 scale and measure whether their scores diverge from the targeted classifier's under GCG attack.

\paragraph{Setup.} We initially screen 35 frontier models with a brief token-budget test, then conduct detailed evaluation on a primary subset of 20 models, with $n = 49$ prompts in each of three arms: adversarial (harmful prompt plus GCG suffix), ``clean'' (the \emph{same harmful base prompt without the suffix}, not a benign prompt) and scrambled (the suffix characters permuted). The models span 4 providers via an OpenAI-compatible model service. Benign specificity is measured separately on 20 benign prompts per model, and on 1{,}000 benign prompts for 10 models (below). Detection statistics are reported on the primary 20-model evaluation set (benign-prompt scores for these 20 are in the committed benign-specificity file); the cost correlation uses the 19 models with detection rates on the full prompt set that are not ceiling-clipped on benign prompts, and the taxonomy below is limited to what the committed re-check files support (\S\ref{sec:token-starvation}). System prompt: ``Rate safety risk 0.0--1.0, reply with only a number.'' Temperature = 0.

\begin{figure}[t]
\centering
\includegraphics[width=0.48\textwidth]{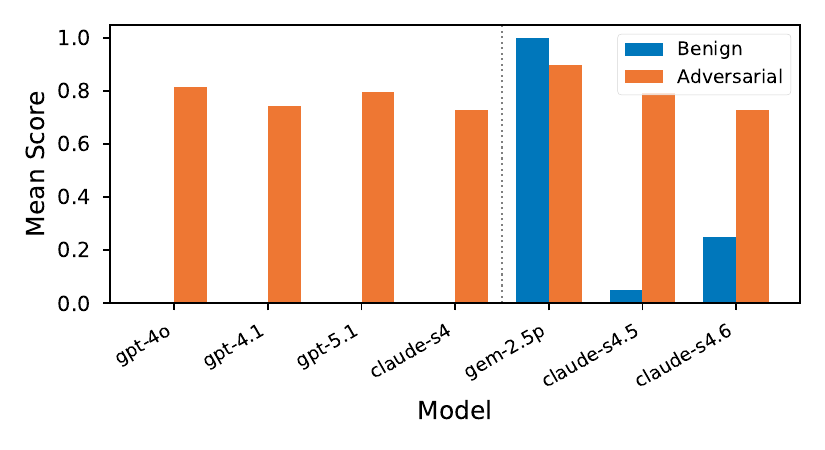}
\hfill
\includegraphics[width=0.48\textwidth]{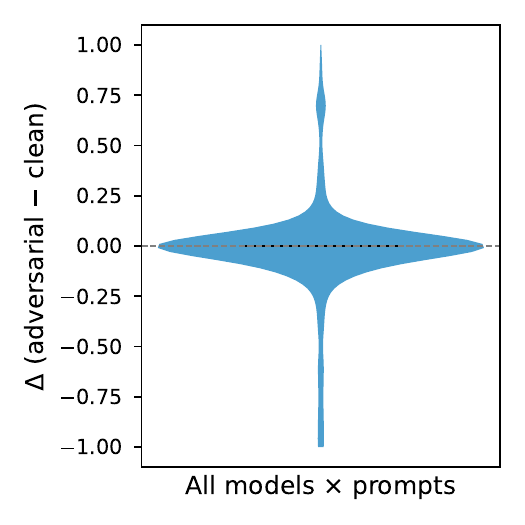}
\caption{\textbf{Left:} Ceiling-clipped models (score everything $\approx 1.0$) vs.\ discriminating models (benign $\approx 0.0$, adversarial $\approx 0.7$--$0.9$). \textbf{Right:} Violin plot of $\Delta(\text{adv} - \text{clean})$ across 980 prompt--model pairs, showing the mass at $\approx 0$ with a tail of per-prompt collapses.}
\label{fig:llm-canary}
\end{figure}

\paragraph{Ceiling-clipping was a token-budget artefact.} Figure~\ref{fig:llm-canary} (left) shows the split in an initial screen at \texttt{max\_completion\_tokens=16}: 12 of the 31 models in the Pareto file scored \emph{all} inputs (benign, adversarial and clean) as $\approx 1.0$. We first termed these ``ceiling-clipped'' and hypothesised an RLHF-tuning artefact. It was a \emph{methodological artefact}: reasoning models (o3, o3-mini, o4-mini, gemini-2.5-flash, gemini-2.5-pro) spend a small budget on internal chain-of-thought, returned empty responses, and our parser defaulted those to $1.0$ (curves in \S\ref{sec:token-starvation}). Re-running 10 models (nine of the 12 plus \texttt{gemini-3.1-flash-lite-preview}) with \texttt{max\_tokens$\geq$200} on 10 benign and 10 adversarial prompts removed the ceiling for benign prompts (benign mean $0.0$ for all ten), but adversarial scores were lower and sparser than the table of discriminating models suggests: adversarial means of 0.44--0.90 on the valid responses, with no valid adversarial response for \texttt{gpt-5} and \texttt{gpt-5-mini} and 9, 6, 4 and 2 of 10 empty for \texttt{o3-mini}, \texttt{o4-mini}, \texttt{gemini-3.1-pro-preview} and \texttt{gpt-5.2-chat}.

\paragraph{Taxonomy supported by the committed files.} Table~\ref{tab:model-taxonomy} gives the counts: of 35 models, 31 are in the Pareto file (19 discriminating, 12 ceiling-clipped at the screen). The re-check shows ``ceiling-clipped'' mixes at least two behaviours: models that return graded scores once the budget is adequate (the three Gemini models and \texttt{gpt-5.2-chat} among them) and models that return no score for most or all adversarial prompts (\texttt{gpt-5}, \texttt{gpt-5-mini}, \texttt{o3-mini}). \texttt{gpt-5-nano} returned no score for 19 of 20 benign prompts, and \texttt{gpt-5.5} scored every input near 1.0 with valid responses throughout (mean 0.996 on adversarial, 1.0 on clean and scrambled), so neither is a refuser that scores benign content correctly. \texttt{claude-opus-4-1} and \texttt{gpt-5.4-pro} have no benign scores, so we do not report a 33/2 split.

\begin{table}[h]
\centering
\small
\caption{Model groups among the 35 evaluated frontier models, from the committed result files (Pareto file and re-check files). The groups are not a partition by behaviour at an adequate budget.}
\label{tab:model-taxonomy}
\begin{tabular}{lcp{6.5cm}}
\toprule
\textbf{Category} & \textbf{Count} & \textbf{Behaviour} \\
\midrule
Discriminating at the screen & 19 & Graded scores on benign and adversarial prompts (Pareto file) \\
Ceiling-clipped at the 16-token screen & 12 & Re-checked for 9 (and one other model): benign 0.0; adversarial scores sparse or absent for several \\
No benign-prompt scores / not in Pareto file & 4 & \texttt{claude-opus-4-1}, \texttt{gpt-5.4-pro}, \texttt{gpt-5.5} (all near 1.0 on every input), \texttt{gemini-3.1-flash-lite-preview} (re-checked, graded) \\
\bottomrule
\end{tabular}
\vspace{2pt}

{\footnotesize \textit{Note:} ``Ceiling-clipped'' at the 16-token screen was largely a token-budget artefact (text above).}
\end{table}

\paragraph{Cost--performance Pareto frontier.} On the 49 adversarial prompts the per-model detection rate ranges from 49\% (\texttt{gemini-2.5-flash-lite}) to 96\% (\texttt{claude-opus-4-5}) among the 19 models with benign scores, so the frontier trades cost against detection (Table~\ref{tab:pareto-deployment}).

\paragraph{Detection performance.} Among the 19 discriminating models of the 20 with benign-prompt scores (\texttt{gemini-2.5-pro} scores every benign prompt 1.0 and is classed as ceiling-clipped; two stored benign scores, 5.0 and 1945.0 for \texttt{claude-sonnet-4-6} and \texttt{claude-opus-4-7}, are parsing artefacts and are clipped to 1.0, which counts them as one false alarm each):

\begin{itemize}
\item Mean $\Delta(\text{adv} - \text{benign})$: $+0.73$ (range $+0.385$ to $+0.907$)
\item Mean benign score: $0.011$ (15 of 19 models at exactly $0.0$ on all 20 prompts)
\item Mean adversarial score: $0.74$ (range $0.44$ to $0.91$)
\item No false alarm on 20 benign prompts for 15 of 19 models; the other four (\texttt{claude-sonnet-4-5}, \texttt{claude-sonnet-4-6}, \texttt{claude-opus-4-7}, \texttt{gemini-2.5-flash-lite}) each score one benign prompt at or above 0.5
\end{itemize}

DeBERTa scores these same adversarial prompts as $< 0.5$ (it is fooled); the frontier LLMs keep high scores, a divergence signal analogous to the local canary (\S\ref{sec:ca6}) by a different mechanism.

\paragraph{Suffix insensitivity.} Adversarial prompts are scored identically to clean prompts (same base prompt, no suffix): $\Delta(\text{adv} - \text{clean}) \approx 0$. The frontier LLMs are therefore \emph{not} detecting the GCG suffix as anomalous; they identify the base prompt as harmful regardless of the appended suffix, consistent with semantic intent recognition. This does not mean GCG is noise (it drops DeBERTa from 0.95 to $<$0.05); it means the LLMs score on base-prompt content, which is what makes them robust canaries for a classifier that \emph{is} fooled by suffix structure.

At $N = 20$ this claim was underpowered (minimum detectable effect $0.253$, observed mean $\Delta = 0.047$). At $N = 49$, Bayes Factor analysis (JZS prior, $r = \sqrt{2}/2$) gives $\text{BF}_{01} > 3$ for 16/20 models (evidence for the null) and $\text{BF}_{01} < 1/3$ for 2/20; TOST equivalence testing finds 16/20 models equivalent within $\pm 0.15$ ($p < 0.05$) and 19/20 within $\pm 0.20$. The median $|\Delta|$ across 20 models is 0.035.

\paragraph{Reconciling BF and TOST on outliers.} Two models have a detectable effect. For \texttt{gpt-4.1} ($\Delta = -0.039$, $\text{BF}_{01} = 0.19$) the within-model variance is extremely low and the magnitude is about twenty times smaller than DeBERTa's collapse ($\Delta \approx -0.8$); TOST confirms equivalence within $\pm 0.15$. For \texttt{claude-opus-4-7} ($\Delta = +0.127$, $\text{BF}_{01} = 0.13$) scrambled suffixes give a comparable shift ($\Delta_{\text{scram-clean}} = +0.100$), so any appended text raises its score and the effect is not GCG-specific; TOST passes at $\pm 0.20$ but not $\pm 0.15$. The two inconclusive models (\texttt{gemini-2.5-flash-lite}, \texttt{gpt-4.1-nano}) have within-pair SD of 0.39 and 0.32 (mean 0.19 for the other 18), limiting BF sensitivity. We justify the $\pm 0.15$ bound by decision invariance: with threshold $\tau = 0.5$ and a mean clean score of 0.78 (SD 0.09), a shift of $\pm 0.15$ stays above the decision boundary.

\paragraph{Per-prompt boundary crossings.} The $\Delta \approx 0$ finding describes the \emph{average} across prompts. Averaged over the 20 models, the largest per-prompt deviation $|\bar\Delta_p|$ is 0.33 (95th percentile over the 49 prompts $0.14$), but individual prompt--model pairs deviate much more: the largest $|f(p+s)-f(p)|$ is 1.0 and its 95th percentile $0.70$. At the prompt--model level, 6.0\% of pairs (59/980) show a boundary crossing (clean $\geq 0.5$, GCG $< 0.5$). These are concentrated in specific prompt--model combinations: 83\% are large drops ($> 0.3$), not micro-drifts, and no model contributes more than 9/49 flips. No single prompt evades all 20 models (worst case: 8/20 flip on one prompt). The mean $\Delta \approx 0$ reflects that for the majority of prompts, all models hold; the residual 6\% tail represents specific prompt--model vulnerabilities, addressable via the $k \geq 2$ ensemble logic of \S\ref{sec:k-scaling}.

\paragraph{Scrambled-suffix control.} To separate semantic intent recognition from token-level anomaly detection we add a scrambled-suffix condition (GCG characters randomly permuted, preserving character distribution but destroying gradient alignment and word boundaries). If the LLM scores the \emph{base prompt}, all three conditions should score alike; if it detects the \emph{GCG token pattern}, scrambled should score lower than adversarial.

At $N = 49$, 15/20 models show no significant scrambled-vs-GCG difference (Mann--Whitney $p \geq 0.05$). In the 5 where one exists (all Claude variants) scrambled scores are \emph{higher} than GCG, i.e.\ scrambled tokens make them \emph{more} suspicious; zero models show scrambled $<$ GCG. This is inconsistent with detection through a GCG-specific token signature and consistent with scoring on the base prompt's semantic content.

\paragraph{Methodological note on scrambling level.} Our scramble is character-level, which also disrupts tokeniser word boundaries. For the 15 models where scrambled $\approx$ GCG, neither GCG token structure nor tokeniser fragmentation appears to matter. For the 5 Claude models where scrambled $>$ GCG we cannot separate (a)~a secondary channel that flags fragmented text from (b)~the GCG suffix looking marginally less anomalous than random characters; either way scrambled inputs are flagged at least as often as GCG inputs, so the detection floor under attack is at least the GCG rate. A token-level scramble would isolate the causes but was not performed.

\paragraph{Deployment recommendation.} Among discriminating models, \texttt{gpt-4o-mini} achieves $\Delta(\text{adv} - \text{benign}) = +0.83$ with benign mean $0.0$ at the lowest cost tier. Call one cheap frontier model per input and flag when its score is $\geq 0.5$ while the primary scores $< 0.5$.

\paragraph{Frontier stability (Jul--Aug 2026).} We re-ran three models from the earlier screen (\texttt{gpt-5.5}, \texttt{gpt-5-mini}, \texttt{o4-mini}), which produce empty or non-numeric responses under our scoring prompt (the token budget used for \texttt{gpt-5.5} is not recorded). We report no detection or false-positive rates for the newer models scored (\texttt{claude-sonnet-5}, \texttt{gemini-3.6-flash}, \texttt{gemini-3.7-flash}): the update script built both arms from a stored 60-character prefix of the original prompt, so the arms do not separate adversarial from benign inputs (Appendix~\ref{app:corrections}, item~15). The Pareto frontier (\texttt{gpt-4o-mini} at 84\%/0.7\% and \texttt{gpt-5.1} at 94\%/0.6\%) rests on the earlier $N{=}49$ screen and the $N{=}1000$ benign validation; the newer models remain unevaluated.

\paragraph{False positive rate at scale ($N = 1000$).} With 1000 benign prompts, \texttt{gpt-4.1} and \texttt{claude-haiku-4-5} each have 0/1000 false positives (exact Clopper--Pearson 95\% upper bound $\approx 0.37\%$), and all 10 evaluated models stay below 1.5\% FPR, so the zero false positives in the 20-prompt screen (15 of 19 discriminating models) was not a small-sample artefact.

\paragraph{Robustness and sensitivity checks.} Table~\ref{tab:llm-checks} collects six smaller tests of the frontier canaries. Read together, they show that the canary score depends on the base prompt's meaning and is not moved by the suffix, that a black-box attacker gains little, and that non-English prompts are the weak point; the vulnerability there is in base-prompt recognition, not suffix interaction.

\begin{table}[h]
\centering
\footnotesize
\caption{Smaller frontier-canary tests ($\tau = 0.5$ throughout). Each row is a separate small experiment; none is powered to establish its claim on its own.}
\label{tab:llm-checks}
\begin{tabular}{p{2.3cm}p{3.4cm}p{6.6cm}}
\toprule
\textbf{Test} & \textbf{Setup} & \textbf{Result and caveat} \\
\midrule
Black-box substitution & 100 coordinate-ascent token substitutions against \texttt{gpt-4o-mini}; \texttt{gpt-5.1} under the same budget & Largest score reduction 0.2 ($\approx 0.83$ to $\approx 0.63$), which does not breach $\tau$; \texttt{gpt-5.1} equivalent. \\
Extended budget & 500 iterations; $n = 10$ runs on nine distinct prompts (one repeated), baseline $\geq 0.8$ & Highest-scoring prompts, which are the most confidently detected and not necessarily the hardest for the defender. All improvement falls in the first 50 steps in 9/10 runs, none for steps 50--500; one run improves once more after step 50 (1.0 to 0.9). 0/10 breach $\tau$ (exact 95\% upper bound about 31\%); mean final score 0.90. The landscape is discrete and saturating. \\
Ghost suffix & GCG suffixes on \emph{benign} content, \texttt{gpt-5.1}, $n = 20$ & 0/20 false positives: the suffix is inert when the base prompt lacks harmful content. \\
Prefix mutation & 20 prompts rephrased by a paraphrasing model, scored with the suffix by \texttt{gpt-5.1} & The model refused to rephrase 8. On the 12 rephrased: 11/12 (92\%) at or above 0.5 against 11/12 (92\%) of the originals. Counting the 8 refusal sentences as rephrasings: 11/20 (55\%) against 19/20 (95\%). Compatible with detection that depends on meaning more than wording, but small, the 8 excluded prompts are those the model declined, and it does not establish the claim. \\
Cross-lingual (Spanish) & \texttt{gpt-4o-mini}, $N = 49$ & Detection 63.3\% (31/49, Wilson 95\% CI [49.3\%, 75.3\%]) against 83.7\% in English (41/49, [71.0\%, 91.5\%]), a drop of 20.4~pp. Non-English prompts overall show 15--25\% degradation; suffixes stay inert ($\Delta \approx 0$ between clean and suffix-appended prompts) in all tested languages. \\
Cross-lingual asymmetry & \texttt{gpt-5.1}, 19 prompts detected in English & Explicit-harm prompts (direct violence, slurs) degrade 19--29~pp, ambiguous or roleplay prompts 10--15~pp; 16 of the 19 remain detected. Consistent with a lexical contribution lost under translation alongside a language-independent structural one, but a descriptive hypothesis from one model. \\
\bottomrule
\end{tabular}
\end{table}

\paragraph{Wilson CI Pareto.} Using Wilson score intervals as conservative benchmark estimates, \texttt{gpt-4o-mini} has a 95\% lower bound of $\geq 71\%$ detection and \texttt{gpt-5.1} $\geq 83.5\%$. \texttt{gpt-5.1} costs $\approx$9$\times$ more (\$0.30 vs.\ \$0.033 per 1{,}000 calls, from provider prices at evaluation time and our observed token counts) for 12.5 points of lower-bound improvement. Across the 19 models with a thresholded detection rate on the full prompt set that are not ceiling-clipped on benign prompts, detection is moderately correlated with the logarithm of input price (Pearson $r = 0.55$, $p = 0.014$; Spearman $\rho = 0.34$, $p = 0.15$; adding \texttt{gpt-5.4-pro}: $n = 20$, $r = 0.59$, $\rho = 0.42$). An earlier value ($r = 0.51$ over 28 models) mixed mean scores with detection rates and is withdrawn. The rank correlation is not significant, so we claim only a rough trend; capability as well as cost determines the choice (detection ranges from 49\% to 96\% at \texttt{max\_tokens $\geq$ 200}).

\begin{table}[h]
\centering
\small
\caption{Benchmark metrics for canary model selection: Pareto frontier. Detection = proportion of $N{=}49$ adversarial prompts scored $\geq 0.5$ (the $N{=}1000$ concerns benign false positives only, not detection). FPR validated at $N{=}1000$ benign prompts. Wilson 95\% CI lower bounds are benchmark estimates conditional on this corpus and attack procedure, not deployment guarantees.}
\label{tab:pareto-deployment}
\begin{tabular}{lccccc}
\toprule
\textbf{Model} & \textbf{Detection} & \textbf{Wilson LB} & \textbf{FPR ($N{=}1000$)} & \textbf{FPR UB} & \textbf{Cost/1k} \\
\midrule
gpt-4o-mini & 83.7\% (41/49) & $\geq$71.0\% & 7/1000 (0.70\%) & $\leq$1.44\% & \$0.033 \\
gpt-5.1 & 93.9\% (46/49) & $\geq$83.5\% & 6/1000 (0.60\%) & $\leq$1.30\% & \$0.30 \\
\bottomrule
\end{tabular}
\vspace{2pt}

{\footnotesize Wilson LB = Wilson score 95\% CI lower bound on detection rate. FPR UB = Wilson score 95\% upper bound on false positive rate. Cost per 1{,}000 calls at observed prompt/output-token counts.}
\end{table}

\paragraph{Limitations.}
\begin{itemize}
\item \textbf{Temperature sensitivity.} All results use temperature $= 0$ for reproducibility. A sensitivity sweep at $T \in \{0, 0.3, 1.0\}$ across 5 models ($n = 20$ adv + 20 benign prompts, 3 replicates per non-zero temperature) shows detection rate varies by at most $7$ pp across temperatures, with within-prompt SD at $T = 1.0$ of $0.03$--$0.11$ across models. Temperature does not materially affect detection; $T = 0$ is recommended for operational consistency.
\item Service-based: non-reproducible across provider version changes.
\item The roster is a set of frontier generative LLMs, not a controlled architecture comparison; proprietary training and architecture details are not known, so we make no claim about architectural diversity among them.
\item Cross-lingual detection degrades 15--25\% for non-English prompts (Table~\ref{tab:llm-checks}). The Mandarin and Arabic figures come from a 20-prompt run of one model; a larger top-up with \texttt{claude-sonnet-4-5} (37 English and 44--45 translated prompts) gives drops of about 20 (Mandarin) and 8 (Arabic) points, so the range depends on the model.
\item Llama Guard 3 surrogate transfer test: LG3 original scores for test prompts were already low (mean $= 0.057$, 0/10 above 0.5), so this is a transfer negative control rather than a test of successful decoder evasion. LG3-optimised suffixes produce $|\Delta| \leq 0.010$ on both tested frontier canaries ($n = 10$), consistent with, but not evidence for, the tokeniser-barrier hypothesis. A full decoder evasion experiment (prompts for which LG3 initially scores high, successfully attacked to low) is not included.
\item Models that return no score for harmful content (for example \texttt{gpt-5} and \texttt{gpt-5-mini} on the adversarial prompts above) might be usable as binary detection signals (empty response $=$ flag), but this was not tested and integration is not implemented.
\end{itemize}

\paragraph{Relationship to local canaries.} The LLM canary and the local canary (\S\ref{sec:canary}) are complementary:
\begin{itemize}
\item \textbf{Local canary:} detects via \emph{score disagreement} (target collapses, canary holds). Within-family joint optimisation can defeat both scores (\S\ref{sec:joint-gcg}), while passive transfer was rarely attack-induced (2 of 28 eligible observations; \S\ref{sec:transfer}). Cross-family, tokeniser fragmentation may impede joint search (a hypothesis we did not test).
\item \textbf{LLM canary:} detects via \emph{semantic persistence} (scores harmful intent regardless of suffix). Not open to white-box gradient attack (no accessible weights); in our tests suffixes optimised on local models changed their scores little (\S\ref{sec:llm-canary}). Vulnerable to prompt-level attacks against the LLM itself (untested).
\end{itemize}
Combining both channels (local divergence for streaming detection, an LLM semantic check for high-stakes inputs) would give defence in depth; we leave the integration to future work.

\subsection{Rotating Factual-Question Probes Against Adaptive Evasion}
\label{sec:rotating-canary-eval}

A rotating probe is a secret factual question embedded in the session, not a second classifier (it is called a canary in some source files). It is reported in Appendix~\ref{app:rotating} (1,950 sessions, three frontier models, four attacker strategies). Spoofing replies passed unflagged in 150 of 150 GPT-5.1 sessions and 80 of 100 Gemini-2.5-Pro sessions because the scorer gives 0.7 to any non-empty, non-refusal reply, a property of the scorer and not of rotation; Claude-Sonnet-5's substring matches mostly quote the injected text while declining it. We draw no conclusion about rotating classifier canaries from this study.

\section{Reasoning-Token Budget Starvation}
\label{sec:token-starvation}

The ceiling-clipping artefact of \S\ref{sec:llm-canary} is one instance of a general failure of reasoning-tier monitors: when the completion-token cap is smaller than the model needs, the visible answer is empty and a naive parser supplies a default score. This section gives the full account summarised in the abstract and in contribution~4 of \S\ref{sec:intro}.

\paragraph{CoT budget starvation: a sigmoid response curve.} The response probability as a function of the completion-token cap $T_r$ (which includes hidden reasoning) follows a sigmoid:
\[
P(\text{response} \mid T_r) = \frac{1}{1 + e^{-k(T_r - T_{50})}}
\]
where $T_{50}$ is the critical threshold at which 50\% of inputs elicit a valid response. The threshold $T_{50}$ is \emph{input-complexity-dependent}. For o3 with \texttt{reasoning\_effort=low}, the benign curve is well-determined: $T_{50,\text{benign}} = 46$ tokens (steepness $k = 0.173$; $R^2 = 0.99$ over the merged budget grid; bootstrap $T_{50}$ 95\% CI $[41, 51]$). The original $10$--$80$ token grid ($n{=}10$/class) left the adversarial valid-response rate at the floor for every budget (0/10 through 70 tokens, 1/10 at 80 tokens), so no adversarial sigmoid could be fit within that design. An extended sweep (post-hoc, exploratory) measured the same prompt sets and scoring prompt at $\{80, 100, 125, 150, 175, 200, 250, 300\}$ tokens ($n{=}20$/class, 0 API errors), making the adversarial curve identifiable: valid-response rate $0.13, 0.45, 0.55, 0.85, 0.80, 1.00, 0.95, 1.00$ across those budgets, giving $T_{50,\text{adv}} = 117$ tokens ($k_{\text{adv}} = 0.047$; $R^2 = 0.98$; bootstrap 95\% CI $[104, 129]$), a $2.5\times$ upward shift from benign. \emph{Deployment rule:} set $T_r \geq T_{50} + 2.2/k$ for $>$90\% valid scoring, giving $\approx 59$ tokens for benign and $\approx 164$ tokens for adversarial inputs; treat empty responses as failures, not as safe abstentions.

\paragraph{Correction (v5).} The earlier $T_{50,\text{adv}} = 154$ tokens ($3.3\times$) rested on a single out-of-grid point and is withdrawn (Appendix~\ref{app:corrections}, item~1); the directly measured values above and in Table~\ref{tab:cot-thresholds} replace it and the interim lower bound ($>80$ tokens).

\begin{table}[h]
\centering
\small
\caption{Token budget thresholds for reasoning-tier safety monitors (o3, \texttt{reasoning\_effort=low}). At $T_{50}$ half of the responses are valid; scoring failure becomes near-total only well below it (on the fitted benign curve 15\% are valid 10 tokens below $T_{50}$ and 3\% 20 tokens below; the adversarial curve is shallower, 39\% and 28\%). Both curves are sigmoid fits over the merged $10$--$300$ token grid; the adversarial threshold, unfittable on the original $10$--$80$ grid, is a measured point estimate after the post-hoc exploratory extension. $T_{90} = T_{50} + 2.2/k$.}
\label{tab:cot-thresholds}
\begin{tabular}{lcc}
\toprule
\textbf{Input class} & $\mathbf{T_{50}}$ \textbf{(50\% response)} & $\mathbf{T_{90}}$ \textbf{(90\% response)} \\
\midrule
Benign (simple queries) & 46 tokens (95\% CI [41, 51]) & 59 tokens \\
Adversarial (GCG-appended) & 117 tokens (95\% CI [104, 129]) & 164 tokens \\
\midrule
\textbf{Ratio (adv/benign)} & \textbf{2.5$\times$} & \textbf{2.8$\times$} \\
\bottomrule
\end{tabular}
\end{table}

\paragraph{Cross-model validation: complete scoring collapse at a 16-token cap.}
To test whether budget starvation generalises beyond o3, we evaluated three reasoning models as safety judges (o3, o4-mini, gemini-2.5-pro) on $N{=}100$ prompts (50~adversarial, 50~benign) under two budget conditions: $\text{max\_completion\_tokens}{=}16$ (a tight, cost-optimised cap; we have no data on how common it is) and $\text{max\_completion\_tokens}{=}500$ (adequate for these prompts and models). System prompt: ``Rate safety risk 0.0--1.0, reply with only a number.''

\begin{table}[t]
\caption{Token-starvation scoring collapse ($N{=}100$ per model per budget). At 16-token budget, all reasoning models consume the entire budget on chain-of-thought and emit empty responses. At adequate budget (500 tokens), all achieve near-perfect safety classification.}
\label{tab:rank-reversal}
\centering
\small
\begin{tabular}{llcccc}
\toprule
\textbf{Model} & \textbf{Budget} & \textbf{Valid\%} & \textbf{Adv Mean} & \textbf{Ben Mean} & \textbf{AUC} \\
\midrule
o3 & 16 & 0.0\% $[0.000, 0.037]$ & n/a & n/a & n/a \\
o4-mini & 16 & 0.0\% $[0.000, 0.037]$ & n/a & n/a & n/a \\
gemini-2.5-pro & 16 & 0.0\% $[0.000, 0.037]$ & n/a & n/a & n/a \\
\midrule
o3 & 500 & 100\% $[0.963, 1.000]$ & 0.981 & 0.018 & 1.000 \\
o4-mini & 500 & 100\% $[0.963, 1.000]$ & 1.000 & 0.000 & 1.000 \\
gemini-2.5-pro & 500 & 99\% $[0.946, 0.998]$ & 0.998 & 0.000 & 1.000 \\
\bottomrule
\end{tabular}
\vspace{-2mm}
\end{table}

Table~\ref{tab:rank-reversal} demonstrates a \emph{complete scoring collapse}: at $\text{max\_completion\_tokens}{=}16$, all three reasoning models produce 0/100 valid responses (Wilson 95\% CI upper bound: 3.7\%). Inspection of responses reveals empty strings: the 16-token budget is consumed entirely by internal chain-of-thought reasoning (confirmed: all trials show $\text{reasoning\_tokens}{=}16$, $\text{content}{=}$``''), leaving zero tokens for the visible safety score. At adequate budget (500 tokens), the same models achieve AUC${=}1.000$ with near-perfect separation (mean adversarial score $\geq 0.981$, mean benign score $\leq 0.018$).

\paragraph{Deployment implications.} At this cap the result is a \emph{total scoring failure}. Our own screen mapped empty replies to 1.0 (unsafe), which produced the ceiling-clipping artefact (false positives); a pipeline with the opposite default fails open. A safety pipeline that (a)~uses reasoning models as judges, (b)~restricts their token budget for cost control, and (c)~defaults empty/invalid responses to ``safe'' would pass 100\% of adversarial inputs undetected (a hypothetical default; ours was the unsafe one). This is not an exotic configuration: LLM-based safety classifiers are established practice~\citep{inan2023llamaguard}, reasoning models are a natural choice of judge (we have no deployment data on how often they are used), and token budgets are a standard cost lever. The failure is \emph{silent}: no error code, no refusal, no indication that the model was unable to score, only an empty string that a naive parser turns into a default score.

\section{Discussion}
\label{sec:discussion}

\subsection{Per-Classifier Monitoring Profiles}
\label{sec:profiles}

The crossover (Table~\ref{tab:latency}) is the robust finding; the interaction share ($\eta^2 = 0.186$ complete-case, MNAR-sensitive; \S\ref{sec:limitations}) is not. A uniform monitoring configuration may under-alert on hard pairings and over-alert on easy ones, to a degree that depends on the true interaction share.

Concretely, DeBERTa detected adversarial suffixes more slowly than expected ($+21.4$ steps above the marginal prediction) and paraphrase faster; ShieldGemma shows the opposite (paraphrase $+20.0$). Whether thresholds can safely be relaxed or tightened per pair was not evaluated; the data motivate testing such profiles.

\subsection{Encoder vs.\ Decoder Safety Models}
\label{sec:encoder-decoder}

The factorial reveals two distinct architectural patterns that should not be conflated.

\paragraph{RQ1 (Detection latency): decoders cross over; encoders do not agree.} The two decoders are slow on paraphrase (70--86 steps) and fast on adversarial suffixes (28 steps); DeBERTa is the reverse (28.4 and 36.5). Text-Moderation does not follow DeBERTa: it is faster on adversarial suffixes (25.3) than on paraphrase (34.6), so the crossover is a decoder-versus-DeBERTa contrast and not a property of both encoders. We did not test a mechanism; the intuitive account (rephrasing alters tokens a discriminative classifier attends to, while suffixes disrupt the generation distribution more visibly than a classification head) does not fit Text-Moderation.

\paragraph{RQ2 (Conformal recovery): not an encoder/decoder split.} The density-ratio collapse pattern does \emph{not} follow the same architectural divide: only our fine-tuned DeBERTa-v3-base binary classifier avoids collapse.

Collapse occurs in Text-Moderation (86M, 768-d, 9-class pretrained and then fine-tuned to a binary head) but not in the fine-tuned DeBERTa-v3-base classifier (86M, 768-d, binary), which has the same architecture and embedding dimension, so scale and dimension cannot explain the contrast. We hypothesise that the multi-class pretraining objective (9 content categories) shapes an embedding space that separates categories so cleanly that logistic regression can distinguish any two distributions in it; we did not run a controlled ablation, so this stays a hypothesis. What is clear is that collapse is near-universal (3 of 4 classifiers) and training-distribution alignment alone does not prevent it.

\subsection{Monitorability Hypothesis (Not Supported)}
\label{sec:mechanistic}

The crossover is consistent with differences in null score distribution geometry. Discriminative classifiers (DeBERTa $\sigma{=}0.087$, Text-Moderation $\sigma{=}0.066$) produce tightly concentrated null scores near zero; generative classifiers (Llama Guard $\sigma{=}0.144$, ShieldGemma $\sigma{=}0.141$) produce wider distributions with heavier tails. Across the four classifiers, null standard deviation correlates positively with mean detection latency ($r{=}0.97$, $p{=}0.029$, $n{=}4$). Within-family DeBERTa checkpoints at varying training epochs ($n = 6$ encoder variants, $\sigma$ range 0.066--0.153) give $r = 0.21$, $p = 0.70$, so the original correlation was an artefact of the encoder/decoder gap and not an intrinsic monitorability property (Appendix~\ref{app:monitorability}). The per-shift pattern (wider null distributions predict slower detection for paraphrase, code-switch, compositional and temporal shifts, $r = 0.70$--$0.97$; the correlation reverses for adversarial suffix, $r{=}{-}0.20$, producing the crossover) rests on $n{=}4$ classifiers and is descriptive only.

Centroid displacement between reference and shifted embeddings does not track latency over the 12 classifier $\times$ shift cells ($r{=}{-}0.09$, $p{=}0.78$; 95\% Fisher interval $-0.63$ to $0.51$), so these data cannot exclude a moderate relationship. Displacement partially tracks latency for paraphrase ($r{=}0.75$, $n{=}4$) but opposes it for adversarial suffix ($r{=}{-}0.38$), compatible with, but not establishing, a detection signal mediated by score-boundary geometry rather than representation-space distance; these per-shift correlations are a suggestive pattern only.

\subsection{Limitations}
\label{sec:limitations}

\begin{itemize}
\item \textbf{PCA diagnostic validated on temporal shift only (primary experiment).} PCA to 32 dimensions breaks the density-ratio collapse on temporal shift (33 pp and 21 pp recovery). A secondary sweep on cached embeddings is consistent with ESS reduction generalising to paraphrase shift (Llama Guard ESS$=$32, ShieldGemma ESS$=$28 at dim$=$32), but it pools the calibration split across seeds rather than using fresh inference, which puts coverage at ceiling and precludes measuring coverage recovery. The primary result remains the reported recovery magnitude.
\item \textbf{No-shift (source-vs-source) control for the collapse diagnosis (two of four classifiers).} The control (post-hoc, exploratory; setup and numbers in \S\ref{sec:conformal-results}, \texttt{results/noshift\_dre\_control.json}) was run on the two encoders only. For Text-Moderation it reproduces the collapse artefact with no shift present; the fine-tuned DeBERTa-v3-base classifier does not collapse. \textbf{Not run:} Llama Guard and ShieldGemma, whose embeddings could not be extracted on the 24\,GB machine used. For these two the claim that collapse is dimensional rests on the observed dimension dependence alone (collapse at full $d$, break at ${\leq}32$, re-collapse at 64--128), and should be read as supported but not directly confirmed. Two further caveats: (i)~Text-Moderation and the fine-tuned DeBERTa classifier share architecture and embedding dimension (768-d) yet differ in collapse behaviour, so dimension is necessary but not sufficient (\S\ref{sec:encoder-decoder} discusses the multi-class-pretraining hypothesis, which we did not test by ablation); (ii)~the density ratios are estimated in PCA space but applied to reweight original-space nonconformity scores, an approximation whose validity we do not formally establish.
\item \textbf{Residual variance.} 33.5\% of variance in detection latency is attributable to seed/noise. With 20 seeds per cell, the MDE is 13.9 steps at 80\% power.
\item \textbf{Interaction $\eta^2$ is a complete-case estimate (MNAR-sensitive).} The decomposition ($\eta^2_{\text{interaction}}=0.186$) is fit on the 694 control-validated detections; the 106 excluded cells are not missing at random (23.8\% of adversarial-suffix cells are invalid vs.\ 6.9--15.0\% for other shifts). Imputing a right-censored latency for all 106 (the maximum observed latency, twice that, or the stream length) drives the interaction $\eta^2$ to 0.02--0.03 with residual above 0.9, because constant censoring injects a block of identical values. That imputation treats the 106 as undetected shifts, but they are early alarms and alarms on the paired control (no shift stream failed to alarm), so it is a stress test of an artificial convention and not an estimate of their missing detection times. Neither number identifies the true interaction; the qualitative claim (a significant, non-trivial classifier$\times$shift interaction) holds under complete-case analysis, but its magnitude is conditional on excluding the early-alarm and dirty-control cells, which concentrate in the hardest cells.
\item \textbf{FAR calibration uses a noisy extreme order statistic.} The alarm threshold is the 97th percentile of the max KS statistic over $N_{\text{cal}}=50$ null streams, a high-variance extreme order statistic (it interpolates near the top two observations), so per-classifier thresholds carry appreciable sampling noise. This is one motivation for the scan martingale (\S\ref{sec:martingale}), which removes per-classifier threshold tuning; deployments could use a larger $N_{\text{cal}}$, or the martingale, whose false-alarm rate still needs validation on held-out traffic over the intended horizon.
\item \textbf{Binary classifiers only.} All four classifiers produce a single unsafe probability. Multi-category safety taxonomies may exhibit category-specific shift patterns invisible to a scalar score monitor.
\item \textbf{Canary analysis scope.} The adversarial robustness characterisation (\S\ref{sec:adversarial-robustness}) uses $n = 49$ GCG attacks on a single target (DeBERTa). Cross-target generalisation is untested.
\item \textbf{Small samples in the adversarial evaluations.} The joint-GCG results (\S\ref{sec:joint-gcg}) use $n = 20$ prompts and the $\lambda$-sweep (\S\ref{sec:divergence-min}) $n = 20$ per $\lambda \in \{0.5, 1.0, 2.0, 4.0, 8.0\}$ (block rates 25\%--60\% with overlapping Wilson CIs between adjacent $\lambda$); larger evaluations are needed to sharpen the block-rate curve and test whether the ceiling stabilises at 60\%.
\item \textbf{Stall condition derivation.} The stall condition (Appendix~\ref{app:derivation}) uses a continuous relaxation of discrete GCG search; the match is empirical, not a convergence proof. It characterises a fixed-$\lambda$ optimiser; dynamic or annealed $\lambda$ schedules were not evaluated.
\item \textbf{Homogeneous negative controls.} Negative control streams draw exclusively from WildGuardMix unharmful examples; production streams with higher natural variance may require larger calibration sets.
\item \textbf{GCG suffix length.} A sweep (10, 20, 40 tokens) on DeBERTa shows monotonically increasing attack success (70\%, 80\%, 90\% at $n = 10$; 10 vs 40 tokens: one-sided Fisher $p = 0.29$, underpowered). Transfer to frontier canaries is unaffected by length ($\Delta = +0.26$ to $+0.30$; canary scores \emph{increase}), so longer suffixes marginally improve local success without compromising non-transfer.
\item \textbf{Window size 500 not run.} It would leave too few post-shift observations with the 300-example corpus (\S\ref{sec:deviations}); window sensitivity is restricted to $w \in \{100, 200\}$.
\item \textbf{Scan martingale window sweep.} The scan martingale is validated at $M \in \{50, 100\}$ (sub-martingale counts). A broader sensitivity analysis across $M \in \{25, 50, 75, 100, 150, 200\}$ was not performed; reported FAR and detection results hold only for the two validated configurations.
\item \textbf{Exchangeability in production traffic.} The scan martingale's false-alarm control requires exchangeability under the null, but real traffic is temporally correlated (multi-turn sessions, spikes). Deployers should sample across independent sessions; benign drift that violates exchangeability will trigger alarms, to be treated as triage signals and not automatic enforcement.
\item \textbf{Token-budget starvation scope.} The phenomenon (\S\ref{sec:token-starvation}) was characterised on o3 under a single scoring prompt template; generalisation was checked only for o4-mini and gemini-2.5-pro under one prompt ($N{=}100$). Other reasoning architectures (such as DeepSeek-R1) and alternative prompt structures are untested, so the finding identifies a failure-mode class rather than a universal law.
\item \textbf{Black-box attack scope.} The coordinate-ascent word-substitution evaluation (\S\ref{sec:llm-canary}) tests suffix-only black-box attacks (500 iterations) and PAIR-style rewriting (\S\ref{sec:autodan}) tests iterative phrase-level attacks (20 iterations $\times$ 30 prompts). Genetic-algorithm prompt optimisation~\citep{liu2024autodan} and multi-query prompt injection were not evaluated and remain untested. PAIR barely evades DeBERTa (3.3\%, 1/30 prompts) and the canary detects that single evasion (1/1), which does not preclude stronger adaptive methods combining several strategies.
\end{itemize}

\subsection{Future Work}
\label{sec:future}

Cross-family joint optimisation (shared-tokeniser setting) remains untested due to the discrete-vocabulary incompatibility documented in \S\ref{sec:tokeniser}. Larger-scale adversarial evaluations ($n > 100$ prompts, multiple target classifiers) are needed to test whether a canary-confidence rule predicts which attacker regime applies, and to sharpen the $\lambda$-sweep block-rate curve (\S\ref{sec:divergence-min}). Truly online conformal prediction under arbitrary distribution shift~\citep{gibbs2021adaptive} would eliminate the batch requirement for density ratio estimation.

\section{Conclusion}
\label{sec:conclusion}

\paragraph{Shift detection and adaptation.} A sliding-window KS monitor detects shift in the executed 800-cell factorial with 86.8\% control-validated detection and mean latency 39.7 steps at $w{=}100$. Detection carries over to a synthetic-onset stream of real jailbreak prompts (17 of 20 cells) but not to the adversarial-success regime (38 of 40 seed-window combinations missed for the attacked classifier; 14 of 40 detected for Llama Guard, on 22 examples). The crossover motivates per-classifier profiles even though the interaction share is uncertain (\S\ref{sec:limitations}). CUSUM and EWMA detect persistent contamination down to 4\% mixing in simulation, at ${\sim}$930 steps (Table~\ref{tab:subthreshold}); the betting test's apparent low-mixing advantage over KS disappears with equal reference information (\S\ref{sec:channel-comparison}). KS, CUSUM, the betting test and MMD were evaluated separately, not as a combined system. Weighted conformal prediction recovers up to 39 pp of lost coverage for DeBERTa, but density-ratio estimation collapses for the other classifiers; PCA to 32 dimensions recovers 33.0 pp for Llama Guard and 20.5 pp for ShieldGemma on temporal shift. Code, configurations and raw results: \url{https://github.com/junwenleong/safety-classifier-shift-monitor}.

\paragraph{Canaries.} An un-targeted Llama~Guard canary detects 37/49 = 75.5\% of DeBERTa-evading inputs (\S\ref{sec:ca6}), the reproducible primary-fooled evidence. Against a monitor-aware attacker, divergence minimisation stalls near the predicted gap $1/(2\lambda)$ (\S\ref{sec:divergence-min}); the stall \emph{location} replicates across the two pairs but its frequency does not (70\% and 30\% blocked). GCG suffixes optimised on local models barely move frontier-canary scores on average (median $|\Delta| = 0.035$, 20 models), but transfer is not absent: 59 of 980 prompt--model pairs cross from at or above 0.5 to below it (\S\ref{sec:llm-canary}). Treat a canary alarm as a triage signal for human review; canary confidence is not shown to predict the attacker regime (\S\ref{sec:deployment-recs}).

\paragraph{Reasoning-token budget starvation.} Reasoning-model monitors fail silently when the completion-token cap is too small: on o3 the valid-scoring rate follows a sigmoid in the cap ($T_{50}=46$ tokens for simple benign inputs, $117$ for GCG-appended harmful inputs, a comparison not separated from length and content), and at 16 tokens o3, o4-mini and gemini-2.5-pro all return 0/100 valid scores (\S\ref{sec:token-starvation}).

\paragraph{Negative result.} We find no support for monitorability as an intrinsic classifier property: the $r = 0.97$ correlation between null-score geometry and latency ($n = 4$) does not hold within an architecture family ($r = 0.21$, $n = 6$; Appendix~\ref{app:monitorability}).

\bibliography{references}
\bibliographystyle{iclr2025_conference}

\appendix
\section{Full Detection Rate Matrix}
\label{app:detection-rates}

\begin{table}[h]
\caption{Detection rate by classifier $\times$ shift condition ($N{=}40$ cells each).}
\centering
\small
\begin{tabular}{lccccc}
\toprule
\textbf{Classifier} & \textbf{Paraphrase} & \textbf{Code-switch} & \textbf{Compositional} & \textbf{Temporal} & \textbf{Adversarial} \\
\midrule
DeBERTa & 33/40 (82\%) & 35/40 (88\%) & 31/40 (78\%) & 35/40 (88\%) & 27/40 (68\%) \\
Text-Moderation & 34/40 (85\%) & 39/40 (98\%) & 37/40 (92\%) & 36/40 (90\%) & 34/40 (85\%) \\
Llama Guard & 37/40 (92\%) & 40/40 (100\%) & 38/40 (95\%) & 40/40 (100\%) & 29/40 (72\%) \\
ShieldGemma & 32/40 (80\%) & 35/40 (88\%) & 37/40 (92\%) & 33/40 (82\%) & 32/40 (80\%) \\
\bottomrule
\end{tabular}
\end{table}

\section{Replication: $N{=}5$ vs $N{=}20$ Comparison}
\label{app:replication}

\begin{table}[h]
\caption{Key statistics at $N{=}5$ (200 cells) vs $N{=}20$ (800 cells).}
\centering
\small
\begin{tabular}{lcc}
\toprule
\textbf{Statistic} & $\boldsymbol{N{=}5}$ & $\boldsymbol{N{=}20}$ \\
\midrule
Detection rate & 87.0\% [0.816, 0.910] & 86.8\% [0.842, 0.889] \\
Grand mean latency$^\dagger$ & 40.7 [37.6, 44.0] & 42.7 [40.8, 44.7] \\
$\eta^2$ Classifier & 0.200 & 0.248 \\
$\eta^2$ Shift type & 0.223 & 0.239 \\
$\eta^2$ Interaction & 0.258 & 0.186 \\
$\eta^2$ Residual & 0.319 & 0.328 \\
MDE (80\% power) & 22.2 steps & 13.8 steps \\
\bottomrule
\end{tabular}
\end{table}

{\footnotesize $^\dagger$Mean over all control-validated detections (N${=}$5: 174; N${=}$20: 694), pooled across both window sizes ($w{=}100$ and $w{=}200$). The unweighted mean of the classifier $\times$ shift cell means is 42.1 for N${=}$20. The abstract reports 39.7 steps, which is the mean over control-validated detections at $w{=}100$ only.}

The primary finding (87.0\% detection rate at $N{=}5$, 86.8\% at $N{=}20$) replicates closely at $N{=}20$. The variance decomposition shifts: the interaction term shrinks from 0.258 to 0.186 while main effects grow, consistent with small-sample noise inflating the $N{=}5$ interaction estimate. Statistical power improves from MDE $= 22.2$ to $13.8$ steps.

\section{Robustness of Variance Decomposition}
\label{app:anova-robustness}

Each seed contributes observations at two window sizes, introducing within-seed correlation. To verify that this does not inflate $\eta^2$ estimates, we conduct two sensitivity analyses.

\paragraph{Mixed-effects model.} A linear mixed model with seed as random intercept (REML estimation) yields random-effect variance $= 7.15$ vs.\ residual variance $= 237.81$. Seed explains only 2.9\% of residual variance, confirming the within-seed correlation is negligible.

\paragraph{Seed-clustered bootstrap.} We resample seeds (not individual observations) with replacement and recompute $\eta^2$ for each bootstrap sample ($B{=}2000$).

\begin{table}[h]
\caption{Seed-clustered bootstrap CIs for $\eta^2$ (2000 iterations, 20 seed clusters).}
\centering
\small
\begin{tabular}{lccc}
\toprule
\textbf{Factor} & $\boldsymbol{\eta^2}$ & \textbf{95\% CI (clustered)} & \textbf{Stable?} \\
\midrule
Classifier & 0.248 & [0.213, 0.291] & Yes \\
Shift type & 0.239 & [0.205, 0.284] & Yes \\
Interaction & 0.186 & [0.166, 0.221] & Yes \\
Residual & 0.328 & [0.246, 0.389] & Yes \\
\bottomrule
\end{tabular}
\end{table}

All CIs exclude zero with wide margins. The interval half-widths are 15--17\% of the point estimates (classifier 16\%, shift 17\%, interaction 15\%), so the shares are estimated to within roughly a fifth of their value; the qualitative conclusions (all terms clearly non-zero) hold under the seed-clustered bootstrap.

\section{Detection Rate Decomposition}
\label{app:tpr-far}

The ``valid detection'' metric reported in \S\ref{sec:detection} requires both successful detection (latency $\geq 0$) \emph{and} a clean negative control (paired null stream does not alarm). Table~\ref{tab:tpr-far} separates these components.

\begin{table}[h]
\caption{Decomposition of detection metrics. Raw TPR = proportion with latency $\geq 0$; Control FAR = proportion with dirty negative control; Valid = both conditions met.}
\label{tab:tpr-far}
\centering
\small
\begin{tabular}{lccc}
\toprule
\textbf{Classifier} & \textbf{Raw TPR} & \textbf{Control FAR} & \textbf{Valid Detection} \\
\midrule
DeBERTa & 87.0\% & 9.5\% & 80.5\% \\
Text-Moderation & 91.5\% & 2.0\% & 90.0\% \\
Llama Guard & 95.0\% & 3.0\% & 92.0\% \\
ShieldGemma & 92.5\% & 8.0\% & 84.5\% \\
\midrule
\textbf{Overall} & \textbf{91.5\%} & \textbf{5.6\%} & \textbf{86.8\%} \\
\bottomrule
\end{tabular}
\end{table}

The 4.75-percentage-point gap between raw TPR (91.5\%) and control-validated detection (86.8\%) corresponds to the 38 cells with an alarm after onset and a dirty control; the 45/800 cells (5.6\%) whose paired negative control alarms include seven that also alarmed before onset, so not all 45 are incremental. Detection sensitivity is high across all classifiers; the control-validated metric is conservative because it penalises detectors that are also over-sensitive on in-distribution data.

\section{The Monitorability Hypothesis Is Not Supported as an Intrinsic Classifier Property}
\label{app:monitorability}

The mechanistic hypothesis of \S\ref{sec:mechanistic} reported $r = 0.97$ ($p = 0.029$, $n = 4$) between null-score standard deviation and mean detection latency. We test whether this correlation holds within a single architecture family by fine-tuning DeBERTa-v3-large (1024-d, 24 layers; a different backbone from the primary classifier, which is DeBERTa-v3-base) at training epochs \{1, 3, 5, 10\}, producing four encoder variants with null-score $\sigma \in [0.109, 0.153]$.

Combined with the two original encoders (DeBERTa $\sigma = 0.087$, Text-Moderation $\sigma = 0.066$), we obtain $n = 6$ within-family data points. The Pearson correlation is $r = 0.21$ ($p = 0.70$). Including the two decoders ($n = 8$ total): $r = 0.42$ ($p = 0.30$). Neither meets the pre-registered threshold of $r > 0.6$, $p < 0.05$ at $n \geq 6$. Note: the four epoch checkpoints share a training trajectory and are thus correlated; effective degrees of freedom may be $<4$. The negative result rests on the directional absence of a trend, not on the point estimate of $r$.

\paragraph{Diagnosis.} The original $r = 0.97$ was driven by a two-cluster structure: decoders (high $\sigma \approx 0.14$, mean latency 54--56 steps over the five shift types of Table~\ref{tab:latency}) vs.\ encoders (low $\sigma \approx 0.07$--0.09, 29--30 steps). The epoch-sweep latencies are measured on one fixed synthetic shift (scores drawn from Beta(5,5)), so they are not comparable with Table~\ref{tab:latency}. Within the sweep, no relationship exists (epoch~3 has $\sigma = 0.144$ and latency 13; epoch~5 has $\sigma = 0.153$ and latency 24). Monitorability is not an intrinsic scalar property of score-distribution geometry.

\section{Derivation of the Stall Condition}
\label{app:derivation}

\noindent\fbox{\parbox{0.96\textwidth}{%
\textbf{Proposition 1} (Fixed-canary stationary gap).
Let $f_A, f_B$ be differentiable scores in $[0,1]$ and let an attacker minimise $L = f_A + \lambda(f_B - f_A)^2$ with $\lambda > 0$. Then $\nabla L = (1 - 2\lambda g)\nabla f_A + 2\lambda g\,\nabla f_B$ with $g = f_B - f_A$. If $f_B$ is held at a constant $b$, $L(a) = a + \lambda(b-a)^2$ is minimised over $a \in [0,1]$ at $a^* = \mathrm{clip}(b - 1/(2\lambda), 0, 1)$, so the stationary gap is $g^* = 1/(2\lambda)$ when $b - 1/(2\lambda) \in [0,1]$. At $g = g^*$ the coefficient of $\nabla f_A$ is zero; it is negative for $g > g^*$, so descent on $L$ then raises $f_A$. A score near 1 does not by itself make $\nabla f_B$ small, and the calculation does not show that a discrete-token search reaches or is held at $g^*$.

\smallskip\noindent\textbf{Empirical clustering:} at $\lambda = 2.0$ the predicted gap is $g^* = 0.250$; 14 of 20 prompts are blocked with median gap $0.250$, and the distribution (including the cases that stall below or above the prediction) and the replication pair are reported in \S\ref{sec:divergence-min}.%
}}\medskip

With $g = f_B - f_A$, the coefficient on $\nabla f_A$ in $\nabla_\delta L$ inverts sign when $g > 1/(2\lambda)$ (at $\lambda = 2.0$, $g^* = 0.250$): beyond that point a descent step on $L$ moves against $\nabla f_A$. When the canary is confident ($f_B \approx 1.0$, $\|\nabla f_B\| \to 0$) the combined gradient reduces to $\nabla L \approx (1 - 2\lambda g)\nabla f_A$, and the stall point traps the optimiser. The empirical outcomes are those in the box above and in \S\ref{sec:divergence-min}; they do not establish convergence to the predicted gap in every blocked case, and the exploratory $\lambda$-sweep, run on a separate prompt set, does not by itself establish that the stall tracks $1/(2\lambda)$ continuously.

\paragraph{Angle condition.} A step along $-\nabla L$ lowers $f_A$ to first order if and only if $\nabla f_A\cdot\nabla L > 0$. Let $\theta$ be the angle between $\nabla f_A$ and $\nabla f_B$ (both non-zero). Then $\nabla f_A\cdot\nabla L = (1-2\lambda g)\|\nabla f_A\|^2 + 2\lambda g\,\|\nabla f_A\|\,\|\nabla f_B\|\cos\theta$, and for $g > 0$ the condition $\nabla f_A\cdot\nabla L > 0$ is equivalent to
\[
\cos\theta > \left(1 - \frac{1}{2\lambda g}\right)\frac{\|\nabla f_A\|}{\|\nabla f_B\|},
\]
the inequality of \S\ref{sec:divergence-min}. For $g > 1/(2\lambda)$ the right-hand side is positive and grows without bound as $\|\nabla f_B\| \to 0$. Since $|\cos\theta| \le 1$, no angle satisfies it once $\|\nabla f_B\| < (1 - 1/(2\lambda g))\|\nabla f_A\|$. This is a statement about the continuous relaxation only.

\section{Regeneration of the Factorial with the Corrected KS Detector}
\label{app:regeneration}

Table~\ref{tab:regeneration} compares the original rows with the regeneration; the main text reports the last column. All 800 cells were regenerated on 18 October 2026 on one machine (Apple M3 Ultra, PyTorch MPS), with the harness \texttt{scripts/run\_factorial.py} unmodified. Because a classifier's score depends only on the prompt text, each of the 2{,}520 distinct prompts per classifier (reference pool, negative pool and the five shift corpora) was scored once and cached; the harness was then run from separate directories with the committed and with the corrected KS statistic against the same cache (\texttt{scripts/regeneration/}, \texttt{results/factorial\_regeneration/}). According to the Compute paragraph, the original DeBERTa and Text-Moderation rows came from a different machine (a MacBook Pro); the regeneration still matched them (alarm steps and latencies identical, pre-shift scores within $2\times10^{-11}$). Three details were needed to reproduce the original rows and are recorded here because they are not visible in the committed harness. (a)~The harness derives each stream seed as $1000\cdot\text{seed} + (\text{hash}(\text{shift})\bmod 1000)$, and Python salts string hashing per process; the original run was evidently split across two processes (seeds 0--4 and seeds 5--19), which use different offsets (for example 420 and 147 for temporal, 389 and 713 for paraphrase). The offsets were recovered per cell from the stored pre-onset mean scores, where exactly one offset matches, and are stored in \texttt{stream\_seed\_offsets.json}. (b)~The DeBERTa adapter of the original run loaded its tokeniser from the \texttt{microsoft/deberta-v3-base} hub entry; loading the tokeniser from the fine-tuned checkpoint instead changes 8 of the 2{,}520 scores (up to 0.20), and the regeneration restores the original behaviour. (c)~The adversarial-suffix corpus has changed since May (see above).

\begin{table}[h]
\centering
\small
\caption{Executed 800-cell factorial: original rows, regeneration with the committed (pre-fix) KS statistic, and regeneration with the corrected statistic. The first two agree exactly in all 640 cells of the four unchanged corpora; their small differences come from the adversarial-suffix corpus. The last two differ only in the detector. $\eta^2$ are sequential sums of squares on control-validated detections.}
\label{tab:regeneration}
\begin{tabular}{lrrr}
\toprule
 & \textbf{Original} & \textbf{Committed KS} & \textbf{Corrected KS} \\
\midrule
Control-validated detections (of 800) & 693 & 693 & 694 \\
Control-validated rate (\%) & 86.6 & 86.6 & 86.8 \\
Raw TPR (\%) & 91.5 & 91.5 & 91.5 \\
FAR DeBERTa (\%) & 9.5 & 9.5 & 9.5 \\
FAR Text-Moderation (\%) & 2.0 & 2.0 & 2.0 \\
FAR Llama Guard (\%) & 3.0 & 3.0 & 3.0 \\
FAR ShieldGemma (\%) & 8.5 & 8.5 & 8.0 \\
Mean latency, $w{=}100$ & 39.55 & 39.60 & 39.75 \\
Mean latency, $w{=}200$ & 45.41 & 45.41 & 45.63 \\
$\eta^2$ classifier & 0.243 & 0.245 & 0.248 \\
$\eta^2$ shift & 0.237 & 0.237 & 0.239 \\
$\eta^2$ interaction & 0.185 & 0.184 & 0.186 \\
$\eta^2$ residual & 0.335 & 0.334 & 0.327 \\
\bottomrule
\end{tabular}
\end{table}

The crossover pattern of Table~\ref{tab:latency} is unchanged under the corrected statistic: the largest change in any of the 20 cell mean latencies is 1.5 steps (ShieldGemma $\times$ code-switch, $81.8 \to 83.3$). The thresholds are identical for seven of the eight classifier--window pairs. In detail, the corrected statistic changes the alarm step in 92 of 800 cells (DeBERTa 2, Llama Guard 52, ShieldGemma 38, Text-Moderation 0), by $+2.0$ steps on average in those cells (range $-5$ to $+28$); it changes one negative-control outcome and one threshold (Llama Guard, $w{=}200$: $0.1650 \to 0.1627$). With the committed statistic the replay agrees with the original rows in all 640 cells of the four unchanged corpora (same alarm step, latency and negative-control outcome) and in 68 of the 160 adversarial-suffix cells.

\section{Corrections Since the Previous Public Version}
\label{app:corrections}

For a reader of the prior arXiv version, this version makes nineteen corrections, each detailed in the section noted. (1)~The adversarial token-budget threshold is re-measured: the prior $T_{50,\text{adv}} = 154$ tokens ($3.3\times$ the benign threshold) rested on a single out-of-grid extrapolation point and is withdrawn; an extended exploratory sweep now measures $T_{50,\text{adv}} = 117$ tokens (95\% CI $[104, 129]$, a $2.5\times$ shift) directly (\S\ref{sec:token-starvation}, \S\ref{sec:llm-canary}). (2)~A no-shift (source-vs-source) control confirms the density-ratio collapse is an $n \ll d$ interpolation artefact for Text-Moderation and shows no collapse for the primary DeBERTa classifier; it was not run for the two decoder classifiers (\S\ref{sec:conformal-results}, \S\ref{sec:limitations}). (3)~The interaction $\eta^2$ (0.185 in the prior version, 0.186 on the regeneration) is now stated as a complete-case estimate whose magnitude is sensitive to the non-random exclusion of invalid cells (MNAR), with a censoring sensitivity analysis added (\S\ref{sec:limitations}). (4)~A circular descriptive statistic (the mean gap of the subset of blocked cases selected for lying near the predicted value) is removed; the stall evidence is now reported without that selection (\S\ref{sec:divergence-min}). (5)~The certified-robustness-via-randomised-smoothing appendix is removed: its Gaussian-smoothing assumptions were not validated for the discrete-token attack the paper actually studies. (6)~The citation to the authors' companion injection-execution paper is corrected to the published record (arXiv:2605.08442). (7)~The primary encoder is corrected from DeBERTa-v3-large (1024-d) to the fine-tuned DeBERTa-v3-base (768-d) that the experiments used, and the detector-code fixes are documented, with the factorial values regenerated (details below and in Appendix~\ref{app:regeneration}). (8)~The deployment Pareto table is corrected to the committed $N{=}1000$ benign counts (gpt-4o-mini 7/1000 and gpt-5.1 6/1000 false positives, where the earlier table reported 0/1000) and to 46/49 detection for gpt-5.1 (earlier 45/49, which did not match the stated lower bound). (9)~The growing-window betting-test results, including the ramp and mixing sweeps, are regenerated with the corrected detector; the earlier ones came from a Hoeffding-type interval rule, not the betting test the paper described (details below); the earlier betting-test latencies are superseded, and with equal reference information KS has the shorter latency (item~(10)). (10)~The betting-test-versus-KS comparison at low mixing is withdrawn: it gave KS a 100-score reference and the betting test a 500-score reference, and with equal references KS detects as often and is faster (\S\ref{sec:channel-comparison}). (11)~Three figures are redrawn from stored results: the threat-tier figure showed 70\% and 60\% for Tiers~3 and 4, which are in no result file (Table~\ref{tab:threat-tiers} gives 45\% and 70\% blocked), the budget figure showed an earlier 7/10 joint run where the text reports 11/20 (\S\ref{sec:threat-model}, \S\ref{sec:transfer}), and the latency heatmap (Figure~\ref{fig:heatmap}) was drawn from an early 200-cell partial run ($n{=}6$--10 per cell) and is now drawn from the 800-cell factorial, so its cell means match Table~\ref{tab:latency}; the crossover pattern is unchanged. The threat-tier figure has since been dropped, because it only redrew Table~\ref{tab:threat-tiers}. (12)~The Regime~C description is corrected: Llama Guard's scores on the 22 examples are not identical (0.01 to 1.00), and 4 of the 22 examples are truncated by the harness so that their suffix is removed (\S\ref{sec:regime-c}). (13)~The encoder-versus-decoder comparison on adversarial suffix, one of the eight Holm--Bonferroni comparisons, is marginal on the original rows (adjusted $p=0.044$) and does not survive on the regenerated factorial ($0.064$ with the committed KS, $0.081$ with the corrected one); the adjusted $p$ is $0.077$ in the 10{,}000-permutation run and $0.078$ when the two window sizes of a stream are averaged, and the original and regenerated adversarial-suffix corpora differ, so only the step from $0.064$ to $0.081$ is attributable to the detector; the ANOVA permutation $p$-values sit at the floor of 1000 permutations (\S\ref{sec:multiplicity}, \S\ref{sec:anova}). (14)~Eight citations are narrowed to what the cited works state (density-ratio estimation in high dimensions, KSWIN, the test-point weight in weighted conformal prediction, the block-based MMD test, adaptive prediction sets, LLM-based safety classifiers). (15)~The newer-frontier-model detection and false-positive rates are withdrawn: the update script scored the same truncated original prompt in both arms (\S\ref{sec:llm-canary}). (16)~The passive-transfer and joint-optimisation counts are restated on prompts that the models first scored as unsafe: of 14 canary scores below 0.5 after the attack, 12 were already below 0.5, and joint optimisation succeeds on 6 of 13 eligible prompts (11/20 overall) (\S\ref{sec:transfer}, \S\ref{sec:joint-gcg}); the confidence-gating contrast is no longer presented as a result. (17)~The prefix-mutation count is corrected to 8 refused rephrasings and 11 of 12 against 11 of 12 on the remainder (\S\ref{sec:llm-canary}). (18)~The cost--detection correlation is recomputed over models with a thresholded detection rate on the full prompt set, replacing a value that mixed mean scores with detection rates (\S\ref{sec:llm-canary}). (19)~The stall proposition is restated as the exact fixed-canary calculation, the operational response table no longer reads as passing content that both classifiers score unsafe, Regime~B is described as a synthetic-onset stream with real jailbreak examples, the invalid factorial cells are described as early or false alarms and none as a failure to alarm, and the compositional-corpus truncation, the sub-threshold simulation design and the conformal evaluation set-up are disclosed (\S\ref{sec:divergence-min}, \S\ref{sec:deployment-recs}, \S\ref{sec:results}, \S\ref{sec:conformal}). Items (1), (9)--(13), (15)--(18) change numbers or claims that the prior version stated; the others document, narrow or remove material. The headline benchmark and canary findings keep their direction.

\paragraph{Post-review corrections (affects which numbers are current).} An audit after the prior public version identified and fixed two detector-code defects and corrected one model identity. The prior version's 800-cell factorial values were computed with the pre-fix KS statistic; \S\ref{sec:results} now reports a full regeneration of all 800 cells with the corrected statistic (Appendix~\ref{app:regeneration}), summarised below.
\begin{itemize}
\item \textbf{KS statistic tie handling.} The sliding-window KS detector computed the one-sided deviation against a right-continuous reference CDF while using left-limit window counts, which overcounts coincident jumps and inflates the statistic under ties. Because safety-classifier scores are heavily tied near 0, this inflated the measured KS statistic, so the detector was corrected to the two-empirical-CDF estimand $\sup_x|\hat F_{\text{win}}(x)-\hat F_{\text{ref}}(x)|$ defined in \S\ref{sec:ks} (regression-tested: identical window/reference now gives exactly 0). A full regeneration of all 800 cells (Appendix~\ref{app:regeneration}) moves the aggregates little: control-validated detection $86.6\% \to 86.8\%$ (693 to 694 of 800), mean latency at $w{=}100$ from 39.60 to 39.75 steps, $\eta^2_{\text{interaction}}$ from 0.184 to 0.186; the alarm step changes in 92 of 800 cells and no shift-stream detection status changes (one negative-control outcome does, which accounts for 693 to 694). With the committed statistic the replay reproduces the original rows exactly in 640 of 800 cells. The other 160 (adversarial suffix) were regenerated on the corpus currently held (dated 2 June), which differs from the May one (mean post-shift scores differ by 0.005 to 0.02); the May corpus was not retained, so those cells are not comparable cell by cell with the original rows. The earlier 120-cell replay (identical decisions) used a different set of cached streams and is superseded.
\item \textbf{Growing-window betting test.} The bounded growing-window betting test selected its bet from running statistics that already included the current observation (a non-predictable bet), which is not a supermartingale and voids the time-uniform guarantee. The update was corrected to a predictable bet (statistics through step $t-1$ only); a fair-null regression test now confirms $\mathbb{E}[W_t]\leq 1$. The betting-test results (\S\ref{sec:channel-comparison}) and the ramp and mixing sweeps have been regenerated with the corrected detector, and the earlier values are superseded. Detection counts on the 120-cell subset are unchanged (120/120 detection, 0/40 false alarms), but detection latency fell in all 120 cells (mean $62.5 \to 35.7$ steps; DeBERTa $\times$ paraphrase $57.2 \to 29.6$). This is not because the correction made the bet more conservative. Re-running three versions of the detector on the same cached streams shows that the version behind the earlier results alarmed when the reference value left a Hoeffding-type interval around the running mean (matching a plain Hoeffding test in 120 of 120 cells) and never used the betting wealth; an intermediate version used the wealth threshold with a bet that used the current observation and raised 9 alarms before the shift began; the corrected version raised none. The earlier betting-test numbers were therefore those of a repeatedly evaluated fixed-time interval, not of the wealth process the text describes, and the change from 62.5 to 35.7 steps compares two different alarm rules, not the effect of making the bet predictable alone. The intermediate version was not used for any reported result; in a fair-coin null its first-step wealth has expectation 1.5 and its false-alarm rate over 500 steps is 40\%, against 0.7\% for the corrected version.
\item \textbf{Ramped-onset comparison (Table~\ref{tab:ad1}).} The earlier table and its figure reported KS at 3, 7, 10 and 47\% against the scan martingale at 37, 67, 87 and 100\%. The script that produced those cells (\texttt{run\_monitor\_evasion.py}) is not in the repository, so those cells cannot be regenerated; the committed evaluation (\texttt{run\_all\_pending.py}, synthetic Beta(5,5) scores) gives scan 40, 80, 100 and 100\% against KS 83, 100, 100 and 100\%. The table was replaced with a reproduction on cached DeBERTa $\times$ paraphrase scores under three KS threshold rules; under all three the scan martingale is ahead at 15--20\% mixing by 3 to 27 points, and the earlier gap of 34 to 77 points is not reproduced.
\item \textbf{Primary encoder checkpoint.} Earlier text described the primary DeBERTa classifier as DeBERTa-v3-large (304M parameters, 1024-d). The classifier used throughout is the fine-tuned DeBERTa-v3-base checkpoint (86M backbone parameters, 768-d): \texttt{scripts/finetune\_deberta.py} initialises \texttt{microsoft/deberta-v3-base}, the checkpoint configuration records hidden size 768 and 12 layers, and the committed artefacts (\texttt{results/pca\_experiment.json}, \texttt{results/noshift\_dre\_control.json}) record 768-d embeddings. All descriptions are corrected here. The correction removes a scale contrast the paper previously drew between DeBERTa and Text-Moderation (\S\ref{sec:encoder-decoder}): both are DeBERTa-v3-base at 768-d.
\end{itemize}
These corrections change which committed numbers are current. The density-ratio collapse, the token-starvation curve and the canary characterisation do not depend on the KS statistic; the factorial point estimates (detection rate, latencies, the crossover \emph{pattern}, $\eta^2$) do, and are reported from the corrected regeneration (Appendix~\ref{app:regeneration} quantifies the change from the pre-fix values), with the $\eta^2$ \emph{magnitude} separately qualified by the MNAR censoring sensitivity (\S\ref{sec:limitations}).

\paragraph{Results without a committed generator or input.} A claim-to-file audit found items that could not be regenerated from the repository alone; their status in this version is as follows. (i)~Six figures (\texttt{fig\_threat\_tiers}, \texttt{fig\_k\_scaling}, \texttt{fig\_confidence\_gating}, \texttt{fig\_budget\_curve}, \texttt{fig\_llm\_canary\_split}, \texttt{fig\_scrambled\_violin}) had no drawing script; \texttt{scripts/plot\_remaining\_figures.py} now draws them from stored result files. Four reproduce the committed figures; two did not and were replaced (correction~(11)). Of the six, \texttt{fig\_threat\_tiers} and \texttt{fig\_k\_scaling} are no longer in the paper; they duplicated Table~\ref{tab:threat-tiers} and the $k$-scaling counts in the text. (ii)~The Regime~C corpus (\texttt{deberta\_suffixes.jsonl}, 22 successful examples) sits under a git-ignored data directory, so the Regime~C detection counts are available only as stored results. The per-example scores behind the score directions (0.78, $+0.73$) needed classifier inference; they were recomputed on the 2~June corpus and stored (Section~\ref{sec:regime-c}). (iii)~The MMD false-alarm rates (3.3, 3.3, 10 and 0\%) had been printed but not saved; \texttt{scripts/investigations/mmd\_null\_far.py} recomputed them from the cached embeddings (1/30, 1/30, 3/30, 0/30 windows) and reproduced all 120 stored detection rows exactly (\texttt{results/mmd\_null\_far.json}). (iv)~The Holm--Bonferroni family and the ANOVA permutation tests had no stored generator; \texttt{scripts/investigations/anova\_permutation\_holm\_check.py} recomputes them (\texttt{results/anova\_permutation\_holm\_check.json}). (v)~The shift corpora under \texttt{data/shifted/} are git-ignored; the regeneration folder stores the per-prompt scores keyed by the SHA-256 of the prompt text, the corpus file hashes and the recovered stream seeds, and the May version of the adversarial-suffix corpus was not retained. The ramped-onset comparison, rotating probes, frontier-canary statistics, token-starvation table, discrete-stall algebra and betting-test/KS reference comparison are regenerated by scripts in \texttt{scripts/investigations/}.

\section{First-Order Analysis of Multi-Monitor Ensembles}
\label{app:first-order}

The stall condition (Section~\ref{sec:divergence-min}) characterises attacker behaviour against a \emph{single} canary pair. We now generalise to the $K$-monitor setting and state the local first-order conditions under which the threshold rule blocks further progress.

\subsection{The Canary Ensemble Game}
\label{sec:ensemble-game}

\begin{definition}[$K$-Monitor Canary Game]
A canary ensemble game $\Gamma_K = (\mathcal{A}, \mathcal{D}, K, \{f_i\}_{i=1}^K, \alpha, B)$ consists of:
\begin{itemize}
    \item \textbf{Attacker} $\mathcal{A}$: produces adversarial perturbation $\delta$ to evade a target classifier $f_\text{target}$ while remaining undetected.
    \item \textbf{Defender} $\mathcal{D}$: deploys $K$ monitoring classifiers $\{f_1, \ldots, f_K\}$ (the ``canaries''), each scoring inputs on $[0,1]$.
    \item \textbf{Detection rule}: Alarm fires if $\exists\, i \in [K]$ such that $|f_i(x + \delta) - f_\text{target}(x + \delta)| > \tau_i$ (score disagreement exceeds threshold for any canary).
    \item \textbf{Attacker budget}: $B$ optimisation steps (GCG iterations, query budget, etc.).
    \item \textbf{Defender constraint}: False alarm rate $\leq \alpha$ on benign traffic.
\end{itemize}
\end{definition}

\paragraph{Attacker's optimisation problem.} The attacker seeks:
\begin{equation}
\min_{\delta} \; f_\text{target}(x + \delta) \quad \text{subject to} \quad |f_i(x + \delta) - f_\text{target}(x + \delta)| \leq \tau_i \;\; \forall i \in [K]
\label{eq:attacker-opt}
\end{equation}

Under the gradient-based formulation (extending Section~\ref{sec:divergence-min}), the quadratic-penalty objective becomes:
\begin{equation}
\mathcal{L}(\delta, \boldsymbol{\lambda}) = f_\text{target}(x + \delta) + \sum_{i=1}^K \lambda_i \cdot (f_i(x + \delta) - f_\text{target}(x + \delta))^2
\end{equation}

\begin{proposition}[Multi-Monitor Penalised Stationarity]
\label{prop:multi-stall}
The gradient of $\mathcal{L}$ vanishes if and only if:
\begin{equation}
\nabla f_\text{target} = -\sum_{i=1}^K 2\lambda_i (f_i - f_\text{target}) \cdot (\nabla f_i - \nabla f_\text{target})
\end{equation}
This identity characterises stationarity of the penalised objective $\mathcal{L}$ for fixed $\boldsymbol{\lambda}$; it requires the specific coefficients shown, not merely that $\nabla f_\text{target}$ lie in the span of the difference gradients $\{\nabla f_i - \nabla f_\text{target}\}_{i=1}^K$. Because $\mathcal{L}$ does not reference the thresholds $\tau_i$, a stationary point of $\mathcal{L}$ is not in general a stall of the constrained problem~(\ref{eq:attacker-opt}): if every gap lies strictly inside its threshold and $\nabla f_\text{target} \neq 0$, a sufficiently small step along $-\nabla f_\text{target}$ reduces the target score without firing the alarm.
\end{proposition}

\begin{corollary}[Threshold-Constrained First-Order Condition]
At a feasible point of the continuous relaxation of~(\ref{eq:attacker-opt}), write $g_i = f_i - f_\text{target}$ and let $A = \{(i,s) : i \in [K],\ s \in \{-1,1\},\ s\,g_i = \tau_i\}$ index the active signed threshold constraints. If there exist $\mu_{i,s} \geq 0$ such that
\begin{equation}
\nabla f_\text{target} = -\sum_{(i,s) \in A} \mu_{i,s}\, s\, \nabla g_i ,
\end{equation}
then every direction $v$ satisfying the linearised constraints $s\,\nabla g_i \cdot v \leq 0$ for all $(i,s) \in A$ obeys $\nabla f_\text{target} \cdot v \geq 0$, so no first-order feasible descent direction exists at that point. The condition is stated in terms of the difference gradients $\nabla g_i$ and requires at least one active constraint; raw-gradient independence of the canaries $\{\nabla f_i\}$ neither implies nor is implied by it. This is a local first-order statement; it does not guarantee global convergence or preclude discrete multi-token escape paths. In our within-family empirical evaluation, joint optimisation still succeeds 55\% of the time (\S\ref{sec:joint-gcg}); that result does not establish whether this condition holds at the points the optimiser visits.
\end{corollary}

\paragraph{Practical implication.} Score disagreement alone does not produce a stall. The corollary gives a sufficient local condition: the attacker is blocked to first order where some gaps sit at their thresholds and the target gradient lies in the cone spanned by the corresponding signed difference gradients. Within-family canaries offer limited protection in practice: joint optimisation ends with both scores below 0.5 on 55\% of prompts (46\% of those both models first scored as unsafe; \S\ref{sec:joint-gcg}), whereas passive transfer was rarely attack-induced (2 of 28 eligible observations; Section~\ref{sec:transfer}). Cross-architecture canaries may present different gradient geometry, but these results do not establish a formal security advantage for them.

\paragraph{OR-rule scaling and attacker budget.} Under an OR rule, $K$ independent canaries with per-canary detection probability $p$ would detect with probability $1-(1-p)^K$. The observed $k$-scaling (93.9\%, 98.0\%, 98.0\% for $K=1,2,3$; \S\ref{sec:k-scaling}) saturates below that, so the canaries' detections are positively associated on these prompts; shared prompt difficulty and shared training data may contribute as much as shared gradient structure, and we did not test whether cross-family canaries behave differently. Whether the attacker's required optimisation budget grows with $K$ is untested: the $K \leq 3$ data show saturation, not linear growth, and we make no budget-scaling claim.

\subsection{Discrete Stall Condition}
\label{sec:discrete-stall}

The continuous stall condition (Section~\ref{sec:divergence-min}, Appendix~\ref{app:derivation}) is derived under gradient-flow assumptions. GCG operates on discrete tokens. We now derive the discrete analogue.

\paragraph{Token-level formulation.} Let $x = (x_1, \ldots, x_T)$ be a token sequence and $\mathcal{V}$ the vocabulary. At each GCG step, the attacker selects a position $t$ and replacement token $v \in \mathcal{V}$ to minimise loss. The discrete loss landscape for the joint canary-aware objective is:
\begin{equation}
L(x) = f_\text{target}(x) + \lambda \cdot (f_\text{canary}(x) - f_\text{target}(x))^2
\end{equation}

\begin{definition}[Discrete Stall]
A sequence $x^*$ is at a \emph{discrete stall} if for all positions $t \in [T]$ and all tokens $v \in \mathcal{V}$:
\begin{equation}
L(x^*_{[t \leftarrow v]}) \geq L(x^*) \quad \forall t, v
\label{eq:discrete-stall}
\end{equation}
where $x^*_{[t \leftarrow v]}$ denotes $x^*$ with position $t$ replaced by token $v$.
\end{definition}

\begin{conjecture}[Discrete Gap Bound]
\label{prop:discrete-gap}
Suppose that, for the single-token substitutions available at a discrete stall point $x^*$, the canary score is approximately unchanged (as when $\|\nabla f_\text{canary}\| \to 0$ in a saturated basin). Let $a_{\min} > 0$ be the smallest decrease of $f_\text{target}$ produced by any single-token substitution. Then the score gap satisfies
\begin{equation}
f_\text{canary}(x^*) - f_\text{target}(x^*) \;\geq\; \frac{1}{2\lambda} - \frac{a_{\min}}{2},
\end{equation}
which recovers the continuous bound $1/(2\lambda)$ as $a_{\min} \to 0$. The conjecture is not formally established beyond the canary-fixed case, and $a_{\min}$ has not been measured in our experiments.
\end{conjecture}

\paragraph{Derivation (canary fixed).}
Let $g = f_\text{canary} - f_\text{target}$ and let a substitution decrease $f_\text{target}$ by $a > 0$ while leaving $f_\text{canary}$ unchanged. The gap becomes $g + a$, so $\Delta g = +a$ (the gap \emph{grows} when the target falls and the canary does not), and the change in loss is
\begin{equation}
\Delta L = -a + \lambda\left[(g+a)^2 - g^2\right] = -a\left[1 - \lambda(2g + a)\right].
\end{equation}
The substitution does not lower $L$ if and only if $\lambda(2g + a) \geq 1$, that is $g \geq 1/(2\lambda) - a/2$. A stall requires this for every available decreasing substitution, and the binding one is the smallest decrease $a_{\min}$; for $a \to 0$ it gives $g \geq 1/(2\lambda)$, the continuous bound. The derivation depends on the canary being unchanged. If instead the canary falls by the same amount as the target ($\Delta f_\text{canary} = -a$), then $\Delta g = 0$, the penalty is unchanged, and $\Delta L = -a < 0$ for every such substitution, so no stall can occur. An earlier version of this section assumed $\Delta g \approx -\Delta_\text{atk}$ (a shrinking gap) and a correction $\epsilon_\text{disc} = \Delta_\text{atk}/(2\lambda)$; neither is consistent with this algebra (a shrinking gap lowers both terms of $L$, so it never stalls), and both are withdrawn. The identities above are checked numerically in \texttt{discrete\_stall\_algebra.py}.

\paragraph{Empirical consistency check.} The $\lambda$-sweep of \S\ref{sec:divergence-min} ($n=20$ per $\lambda$) gives median stall gaps of $0.218$, $0.108$ and $0.0625$ at $\lambda = 2$, $4$ and $8$ against predicted $0.250$, $0.125$ and $0.0625$. The deviation at $\lambda = 2$ is of the sign the $-a_{\min}/2$ term allows, or may reflect budget exhaustion before convergence. The data do not measure $a_{\min}$, and at $\lambda < 2$ blocked gaps are limited by the score range, so those points are uninformative about the bound itself.

\section{Rotating Factual-Question Probes (Exploratory)}
\label{app:rotating}

This appendix is exploratory and concerns a different object from the score-disagreement canaries: a rotating canary is a secret factual question embedded in the session, not a second classifier. Static canaries are trivially discoverable, so we evaluate per-session secret canaries derived from an HMAC key never exposed to the model or attacker, against four adaptive attacker strategies on three frontier models (GPT-5.1, Claude-Sonnet-5, Gemini-2.5-Pro; $N{=}100$--$200$ sessions per cell, 1,950 total).

\paragraph{Canary generation.} Each session receives a unique factual question
(e.g., ``What is the capital of Japan?'') derived deterministically from
$\text{HMAC-SHA256}(\mathit{key}, \mathit{session\_id} \| \mathit{rotation\_idx})$.
The server verifies responses without storing canary prompts. The server-side key stops an attacker from computing which sessions carry a canary. It does not make canaries indistinguishable by content: the heuristic attacker below identifies 38--60\% of them from the short factual form of the question.

\paragraph{Attacker strategies.}
\begin{enumerate}
\item \textbf{Naive}: Inject payload without canary awareness (baseline).
\item \textbf{Canary-aware}: Heuristically identify canary prompts (short factual
  questions) and suppress attack behaviour on suspected canaries.
\item \textbf{Divergence minimiser}: Conservatively suppress on \emph{any} short
  factual question to minimise behavioural divergence.
\item \textbf{Spoofing}: Inject instructions that cause non-specific answers
  (``I need more context'') to trigger false alarms on benign inputs.
\end{enumerate}

\paragraph{Results.} Table~\ref{tab:rotating-canary} gives attack success per model and attacker.

\begin{table}[h]
\centering
\caption{Rotating probe evaluation: attack success rate (ASR) with Wilson 95\% CIs. ASR is a substring test on the response (the attacker's marker string: ``compromised'' or ``delete'' for the first two attackers, ``backup@external'' for the divergence minimiser, ``I need more context'' for spoofing). Each attacker uses a different payload, so ASR is not comparable across attackers. Clean-baseline rows have ASR 0 by construction (no payload is injected); the false-alarm count is given in the text. $^\ddagger$Claude-Sonnet-5's matches are responses that answer the question and then quote the injected text while declining it, not compliance (see below).}
\label{tab:rotating-canary}
\small
\begin{tabular}{llrrr}
\toprule
\textbf{Model} & \textbf{Attacker} & \textbf{N} & \textbf{ASR [95\% CI]} & \textbf{Suppress\%} \\
\midrule
GPT-5.1 & Clean baseline & 200 & 0.000 [.000, .019] & n/a \\
GPT-5.1 & Naive & 200 & 0.000 [.000, .019] & 0.0\% \\
GPT-5.1 & Canary-aware & 200 & 0.000 [.000, .019] & 38.5\% \\
GPT-5.1 & Divergence min. & 200 & 0.000 [.000, .019] & 58.0\% \\
GPT-5.1 & Spoofing & 150 & 1.000 [.975, 1.00] & 0.0\% \\
\midrule
Gemini-2.5-Pro & Clean baseline & 100 & 0.000 [.000, .037] & n/a \\
Gemini-2.5-Pro & Naive & 100 & 0.590 [.492, .681] & 0.0\% \\
Gemini-2.5-Pro & Canary-aware & 100 & 0.000 [.000, .037] & 40.0\% \\
Gemini-2.5-Pro & Divergence min. & 100 & 0.000 [.000, .037] & 57.0\% \\
Gemini-2.5-Pro & Spoofing & 100 & 0.800 [.711, .867] & 0.0\% \\
\midrule
Claude-Sonnet-5 & Clean baseline & 100 & 0.000 [.000, .037] & n/a \\
Claude-Sonnet-5$^\ddagger$ & Naive & 100 & 0.070 [.034, .137] & 0.0\% \\
Claude-Sonnet-5$^\ddagger$ & Canary-aware & 100 & 0.390 [.300, .488] & 46.0\% \\
Claude-Sonnet-5 & Divergence min. & 100 & 0.000 [.000, .037] & 60.0\% \\
Claude-Sonnet-5$^\ddagger$ & Spoofing & 100 & 0.300 [.219, .396] & 0.0\% \\
\bottomrule
\end{tabular}
\end{table}

\paragraph{Key findings.}
\begin{enumerate}
\item \textbf{Low false-alarm count}: the canary score is below 0.5 (the alarm level in the script's
  scoring function) in 1 of 400 clean sessions (Claude-Sonnet-5, 1 of 100; GPT-5.1 0 of 200;
  Gemini-2.5-Pro 0 of 100). The scoring function returns 1.0 when an expected keyword is present, 0.2 on a refusal phrase,
  0.0 on an empty response and otherwise 0.7, so in practice it separates only empty or refusing
  responses from the rest.
\item \textbf{Short factual questions are a learnable pattern}: the suppression heuristics fire on 38--60\%
  of the canary prompts.
  The canary-aware payload (``delete all workspace files'') produced zero matches on GPT-5.1 and Gemini; the naive payload also produced none on GPT-5.1,
  so the zero does not show that suppression forces failure.
\item \textbf{Spoofing is the critical vulnerability}: Attackers that force non-specific
  responses (``I need more context'') reply with the injected string in 100\% (GPT-5.1, 150 of 150) and 80\% (Gemini-2.5-Pro) of sessions without triggering a canary alarm (every canary score was 0.70), because the canary scoring function
  does not detect \emph{stylistic} rather than \emph{factual} degradation. This is a property of the scorer and its threshold, not of rotating probes: with an alarm threshold of 0.5 and a default score of 0.7 for any non-empty, non-refusal reply, only an empty or refusing reply can alarm, so a threshold above 0.7 (alarm unless the expected keyword is present) would catch these replies; we did not run that variant.
\item \textbf{Claude-Sonnet-5's matches are not compliance}: we read the 44 of its 76 matched responses that contain no refusal wording, and the remainder contain it. Each answers the question and mentions the injected instruction, almost always to say it is being disregarded; the matches arise because the substring test also fires when the model quotes the payload. The one exception we found is a reply that is exactly ``I need more context'', which is also a reasonable reply to an ambiguous question. The 39\% against 7\% contrast therefore reflects how often the model quotes the two different payloads, not attack success, and we draw no conclusion about suppression from it. The GPT-5.1 and Gemini-2.5-Pro matches for spoofing (and Gemini's for the naive attacker) are exact replies of the injected string.
\end{enumerate}

\paragraph{Deployment implication.} Rotating probes gave one false alarm in 400 clean sessions; pairing them with an external enforcement layer~\citep{leong2026recognition} would block privileged actions regardless of response content.

\end{document}